\documentclass{article} 
\usepackage{enumitem}
\usepackage{times}
\usepackage{subcaption}
\usepackage{graphicx}
\usepackage{amsmath}
\usepackage{amssymb}
\usepackage{comment}
\usepackage{amsthm}

\usepackage{tikz, pgfplots}
\usepgfplotslibrary{groupplots}
\usetikzlibrary{spy}
\usepackage{pgfplotstable}

\newtheorem{theorem}{Theorem}[section]
\newtheorem{definition}{Definition}[section]
\newtheorem{lemma}{Lemma}[section]
\DeclareMathOperator*{\E}{\mathbb{E}}
\DeclareMathOperator*{\argmin}{\operatornamewithlimits{argmin}}
\DeclareMathOperator*{\argmax}{\operatornamewithlimits{argmax}}
\usepackage[margin=1in,bottom=1in,top=.875in]{geometry}

\title{Quantization based Fast Inner Product Search}
\date{}
\author{
Ruiqi Guo,~Sanjiv Kumar,~Krzysztof Choromanski,~David Simcha\\
\\
Google Research, New York, NY 10011, USA\\
\texttt{\{guorq, sanjivk, kchoro, dsimcha\}@google.com} \\
}

\begin{document}

\maketitle

\begin{abstract}
We propose a quantization based approach for fast approximate Maximum Inner Product Search (MIPS). Each database vector is quantized in multiple subspaces via a set of codebooks, learned directly by minimizing the inner product quantization error. Then, the inner product of a query to a database vector is approximated as the sum of inner products with the subspace quantizers. Different from recently proposed LSH approaches to MIPS, the database vectors and queries do not need to be augmented in a higher dimensional feature space. We also provide a theoretical analysis of the proposed approach, consisting of the concentration results under mild assumptions. Furthermore, if a small sample of example queries is given at the training time, we propose a modified codebook learning procedure which further improves the accuracy.  Experimental results on a variety of datasets including those arising from deep neural networks show that the proposed approach significantly outperforms the existing state-of-the-art.
\end{abstract}

\section{Introduction}
Many information processing tasks such as retrieval and classification involve  computing the inner product of a query vector with a set of database vectors, with the goal of returning the database instances having the largest inner products. This is often called Maximum Inner Product Search (MIPS) problem. Formally, given a database $X=\{x_i\}_{i=1\cdots n}$, and a query vector $q$ drawn from the query distribution $\mathbf{Q}$, where $x_i, q \in \mathbb{R}^d$, we want to find $x_q^* \in X$ such that 
$ x_q^*=\argmax_{x \in X} (q^T x)$. This definition can be trivially extended to return top-$N$ largest inner products. 

The MIPS problem is particularly appealing for large scale applications. For example, a recommendation system needs to retrieve the most relevant items to a user from an inventory of millions of items, whose relevance is commonly represented as inner products~\cite{cremonesi2010performance}. Similarly, a large scale classification system needs to classify an item into one of the categories, where the number of categories may be very large~\cite{dean2013cvpr}. A brute-force computation of inner products via a linear scan requires $O(n d)$ time and space, which becomes computationally prohibitive when the number of database vectors and the data dimensionality is large. Therefore it is valuable to consider algorithms that can compress the database $X$ and compute approximate $x_q^*$ much faster than the brute-force search. 

The problem of MIPS is related to that of Nearest Neighbor Search with respect to $L_2$ distance ($L_2$NNS) or angular distance ($\theta$NNS) between a query and a database vector:
\[q^T x = 1/2(||x||^2 + ||q||^2 - ||q-x||^2) = ||q||||x||\cos{\theta}, \] or
\[ \argmax_{x \in X} (q^T x) = \argmax_{x \in X} (||x||^2 - ||q-x||^2) = \argmax_{x \in X}(||x|| cos\theta), \]

where $||.||$ is the $L_2$ norm. Indeed, if the database vectors are scaled such that $||x||=$ constant $~~\forall x \in X$, the MIPS problem becomes equivalent to L$_2$NNS or $\theta$NNS problems, which have been studied extensively in the literature. However, when the norms of the database vectors vary, as often true in practice, the MIPS problem becomes quite challenging. The inner product (distance) does not satisfy the basic axioms of a metric such as triangle inequality and co-incidence. For instance, it is possible to have $x^T x \leq x^T y$ for some $y \neq x$. In this paper, we focus on the MIPS problem where both database and the query vectors can have arbitrary norms.

As the main contribution of this paper, we develop a Quantization-based Inner Product (QUIP) search method to address the MIPS problem. We formulate the problem of quantization as that of codebook learning, which directly minimizes the quantization error in inner products (Sec.~\ref{sec:approx}). Furthermore, if a small sample of example queries is provided at the training time, we propose a constrained optimization framework which further improves the accuracy (Sec.~\ref{sec:opt}). We also provide a concentration-based theoretical analysis of the proposed method (Sec.~\ref{sec:theory}).   Extensive experiments on four real-world datasets, involving recommendation (\emph{Movielens}, \emph{Netflix}) and deep-learning based classification (\emph{ImageNet} and \emph{VideoRec}) tasks show that the proposed approach consistently outperforms the state-of-the-art techniques under both fixed space and fixed time scenarios (Sec.~\ref{sec:experiment}).


\section{Related works}
\label{sec:relatedworks}

The MIPS problem has been studied for more than a decade. For instance, Cohen et al.~\cite{cohen1999approximating} studied it in the context of document clustering and presented a method based on randomized sampling without computing the full matrix-vector multiplication. In \cite{koenigstein2012cikm, ram2012kdd}, the authors described a procedure to modify tree-based search to adapt to MIPS criterion. Recently, Bachrach et al.~\cite{bachrach2014speeding} proposed an approach that transforms the input vectors such that the MIPS problem becomes equivalent to the $L_2$NNS problem in the transformed space, which they solved using a PCA-Tree.

The MIPS problem has received a renewed attention with the recent seminal work from Shrivastava and Li~\cite{shrivastava2014asymmetric}, which introduced an Asymmetric Locality Sensitive Hashing (ALSH) technique with provable search guarantees. They also transform MIPS into $L_2$NNS, and use the popular LSH technique~\cite{andoni2006lsh}.
Specifically, ALSH applies different vector transformations to a database vector $x$ and the query $q$,  respectively:
\[
\hat{x} = [\tilde{x}; ||\tilde{x}||^2; ||\tilde{x}||^4; \cdots ||\tilde{x}||^{2^m}].~~~~\hat{q} = [q; 1/2; 1/2; \cdots ;1/2].
\]
where $\tilde{x}=U_0 \frac{x}{\max_{x \in X} ||x||}$, $U_0$ is some constant that satisfies $0<U_0<1$, and $m$ is a nonnegative integer. Hence, $x$ and $q$ are mapped to a new $(d+m)$ dimensional space asymmetrically. Shrivastava and Li~\cite{shrivastava2014asymmetric} showed that when $m\rightarrow \infty$, MIPS in the original space is equivalent to $L_2$NNS in the new space. The proposed hash function followed $L_2$LSH form~\cite{andoni2006lsh}:  
$
h^{L2}_i(\hat{x})=\lfloor \frac{P_i^T \hat{x}+b_i}{r} \rfloor,
$
where $P_i$ is a $(d+m)$-dimensional vector whose entries are sampled i.i.d from the standard Gaussian, $\mathcal{N}(0,1)$, and $b_i$ is sampled uniformly from $[0, r]$. The same authors later proposed an improved version of ALSH  
based on Signed Random Projection (SRP)~\cite{shrivastava014e}. It transforms each vector using a slightly different procedure and represents it as a binary code. Then, Hamming distance is used for MIPS.
\[
\hat{x} = [\tilde{x}; \frac{1}{2}-||\tilde{x}||^2; \frac{1}{2}-||\tilde{x}||^4; \cdots \frac{1}{2}-||\tilde{x}||^{2^m}],~~~~\hat{q} = [q; 0; 0; \cdots ;0], ~~\textrm{and}
\]
\[
h^{SRP}_i(\hat{x})=sign(P_i^T \hat{x});~~Dist^{SRP}(x,q)=\sum_{i=1}^{b} h_i^{SRP} (\hat{x}) \neq h_i^{SRP} (\hat{q}).
\]
Recently, Neyshabur and Srebro~\cite{neyshabur2014simpler} argued that a symmetric transformation was sufficient to develop a provable LSH approach for the MIPS problem if query was restricted to unit norm. They used a transformation similar to the one used by Bachrach et al.~\cite{bachrach2014speeding} to augment the original vectors:
\[
\hat{x} = [\tilde{x}; \sqrt{1-||\tilde{x}||^2}].~~~~\hat{q} = [\tilde{q}; 0].
\]
where $\tilde{x}=\frac{x}{max_{x\in X} ||x||}$,  $\tilde{q}=\frac{q}{||q||}$. They showed that this transformation led to significantly improved results over the SRP based LSH from~\cite{shrivastava014e}. In this paper, we take a quantization based view of the MIPS problem and show that it leads to even better accuracy under both fixed space or fixed time budget on a variety of real world tasks. 


\section{Quantization-based inner product (QUIP) search}
\label{sec:approx}
Instead of augmenting the input vectors to a higher dimensional space as in~\cite{neyshabur2014simpler,shrivastava2014asymmetric}, we approximate the inner products by mapping each vector to a set of subspaces, followed by independent quantization of database vectors in each subspace. In this work, we use a simple procedure for generating the subspaces. Each vector's elements are first permuted using a random (but fixed) permutation\footnote{Another possible choice is random rotation of the vectors which is slightly more expensive than permutation but leads to improved theoretical guarantees as discussed in the appendix.}. Then each permuted vector is mapped to $K$ subspaces using simple chunking, as done in product codes~\cite{sabin1984icassp, jegou2011pami}. For ease of notation, in the rest of the paper we will assume that both query and database vectors have been permuted. Chunking leads to block-decomposition of the query $q \sim \mathbf{Q}$ and each database vector $x \in X$:
\[
x=[x^{(1)}; x^{(2)}; \cdots; x^{(K)}]~~~~q=[q^{(1)}; q^{(2)}; \cdots; q^{(K)}],  
\]
where each $x^{(k)}, q^{(k)} \in \mathbb{R}^l, l = \lceil d/K \rceil.$\footnote{One can do zero-padding wherever necessary, or use different dimensions in each block.}
The $k^{th}$ subspace containing the $k^{th}$ blocks of all the database vectors, $\{x^{(k)}\}_{i=1...n}$, is then quantized by a codebook $U^{(k)} \in \mathbb{R}^{l \times C_k}$ where $C_k$ is the number of quantizers in subspace $k$. Without loss of generality, we assume $C_k = C ~~\forall~k$. Then, each database vector $x$ is quantized in the $k^{th}$ subspace as $x^{(k)} \approx U^{(k)} \alpha_x^{(k)}$, where $\alpha_x^{(k)}$ is a $C$-dimensional one-hot assignment vector with exactly one $1$ and rest $0$. Thus, a database vector $x$ is quantized by a single dictionary element $u^{(k)}_x$ in the $k^{th}$ subspace.
Given the quantized database vectors, the exact inner product is approximated as:
\begin{equation}
q^T x = \sum_k q^{(k)T} x^{(k)} \approx \sum_k q^{(k)T} U^{(k)} \alpha^{(k)}_x = \sum_k q^{(k)T} u^{(k)}_x
\label{eqn:approx}
\end{equation}
Note that this approximation is 'asymmetric' in the sense that only database vectors $x$ are quantized, not the query vector $q$. One can quantize $q$ as well but it will lead to increased approximation error. 
In fact, the above asymmetric computation for all the database vectors can still be carried out very efficiently via look up tables similar to~\cite{jegou2011pami}, except that each entry in the $k^{th}$ table is a dot product between $q^{(k)}$ and columns of $U^{(k)}$ . 

Before describing the learning procedure for the codebooks $U^{(k)}$ and the assignment vectors $\alpha_x^{(k)}$ $\forall~x, k$, we first show an interesting property of the approximation in (\ref{eqn:approx}). Let $S^{(k)}_c$ be the $c^{th}$ partition of the database vectors in subspace $k$ such that $S^{(k)}_c = \{x^{(k)}\!:\alpha_x^{(k)}[c] = 1\}$, where  $\alpha^{(k)}_x[c]$ is the $c^{th}$ element of $\alpha^{(k)}_x$ and $U_c^{(k)}$ is the $c^{th}$ column of $U^{(k)}$.  

\begin{lemma}
\label{thm:unbiased}
If $\displaystyle U^{(k)}_c=\frac{1}{|S^{(k)}_c|}\sum_{x^{(k)} \in S^{(k)}_c} x^{(k)}$, then~(\ref{eqn:approx}) is an unbiased estimator of $q^T x$.
\end{lemma}
\begin{proof}
\begin{align*}
\E_{q\sim \mathbf{Q},x\in X}[q^Tx - \sum_k q^{(k)T} u_x^{(k)}]&=\sum_k \E_{q\sim\mathbf{Q}} q^{(k)T} \E_{x\in X}[(x^{(k)}-u_x^{(k)}]\\
&=\sum_k \E_{q \sim \mathbf{Q}} q^{(k)T} \E_{x\in X}[\sum_c \mathbb{I}[x^{(k)} \in S^{(k)}_c] (x^{(k)}-U_c^{(k)})]\\
&=0.
\end{align*}
Where $\mathbb{I}$ is the indicator function, and the last equality holds because for each $k$, $\E_{x \in S^{(k)}_c}[x^{(k)}-U^{(k)}_c]=0$ by definition. 
\end{proof}

We will provide the concentration inequalities for the estimator in ({\ref{eqn:approx}) in Sec.~\ref{sec:theory}. Next we describe the learning of quantization codebooks in different subspaces. We focus on two different training scenarios: when only the database vectors are given (Sec.~\ref{sec:kmeans}), and when a sample of example queries is also provided (Sec.~\ref{sec:opt}). The latter can result in significant performance gain when queries do not follow the same distribution as the database vectors. Note that the actual queries used at the test time are different from the example queries, and hence unknown at the training time.

\subsection{Learning quantization codebooks from database}
\label{sec:kmeans}
Our goal is to learn data quantizers that minimize the quantization error due to the inner product approximation given in (\ref{eqn:approx}). Assuming each subspace to be independent, the expected squared error can be expressed as:
\begin{align}
\begin{split}
\E_{q \sim \mathbf{Q}} \E_{x \in X} [q^T x - \sum_k q^{(k)T} U^{(k)}\alpha_x^{(k)}]^2 &= \sum_k \E_{q \sim \mathbf{Q}} \E_{x \in X} [q^{(k)T} (x^{(k)} - u_x^{(k)})]^2\\
&=\sum_k \E_{x \in X} (x^{(k)} - u_x^{(k)})^T \Sigma_{\mathbf{Q}}^{(k)} (x^{(k)} - u_x^{(k)}),
\end{split}
\label{eqn:mse}
\end{align}
where $\Sigma^{(k)}_{\mathbf{Q}}=\E_{q \sim \mathbf{Q}}  q^{(k)} q^{(k)T}$ is the non-centered query covariance matrix in subspace $k$. 
Minimizing the error in (\ref{eqn:mse}) is equivalent to solving a modified \emph{k-Means} problem in each subspace independently. Instead of using the Euclidean distance, Mahalanobis distance specified by $\Sigma_{\mathbf{Q}}^{(k)}$ is used for assignment. One can use the standard Lloyd's algorithm to find the solution for each subspace $k$ iteratively by alternating between two steps:
\begin{eqnarray}
\label{eqn:kmeans}
c^{(k)}_x &=& \argmin_{c} (x^{(k)} - U^{(k)}_c)^T \Sigma_{\mathbf{Q}}^{(k)} (x^{(k)} - U^{(k)}_c), ~~ \alpha_x^{(k)}[c_x^{(k)}] = 1, ~~\forall~c,x \nonumber \\
U^{(k)}_c&=&\frac{\sum_{x^{(k)} \in S^{(k)}_c} x^{(k)}} {|S^{(k)}_c|} ~~~\forall~c.
\end{eqnarray}
The Lloyd's algorithm is known to converge to a local minimum (except in pathological cases where it may oscillate between equivalent solutions)~\cite{bottou94kmeans}. Also, note that the resulting quantizers are always the Euclidean means of their corresponding partitions, and hence, Lemma~\ref{thm:unbiased} is applicable to (\ref{eqn:mse}) as well, leading to an unbiased estimator.

The above procedure requires the non-centered query covariance matrix $\Sigma_{\mathbf{Q}}$, which will not be known if query samples are not  available at the training time. In that case, one possibility is to assume 
that the queries come from the same distribution as the database vectors, i.e., $\Sigma_{\mathbf{Q}}=\Sigma_{X}$. In the experiments we will show that this version performs reasonably well. However, if a small set of example queries is available at the training time, besides estimating the query covariance matrix, we propose to impose novel constraints that lead to improved quantization, as described next.

\subsection{Learning quantization codebook from database and example query samples}
\label{sec:opt}
In most applications, it is possible to have access to a small set of example queries, $Q$. Of course, the actual queries used at the test-time are different from this set. Given these exemplar queries, we propose to modify the learning criterion by imposing additional constraints while minimizing the expected quantization error. Given a query $q$, since we are interested in finding the database vector $x^*_q$ with highest dot-product, ideally we want the dot product of query to the quantizer of $x^*_q$ to be larger than the dot product with any other quantizer. Let us denote the matrix containing the $k^{th}$ subspace assignment vectors $\alpha_x^{(k)}$ for all the database vectors by $A^{(k)}$. Thus, the modified optimization is given as,
\begin{align}
\begin{split}
\argmin_{U^{(k)}, A^{(k)}} ~~~~~~& \E_{q \in Q}  \sum_{x \in X} [\sum_k  q^{(k)T} x^{(k)} - \sum_k  q^{(k)T} U^{(k)} \alpha^{(k)}_x ]^2 \\
s.t. ~~~~~~& \forall q,x,~~ \sum_k  q^{(k)T} U^{(k)}\alpha_x^{(k)} \leq \sum_k  q^{(k)T}U^{(k)}\alpha_{x_q^*}^{(k)}~\text{where}~~x_q^*=\argmax_{x} q^T x
\end{split}
\label{eqn:opt}
\end{align}
We relax the above hard constraints using slack variables to allow for some violations, which leads to the following equivalent objective:
\begin{align}
\argmin_{U^{(k)}, A^{(k)}}  \E_{q \in Q}  \sum_{x \in X} \sum_k \big(q^{(k)T} (x^{(k)} - U^{(k)}\alpha_x^{(k)})\big)^2 
+ \lambda \sum_{q \in Q}  \sum_{x \in X} [\sum_k q^{(k)T} (U^{(k)}\alpha_x^{(k)}-U^{(k)}\alpha_{x^*_q}^{(k)})]_{+} 
\label{eqn:opt2}
\end{align}
where $[z]_{+}=max(z,0)$ is the standard hinge loss, and $\lambda$ is a nonnegative coefficient. We use an iterative procedure to solve the above optimization, which alternates between solving $U^{(k)}$ and $A^{(k)}$ for each $k$. In the beginning, each codebook $U^{(k)}$ is initialized with a set of random database vectors mapped to the $k^{th}$ subspace. Then, we iterate through the following three steps:
\begin{enumerate}[leftmargin=15pt,itemsep=-.3ex]
\item Find a set of violated constraints $W$ with each element as a triplet, i.e., $W_j = \{q_j, x^*_{q_j}, x_j^-\}_{j=1\cdots J}$, where $q_j \in Q$ is an exemplar query, $x^*_{q_j}$ is the database vector having the maximum dot product with $q_j$, and $x_j^-$ is a vector such that $q_j^T x_{q_j}^* \geq q_j^T x_j^-$ but 
\[
\sum_k q_j^{(k)T} U^{(k)} \alpha^{(k)}_{x^*_{q_j}} < \sum_k q_j^{(k)T} U^{(k)}\alpha^{(k)}_{x_j^-}  
\]
\item Fixing $U^{(k)}$ and all columns of $A^{(k)}$ except $\alpha^{(k)}_{x}$, one can update $\alpha^{(k)}_{x}$ $\forall ~ x, k$ as:
\begin{eqnarray*}
c^{(k)}_{x}\!=\!\argmin_c\!\big( (x^{(k)} \!-\! U^{(k)}_c)^T \Sigma^{(k)}_{Q} (x^{(k)}\! - \!U^{(k)}_c) \!+\!
\lambda \big(\sum_j\! q^{(k)T} U^{(k)}_c (\mathbb{I}[x=x_j^-]\! -\! \mathbb{I}[x=x_{q_j}^*]) \big), \\
\alpha^{(k)}_{x}[c^{(k)}_{x}] = 1
\end{eqnarray*}

Since $C$ is typically small (256 in our experiments), we can find $c^{(k)}_{x}$ by enumerating all possible values of $c$.

\item Fixing $A$, and all the columns of $U^{(k)}$ except $U_c^{(k)}$, one can update $U_c^{(k)}$ by gradient descent where gradient can be computed as:
\begin{align*}
\nabla U_c^{(k)}=  2 \Sigma^{(k)}_{Q} \sum_{x\in X} \alpha^{(k)}_x[c] (U_c^{(k)} - x^{(k)}) + \lambda \sum_j \big( q_j^{(k)}(\alpha^{(k)}_{x_j^-}[c]-\alpha^{(k)}_{x_{q_j}^*}[c]) \big)
\end{align*}
\end{enumerate}
Note that if no violated constraint is found, step 2 is equivalent to finding the nearest neighbor of $x^{(k)}$ in $U^{(k)}$in Mahalanobis space specified by $\Sigma^{(k)}_{Q}$. Also, in that case, by setting $\nabla U_c^{(k)}=0$, the update rule in step 3 becomes $ U^{(k)}_c=\frac{1}{|S^{(k)}_c|}\sum_{x^{(k)} \in S^{(k)}_c} x^{(k)}$ which is the stationary point for the first term. Thus, if no constraints are violated, the above procedure becomes identical to \emph{k-Means}-like procedure described in Sec.~\ref{sec:kmeans}. The steps 2 and 3 are guaranteed not to increase the value of the objective in (\ref{eqn:opt}).
In practice, we have found that the iterative procedure can be significantly sped up by modifying the step 3 as perturbation of the stationary point of the first term 
with a single gradient step of the second term. The time complexity of step 1 is at most $O(nKC|Q|)$, but in practice it is much cheaper because we limit the number of constraints in each iteration to be at most $J$. Step 2 takes $O(nKC)$ and step 3 $O((n+J)KC)$ time. In all the  experiments, we use at most $J=1000$ constraints in each iteration, Also, we fix $\lambda=.01$, step size $\eta_t=1/(1+t)$ at each iteration $t$, and the maximum number of iterations $T=30$.

\section{Theoretical analysis}
\label{sec:theory}

In this section we present concentration results about the quality of the quantization-based inner product search method. 
Due to the space constraints, proofs of the theorems are provided in the appendix. We start by defining a few quantities.
\begin{definition}
\label{def:event}
Given fixed $a, \epsilon > 0$, let $\mathcal{F}(a, \epsilon)$ be an event such that
the exact dot product $q^{T}x$ is at least $a$, but the quantized version is either smaller than $q^{T}x(1-\epsilon)$ or larger than $q^{T}x(1+\epsilon)$.
\end{definition}
Intuitively, the probability of event $\mathcal{F}(a,\epsilon)$ measures the chance that difference between the exact and the quantized dot product is large, when the exact dot product is large. We would like this probability to be small. Next, we introduce the concept of balancedness for subspaces.

\begin{definition}
\label{def:balance}
Let $v$ be a vector which is chunked into $K$ subspaces: $v^{(1)},...,v^{(K)}$. We say that chunking is $\eta$-balanced if the following holds for every $k \in \{1,...,K\}$:
$$\|v^{(k)}\|^{2} \leq (\frac{1}{K} + (1-\eta))\|v\|^{2}$$
\end{definition}

Since the input data may not satisfy the balancedness condition, we next show that random permutation tends to create more balanced subspaces. Obviously, a (fixed) random permutation applied to vector entries does not change the dot product. 

\begin{theorem}
\label{perm_theorem_main}
Let $v$ be a vector of dimensionality $d$ and let $perm(v)$ be its version after applying random permutation of its dimensions. Then the expected $perm(v)$ is $1$-balanced.
\end{theorem}

Another choice of creating balancedness is via a (fixed) random rotation, which also does not change the dot-product. This leads to even better balancedness property as discussed in the appendix (see Theorem 2.1). 
Next we show that the probability of $\mathcal{F}(a,\epsilon)$ can be upper bounded by an exponentially small quantity in $K$,  indicating that the quantized dot products accurately approximate large exact dot products when the quantizers are the means obtained from Mahalanobis \emph{k-Means} as described in Sec.~\ref{sec:kmeans}. Note that in this case quantized dot-product is an unbiased estimator of the exact dot-product as shown in Lemma~\ref{thm:unbiased}.
\begin{theorem}
\label{lipsch_theory_main}
Assume that the dataset $X$ of dimensionality $d$ resides entirely in the ball $\mathcal{B}(p,r)$ of radius $r$, centered at $p$ . Further, let $\{x-p : x \in X\}$ be $\eta$-balanced for some $0 < \eta < 1$, where $\backslash$ is applied pointwise, and let $\E[\sum_k (x^{(k)}-u_x^{(k)})]_{k=1\cdots K}$ be a martingale. Denote $q_{max} = \max_{k=1,...,K} \max_{q \in Q}\|q^{(k)}\|$. Then, there exist $K$ sets of codebooks, each with $C$ quantizers, such that the following is true:
$$\mathbb{P}(\mathcal{F}(a,\epsilon)) \leq 
2e^{-(\frac{a \epsilon}{r})^{2}\frac{C^{\frac{2K}{d}}}{8q_{max}^{2}(1+(1-\eta)K)}}.$$
\end{theorem}

The above theorem shows that the probability of $\mathcal{F}(a, \epsilon)$ decreases exponentially as the number of subspaces (i.e., blocks) $K$ increases. This is consistent with experimental observation that increasing $K$ leads to more accurate retrieval. 

Furthermore, if we assume that each subspace is independent, which is a slightly more restrictive assumption than the martingale assumption made in Theorem~\ref{lipsch_theory_main}, we can use Berry-Esseen~\cite{NT07} inequality to obtain an even stronger upper bound as given below. 

\begin{theorem}
\label{ind_theory1_main}
Suppose, $\Delta = \max_{k=1,...,K} \Delta^{(k)}$, where $\Delta^{(k)}=\max_x ||u^{(k)}_x-x^{(k)}||$ is the maximum distance between a datapoint and its quantizer in subspace $k$.   Assume $\Delta \leq \frac{a^{\frac{1}{3}}}{q_{max}}$. Then,
$$\mathbb{P}(\mathcal{F}(a,\epsilon)) \leq \frac{2\sum_{k=1}^{K}L^{(k)}}{\sqrt{2\pi}|X|a\epsilon}e^{-\frac{a^{2}\epsilon^{2}|X|^{2}}{2(\sum_{k=1}^{K}L^{(k)})^{2}}}+
\frac{\beta K(\sum_{k=1}^{K} L^{(k)})^{\frac{3}{2}}}{a^{2}\epsilon^{3}|X|^{\frac{3}{2}}},$$
where
$L^{(k)} = E_{q \in Q} [\sum_{S^{(k)}_{c}} \sum_{x\in S^{(k)}_{c}}(q^{(k)T} x^{k} - q^{(k)T} u^{(k)}_{x})^2]$ and
$\beta>0$ is some universal constant.
\end{theorem}



\section{Experimental results}
\label{sec:experiment}

\begin{figure}
\centering
\begin{subfigure}[c]{1 \textwidth}
\vskip -9pt
\includegraphics[trim=0.6in 2.5in 0.9in 2.5in,clip,width=.26 \textwidth]{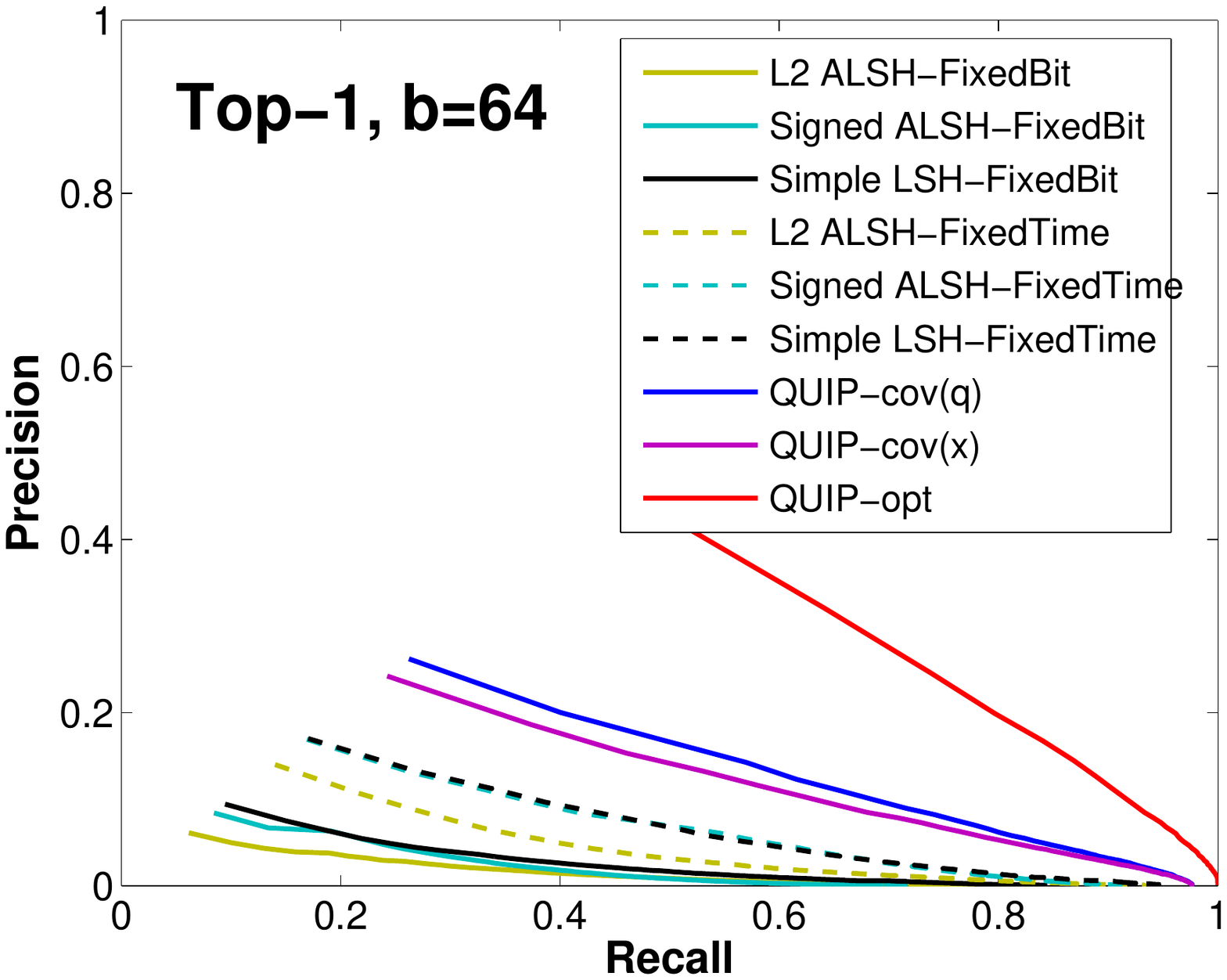} \hskip -3pt
\includegraphics[trim=0.9in 2.5in 0.9in 2.5in,clip,width=.248 \textwidth]{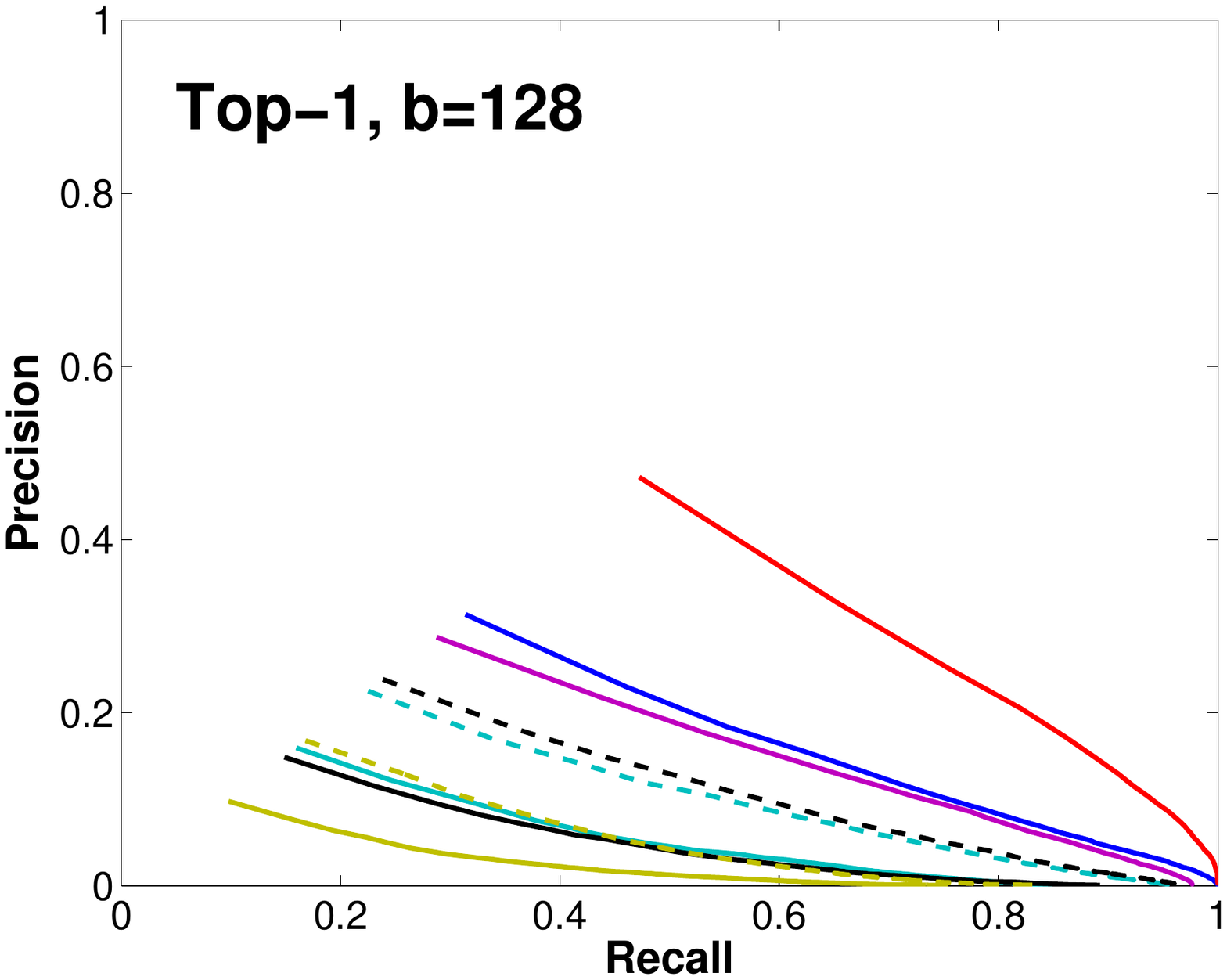} \hskip -3pt
\includegraphics[trim=0.9in 2.5in 0.9in 2.5in,clip,width=.248 \textwidth]{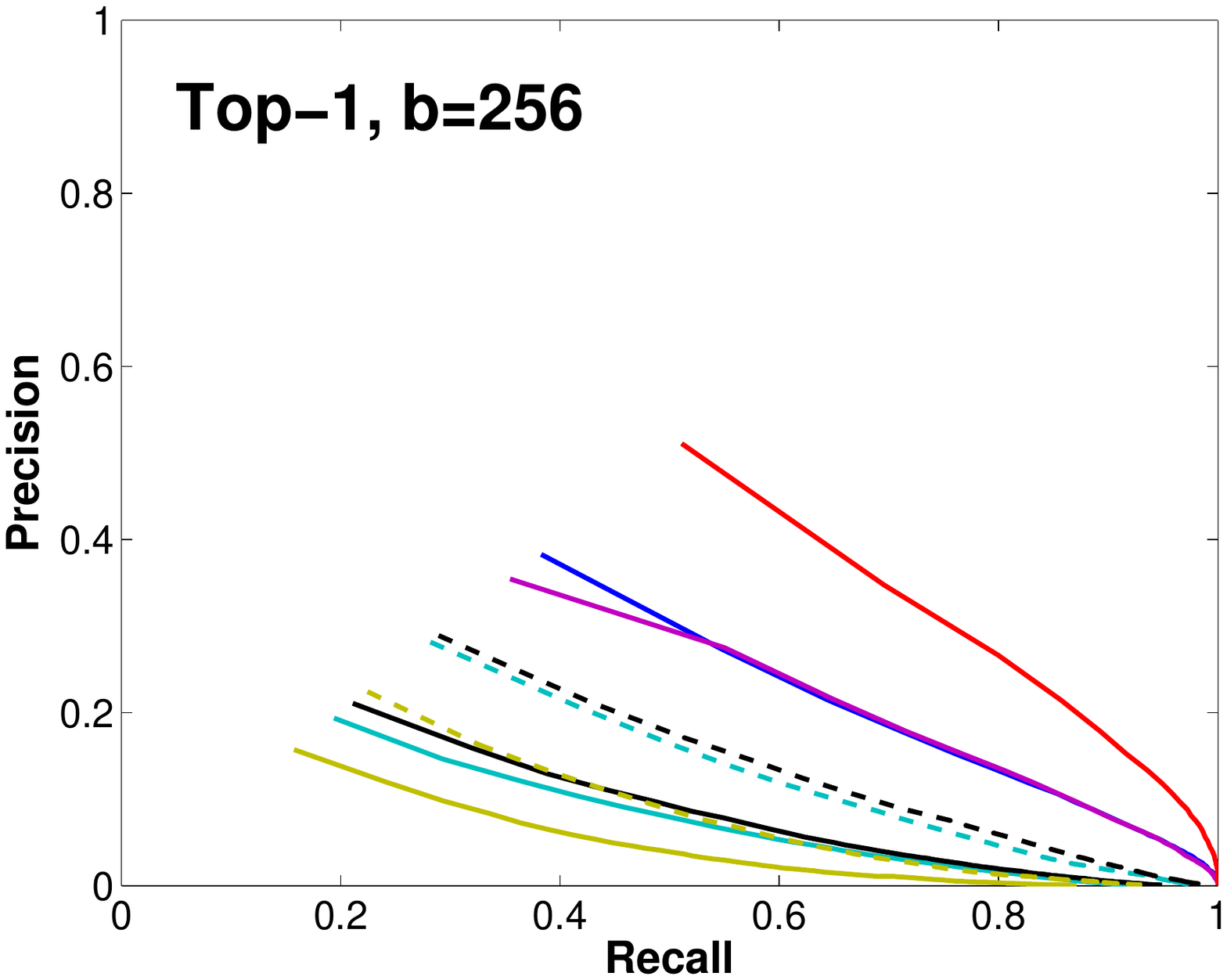} \hskip -3pt
\includegraphics[trim=0.9in 2.5in 0.9in 2.5in,clip,width=.248 \textwidth]{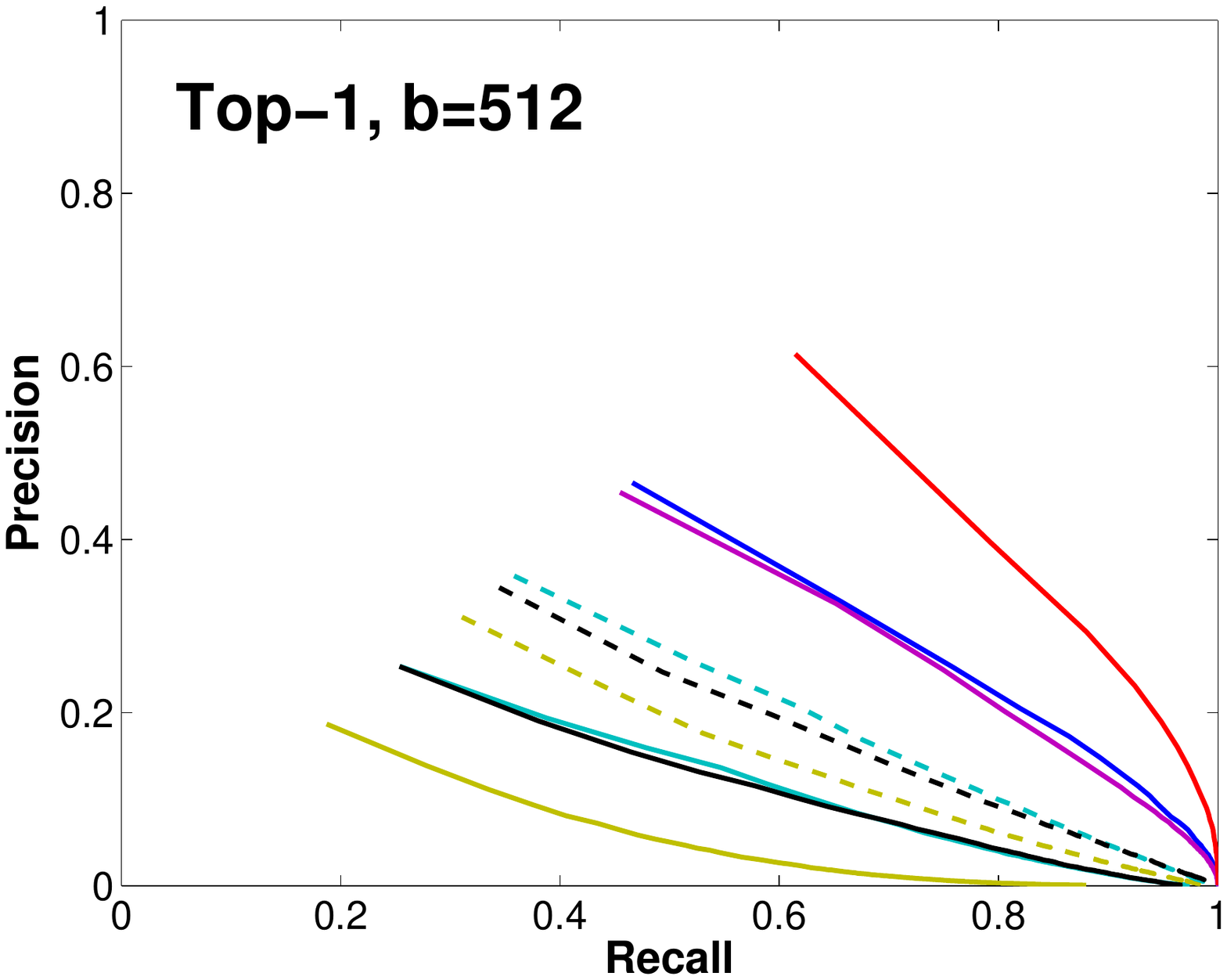}
\vskip -9pt
\includegraphics[trim=0.6in 2.5in 0.9in 2.5in,clip,width=.26 \textwidth]{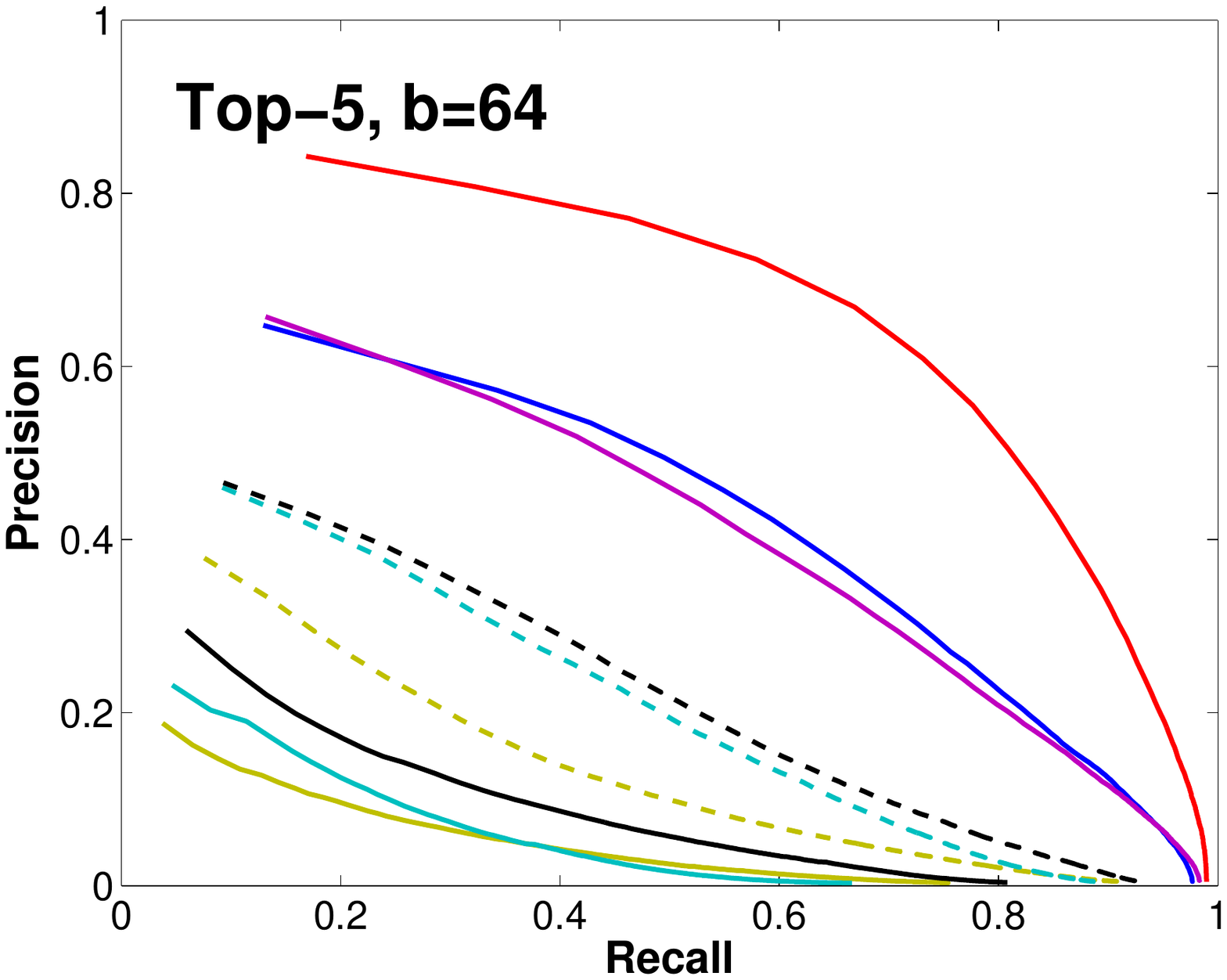} \hskip -3pt
\includegraphics[trim=0.9in 2.5in 0.9in 2.5in,clip,width=.248 \textwidth]{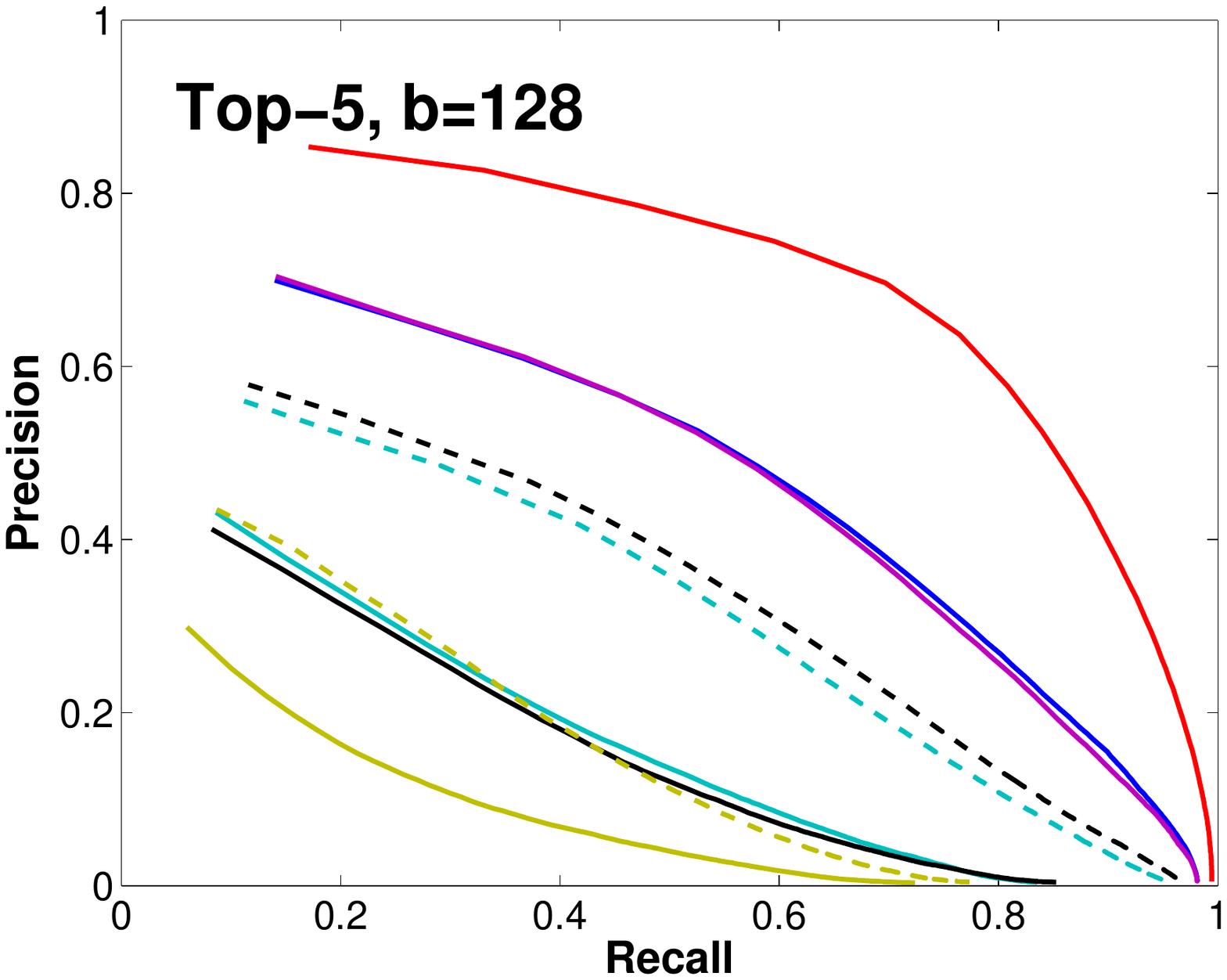} \hskip -3pt
\includegraphics[trim=0.9in 2.5in 0.9in 2.5in,clip,width=.248 \textwidth]{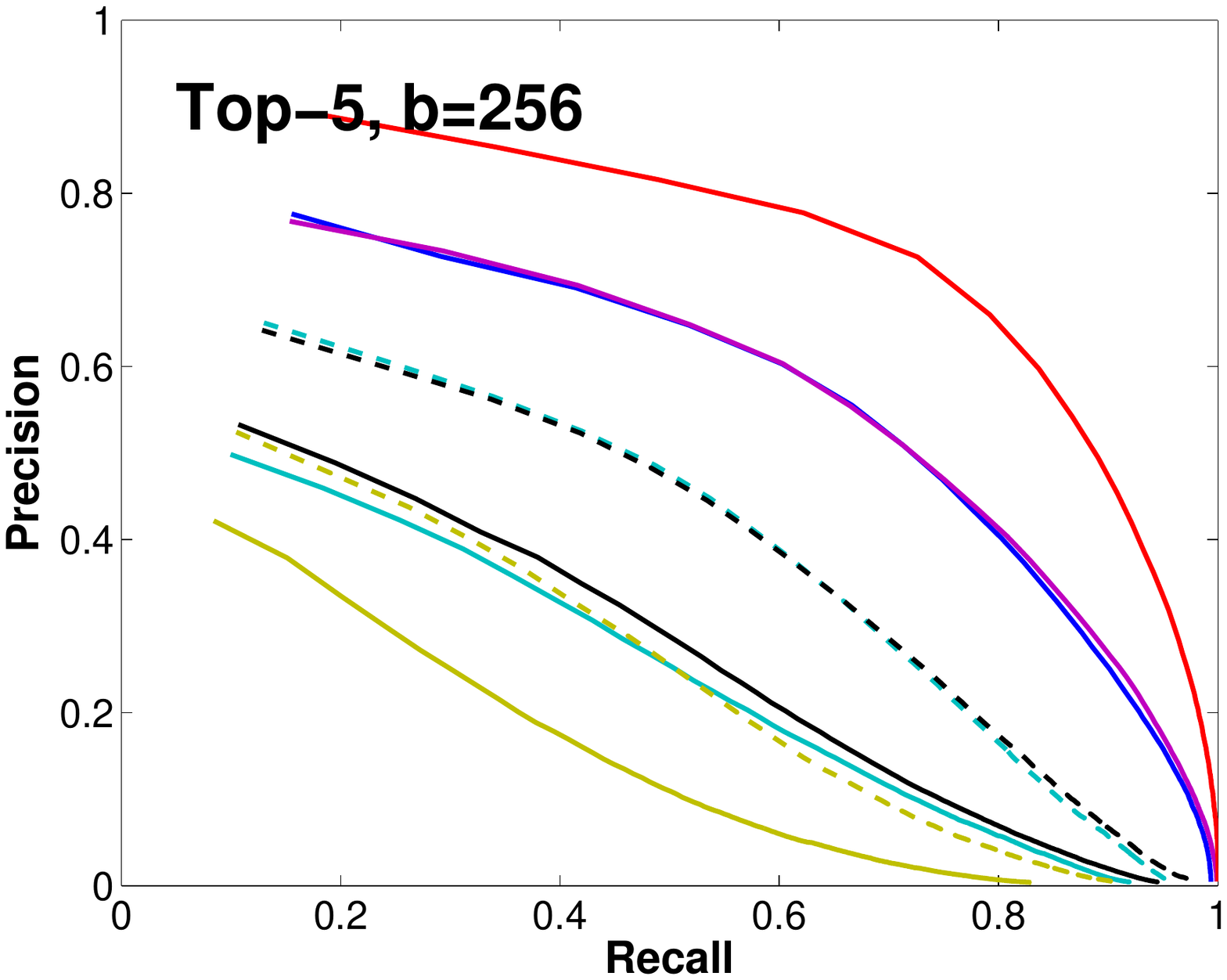} \hskip -3pt
\includegraphics[trim=0.9in 2.5in 0.9in 2.5in,clip,width=.248 \textwidth]{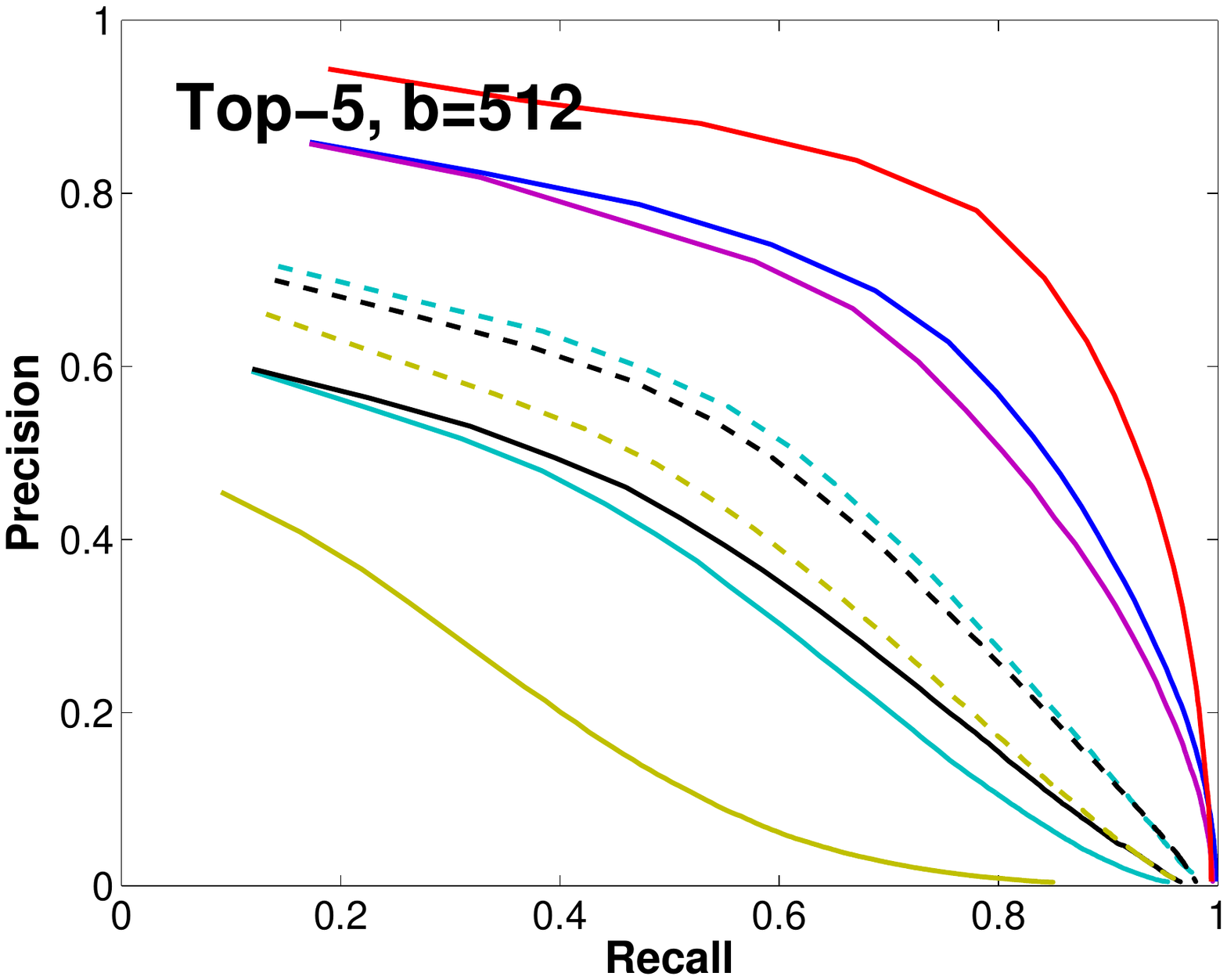}
\vskip -9pt
\includegraphics[trim=0.6in 2.5in 0.9in 2.5in,clip,width=.26 \textwidth]{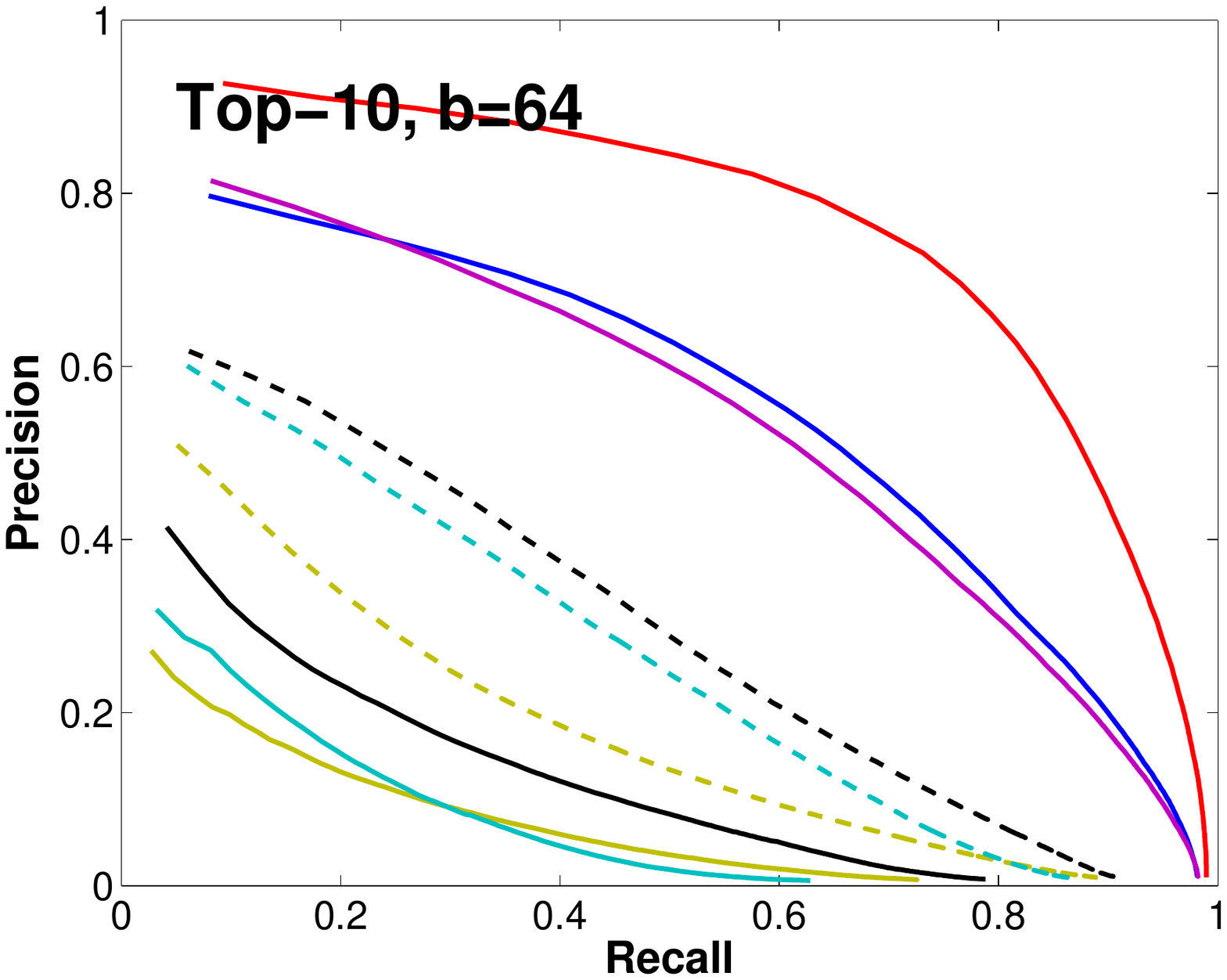} \hskip -3pt
\includegraphics[trim=0.9in 2.5in 0.9in 2.5in,clip,width=.248 \textwidth]{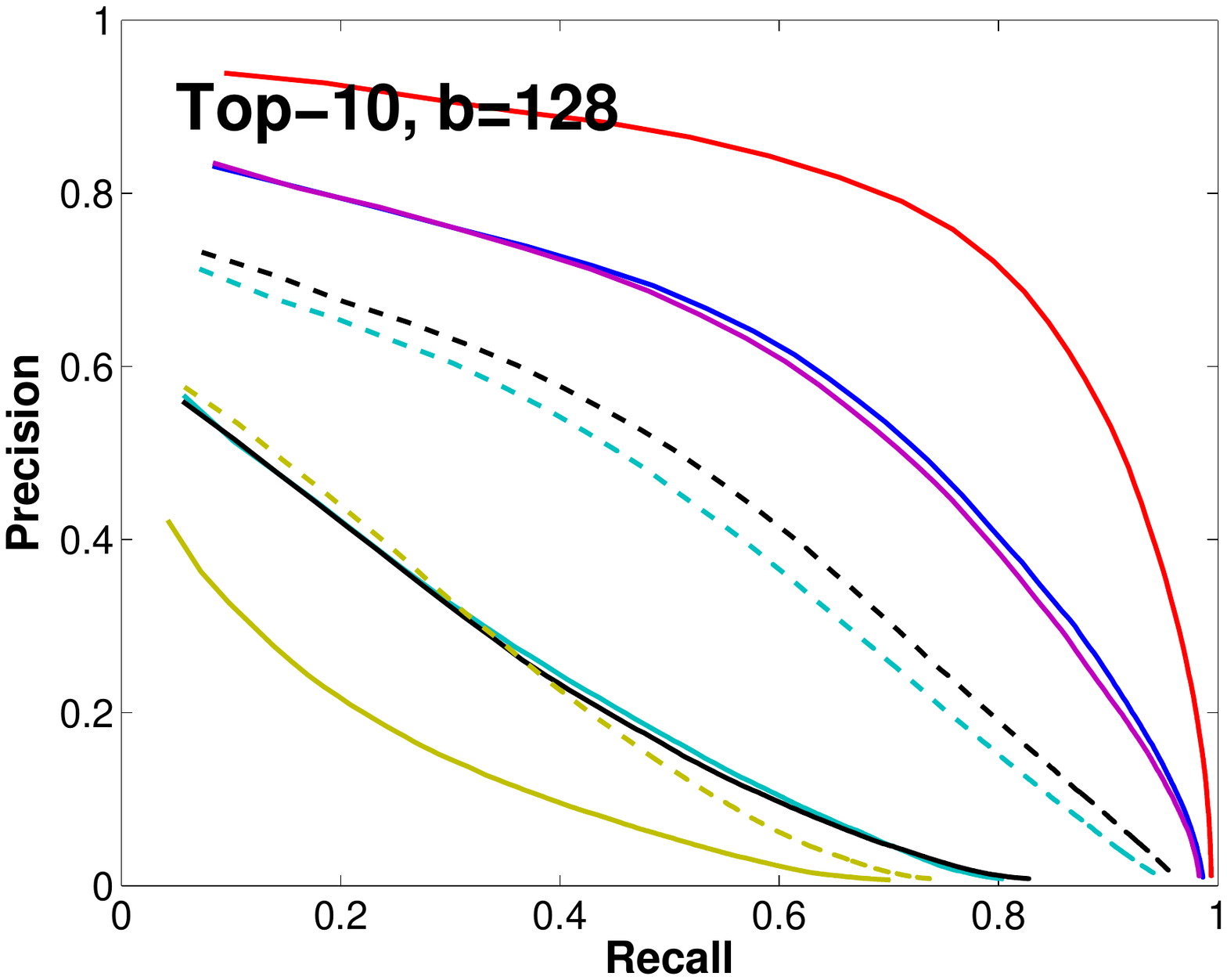} \hskip -3pt
\includegraphics[trim=0.9in 2.5in 0.9in 2.5in,clip,width=.248 \textwidth]{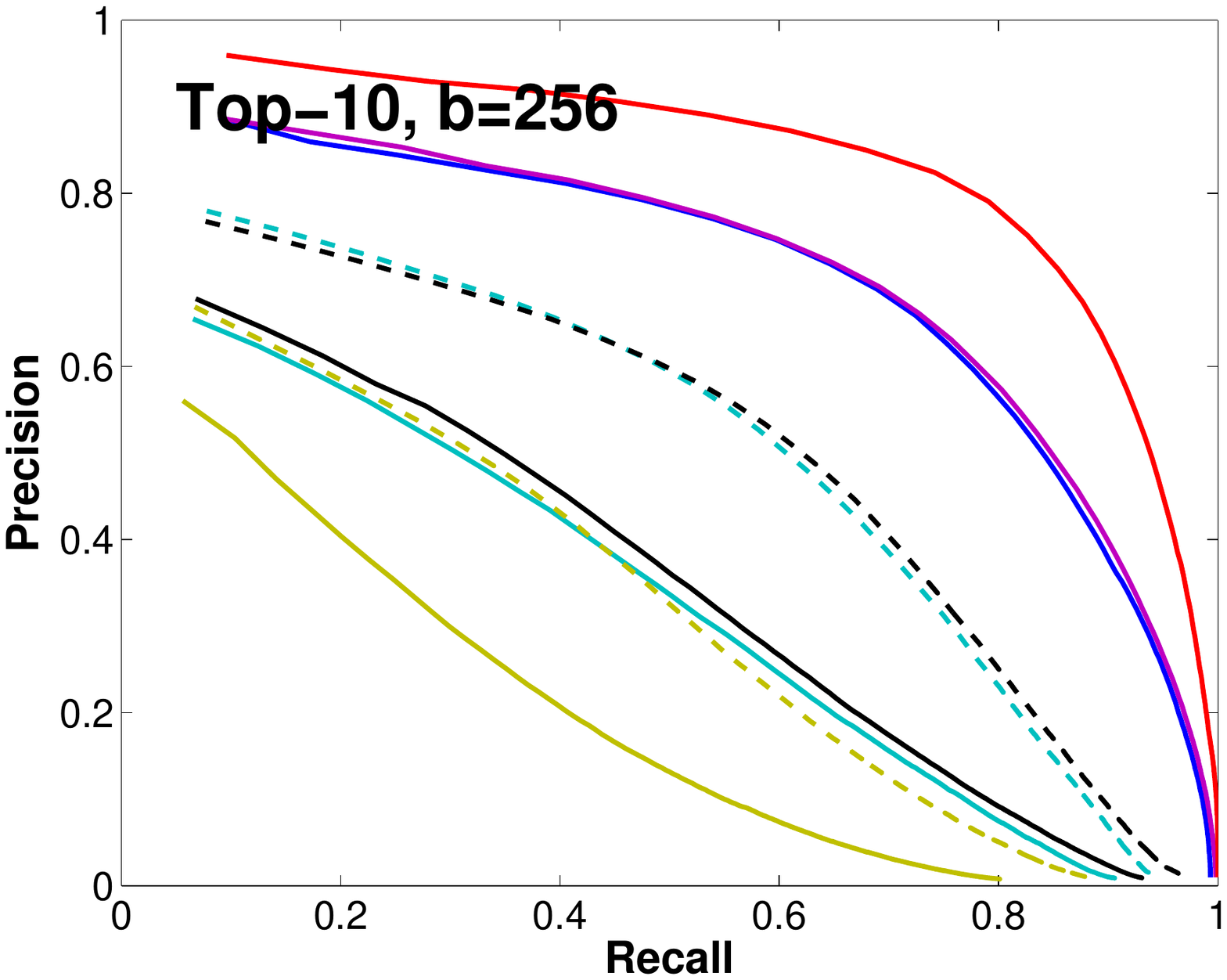} \hskip -3pt
\includegraphics[trim=0.9in 2.5in 0.9in 2.5in,clip,width=.248 \textwidth]{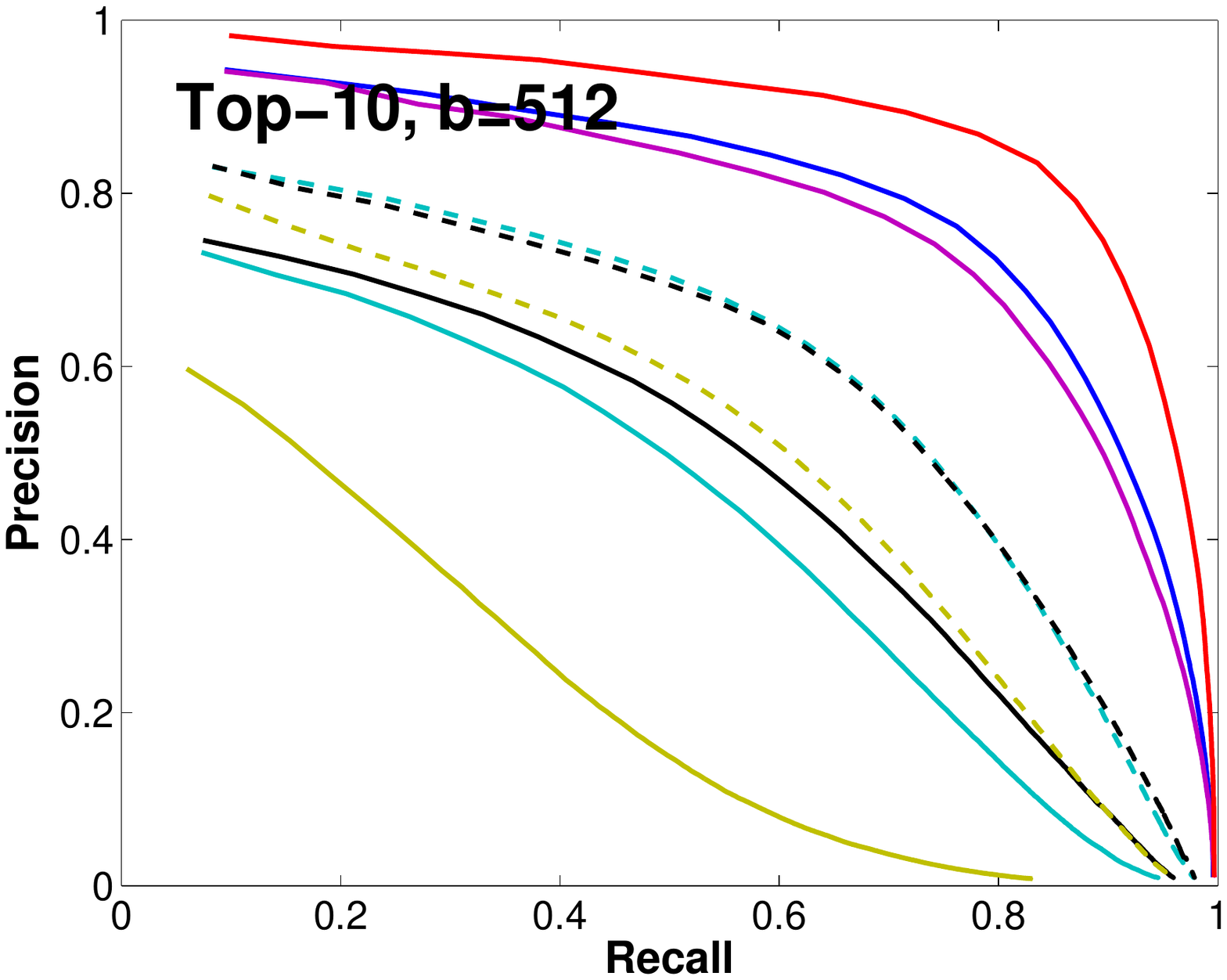}
\vskip -3pt
\subcaption{Movielens dataset}
\end{subfigure}
\begin{subfigure}[c]{1 \textwidth}
\includegraphics[trim=0.6in 2.5in 0.9in 2.5in,clip,width=.26 \textwidth]{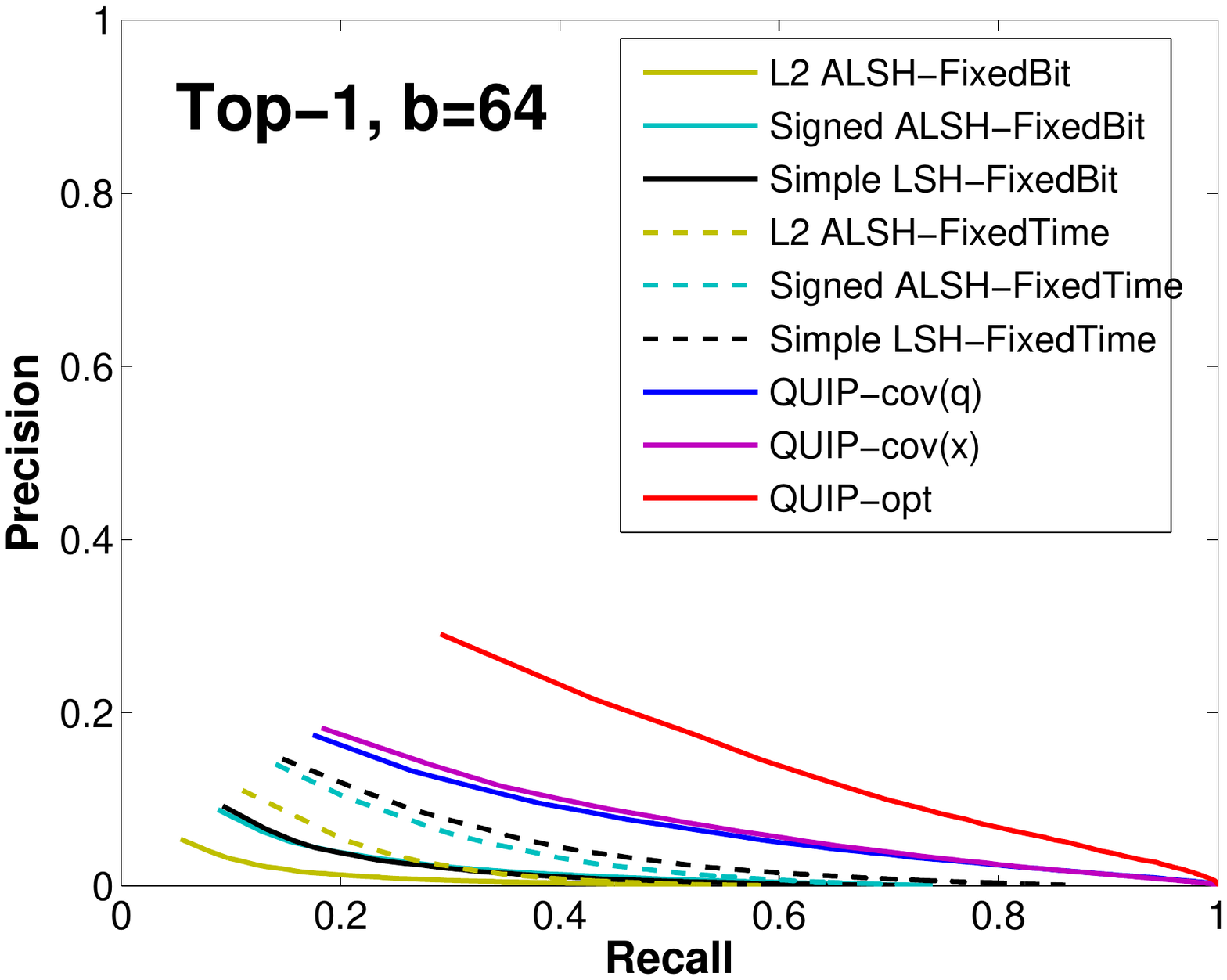} \hskip -3pt
\includegraphics[trim=0.9in 2.5in 0.9in 2.5in,clip,width=.248 \textwidth]{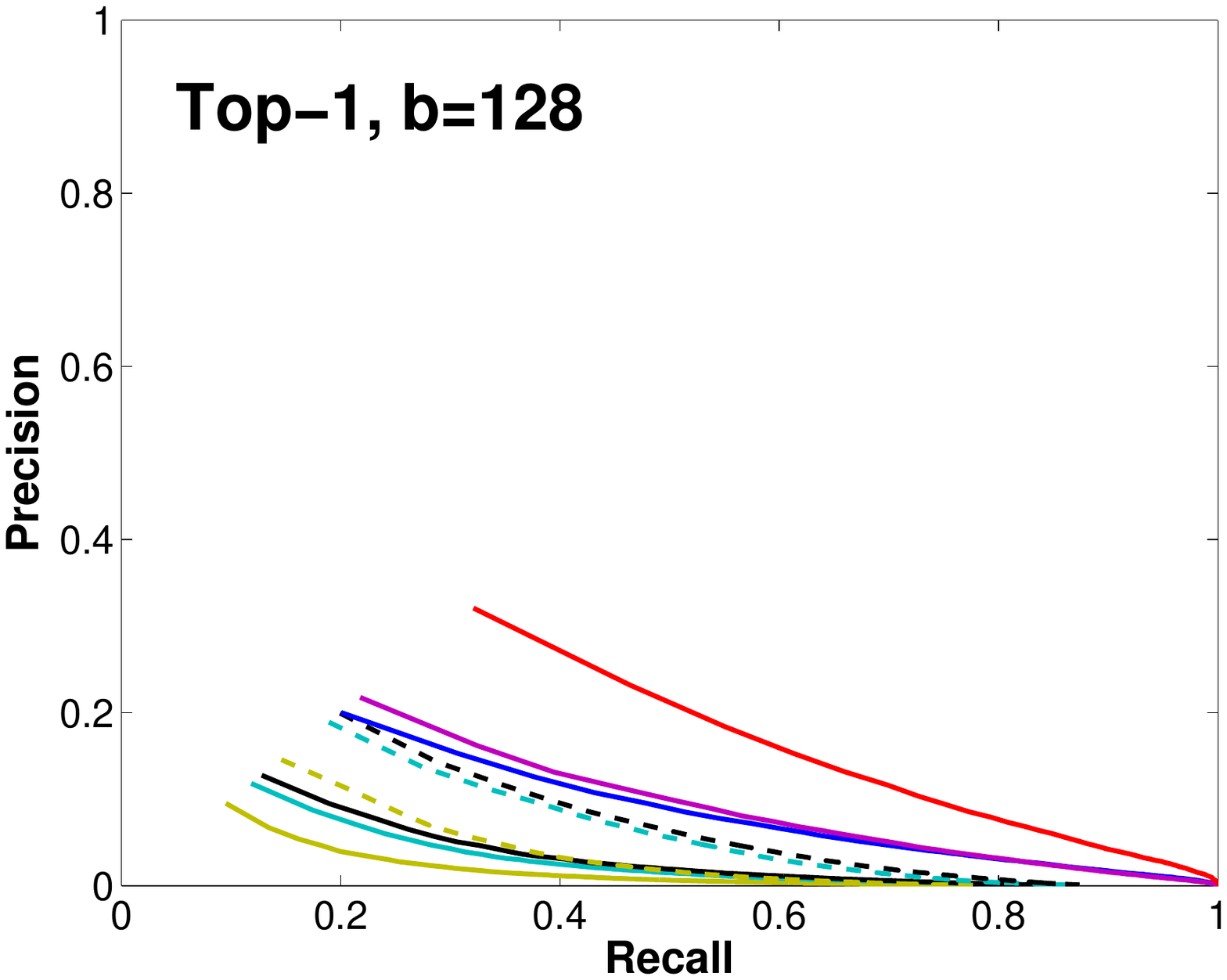} \hskip -3pt
\includegraphics[trim=0.9in 2.5in 0.9in 2.5in,clip,width=.248 \textwidth]{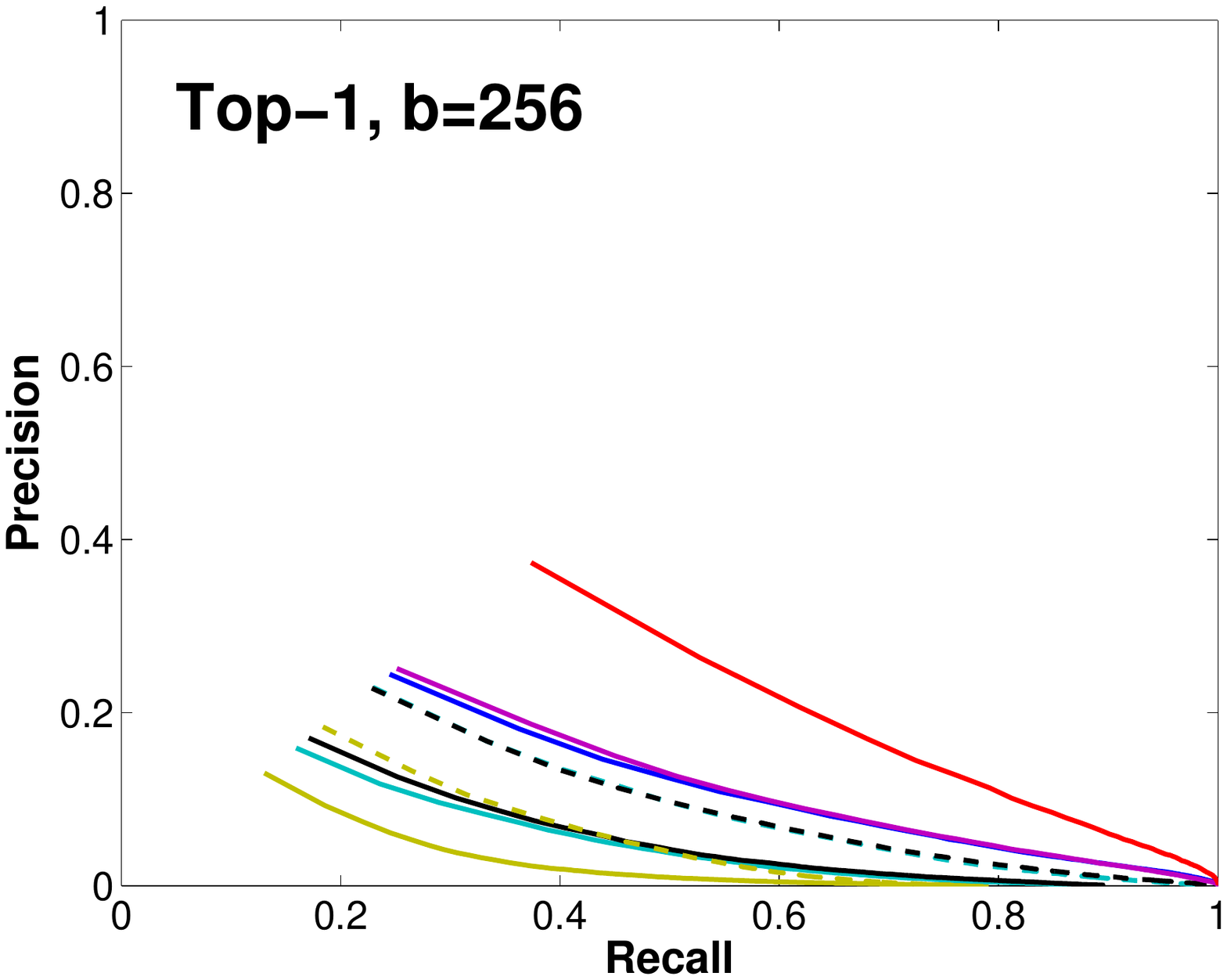} \hskip -3pt
\includegraphics[trim=0.9in 2.5in 0.9in 2.5in,clip,width=.248 \textwidth]{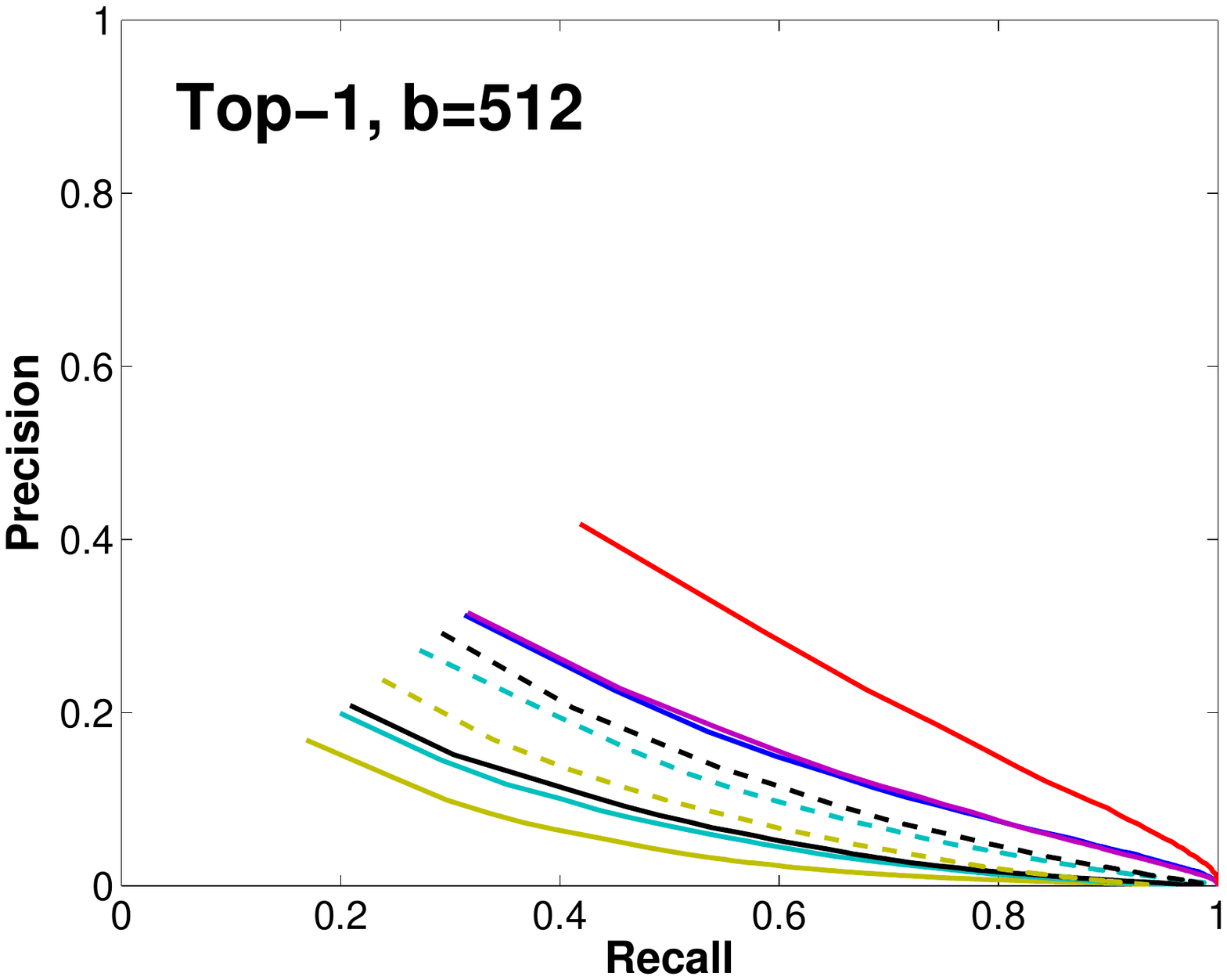}
\vskip -9pt
\includegraphics[trim=0.6in 2.5in 0.9in 2.5in,clip,width=.26 \textwidth]{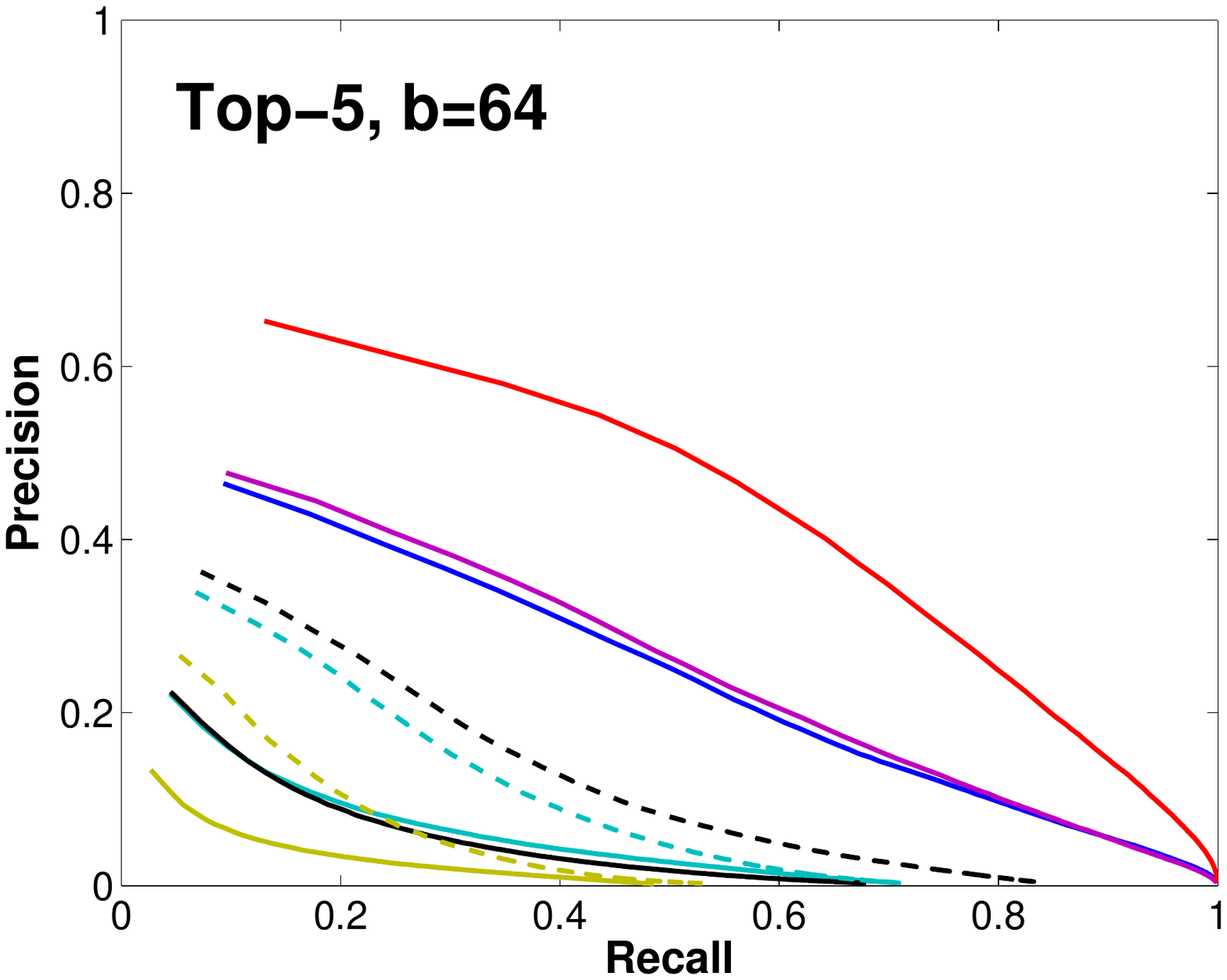} \hskip -3pt
\includegraphics[trim=0.9in 2.5in 0.9in 2.5in,clip,width=.248 \textwidth]{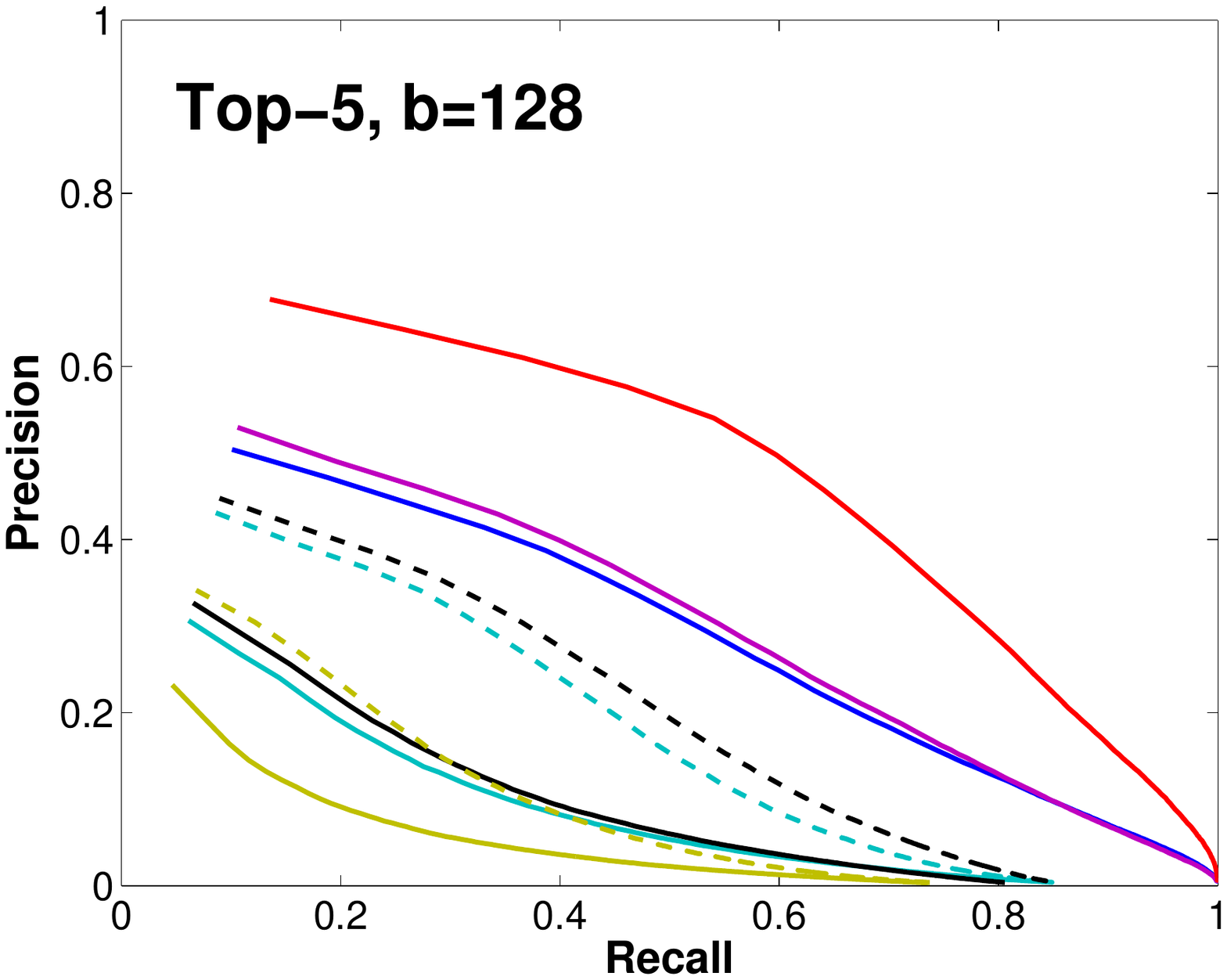} \hskip -3pt
\includegraphics[trim=0.9in 2.5in 0.9in 2.5in,clip,width=.248 \textwidth]{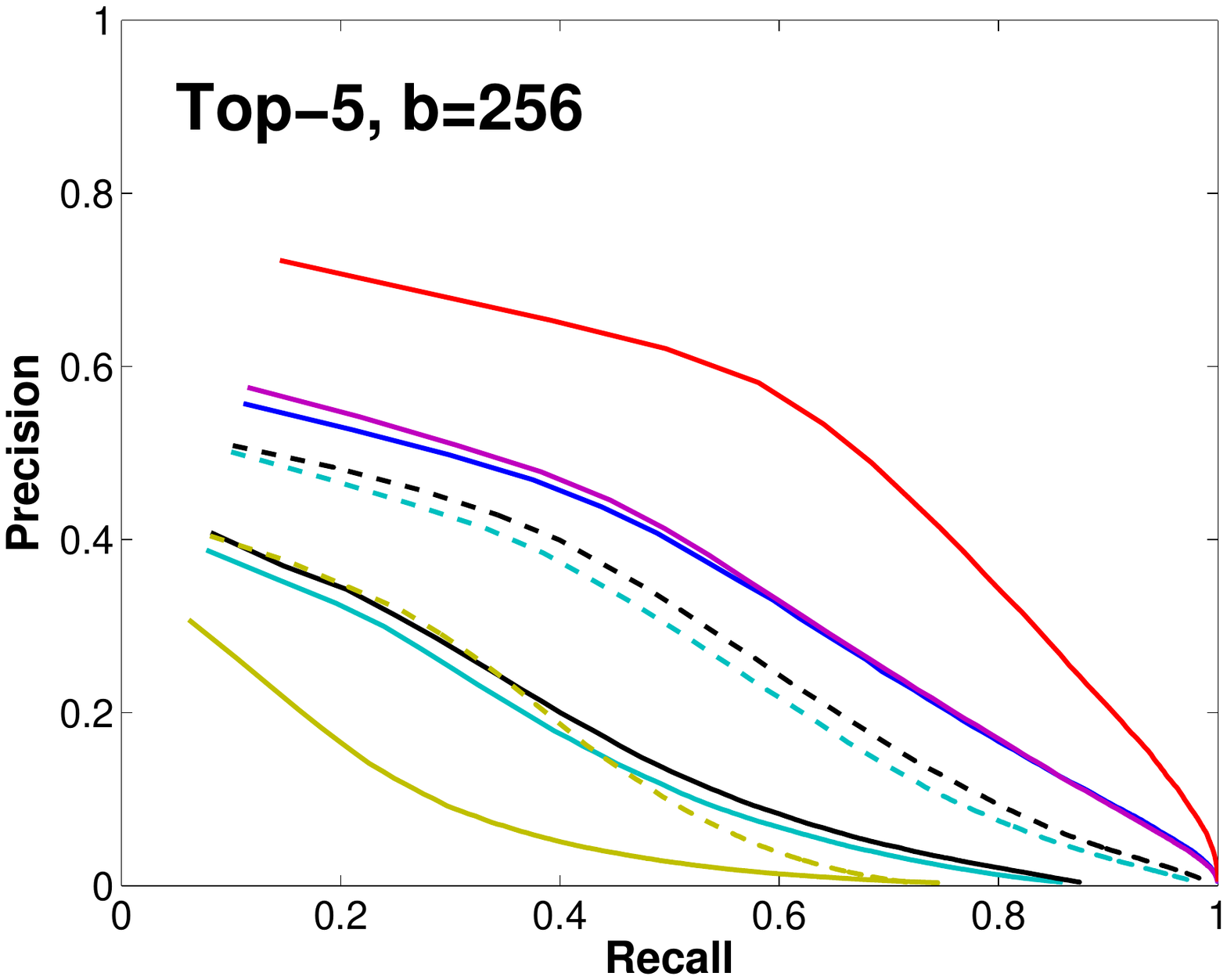} \hskip -3pt
\includegraphics[trim=0.9in 2.5in 0.9in 2.5in,clip,width=.248 \textwidth]{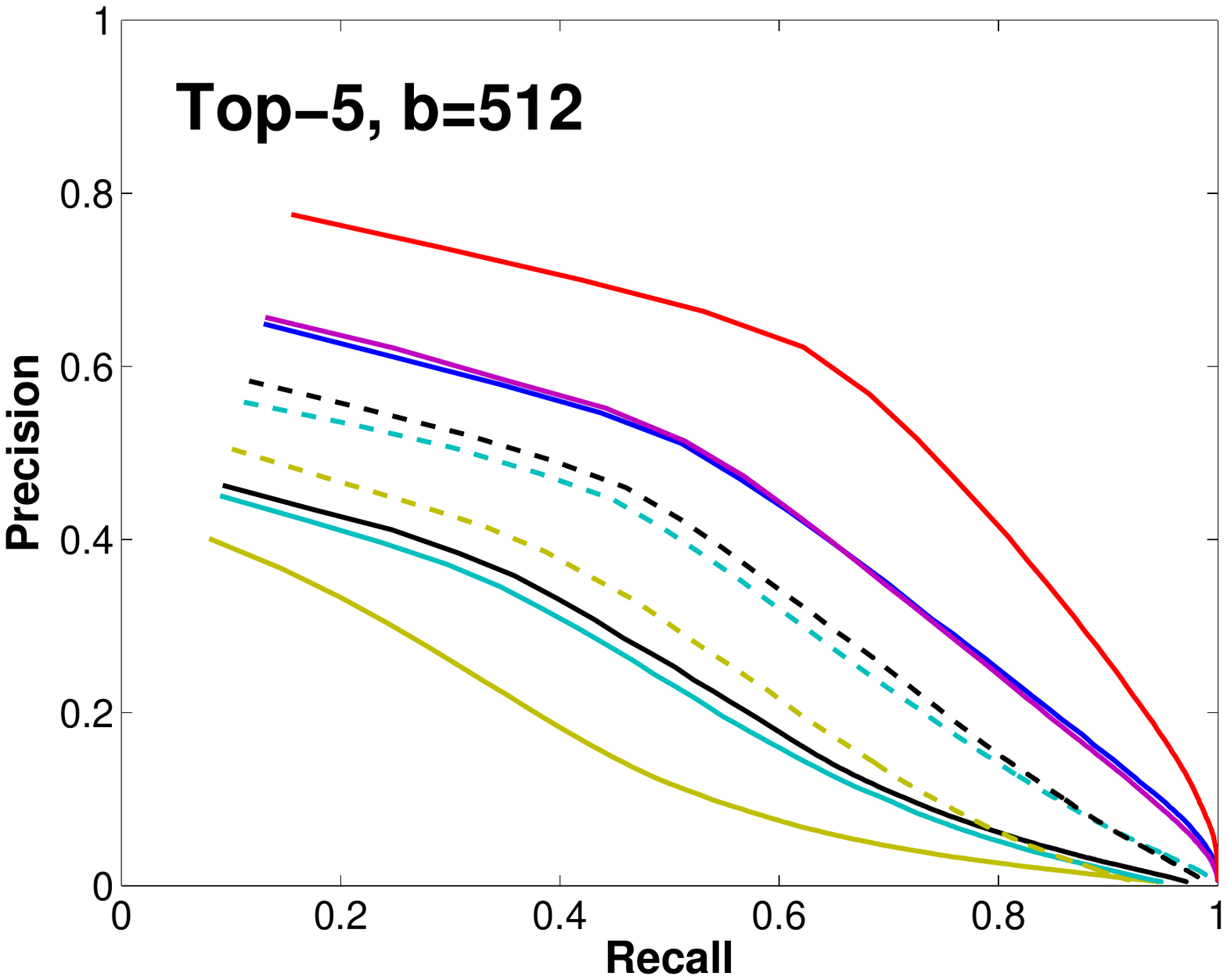}
\vskip -9pt
\includegraphics[trim=0.6in 2.5in 0.9in 2.5in,clip,width=.26 \textwidth]{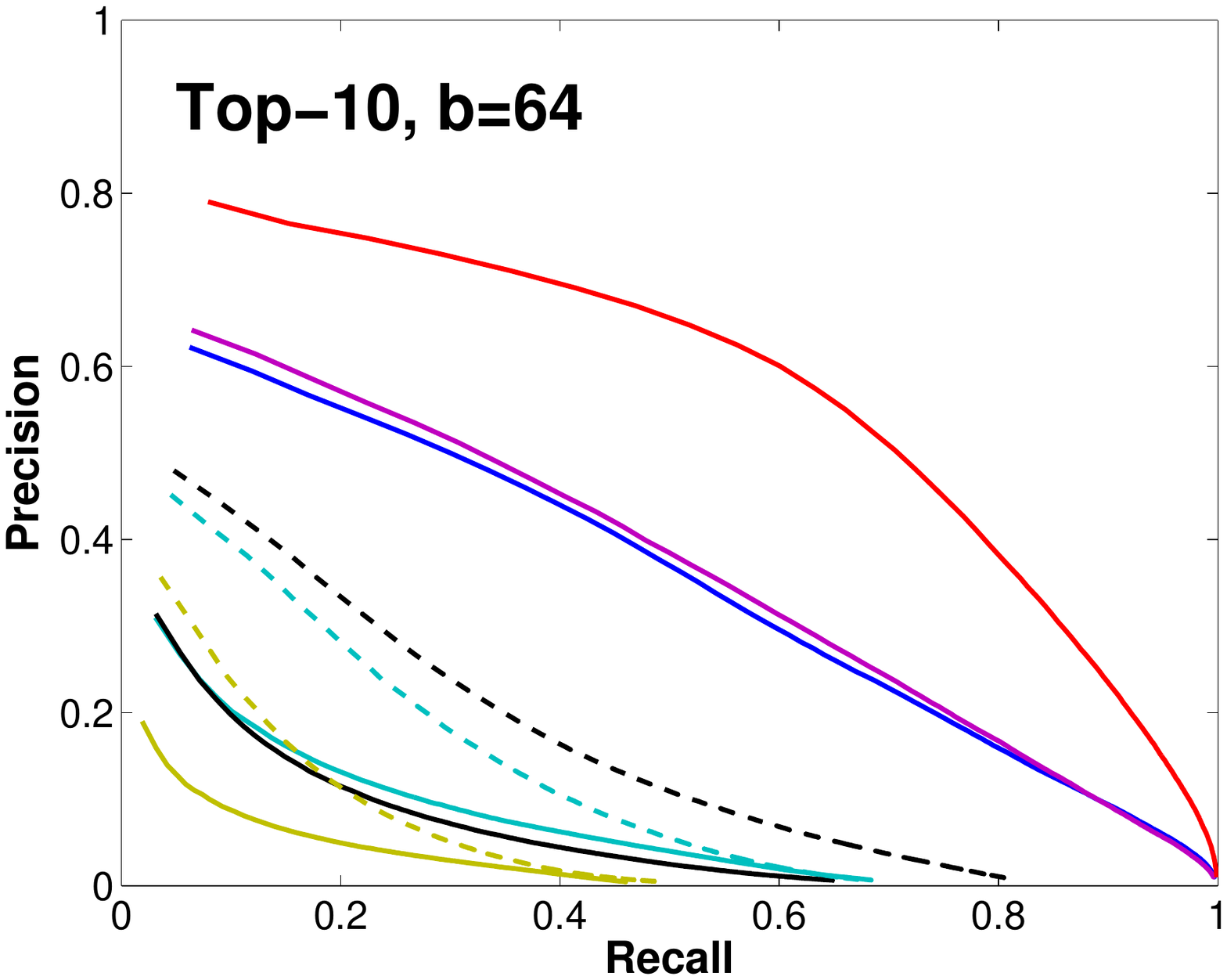} \hskip -3pt
\includegraphics[trim=0.9in 2.5in 0.9in 2.5in,clip,width=.248 \textwidth]{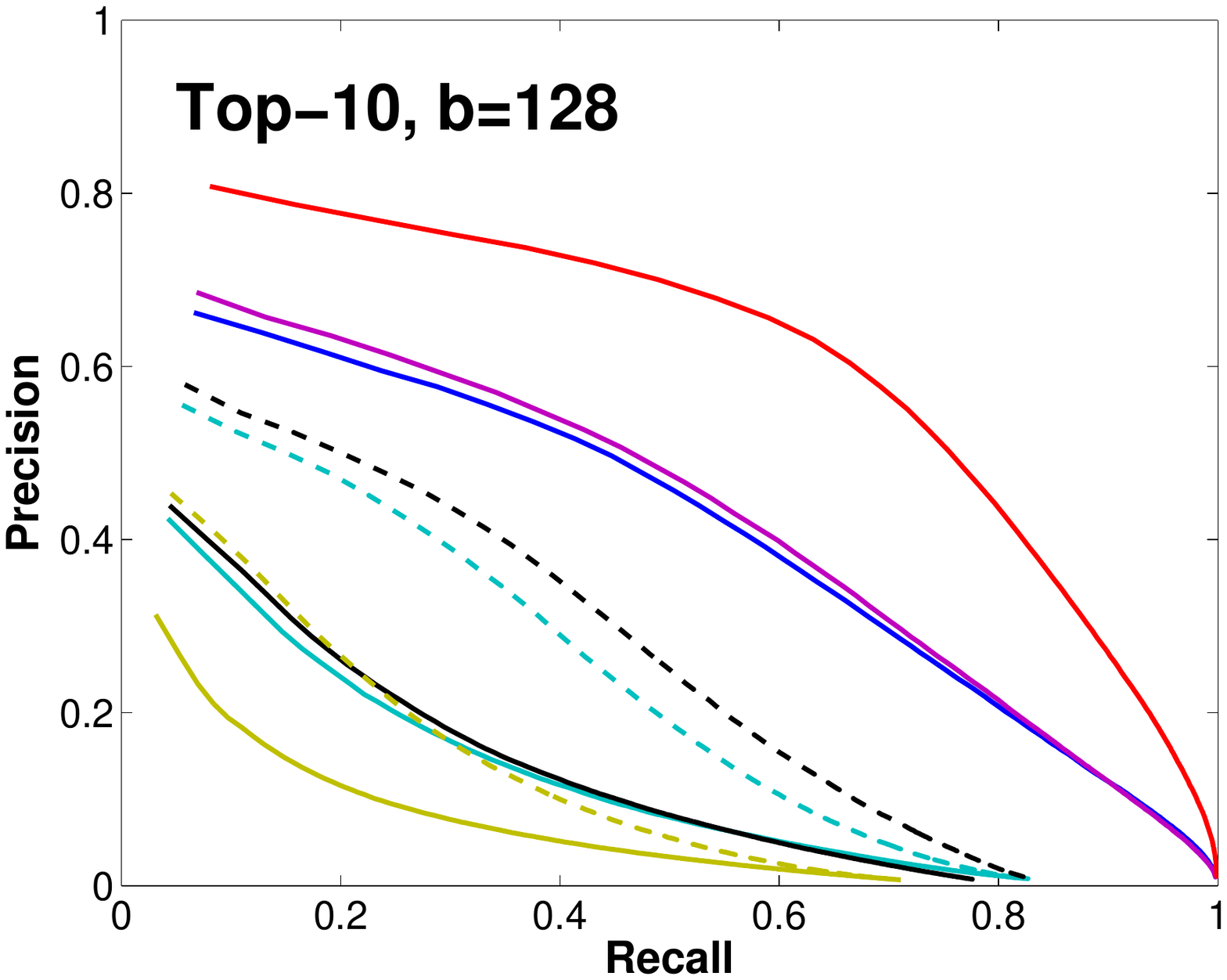} \hskip -3pt
\includegraphics[trim=0.9in 2.5in 0.9in 2.5in,clip,width=.248 \textwidth]{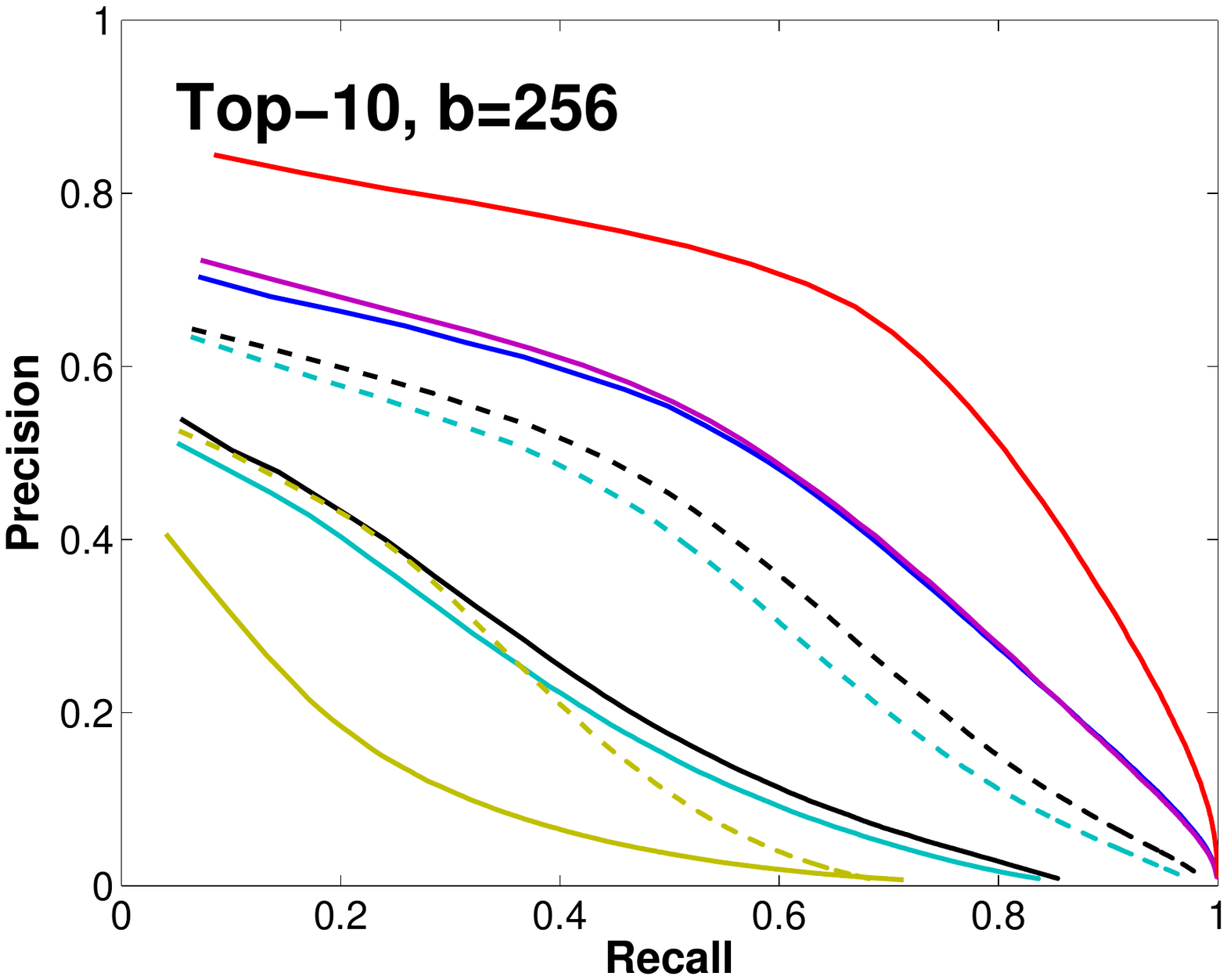} \hskip -3pt
\includegraphics[trim=0.9in 2.5in 0.9in 2.5in,clip,width=.248 \textwidth]{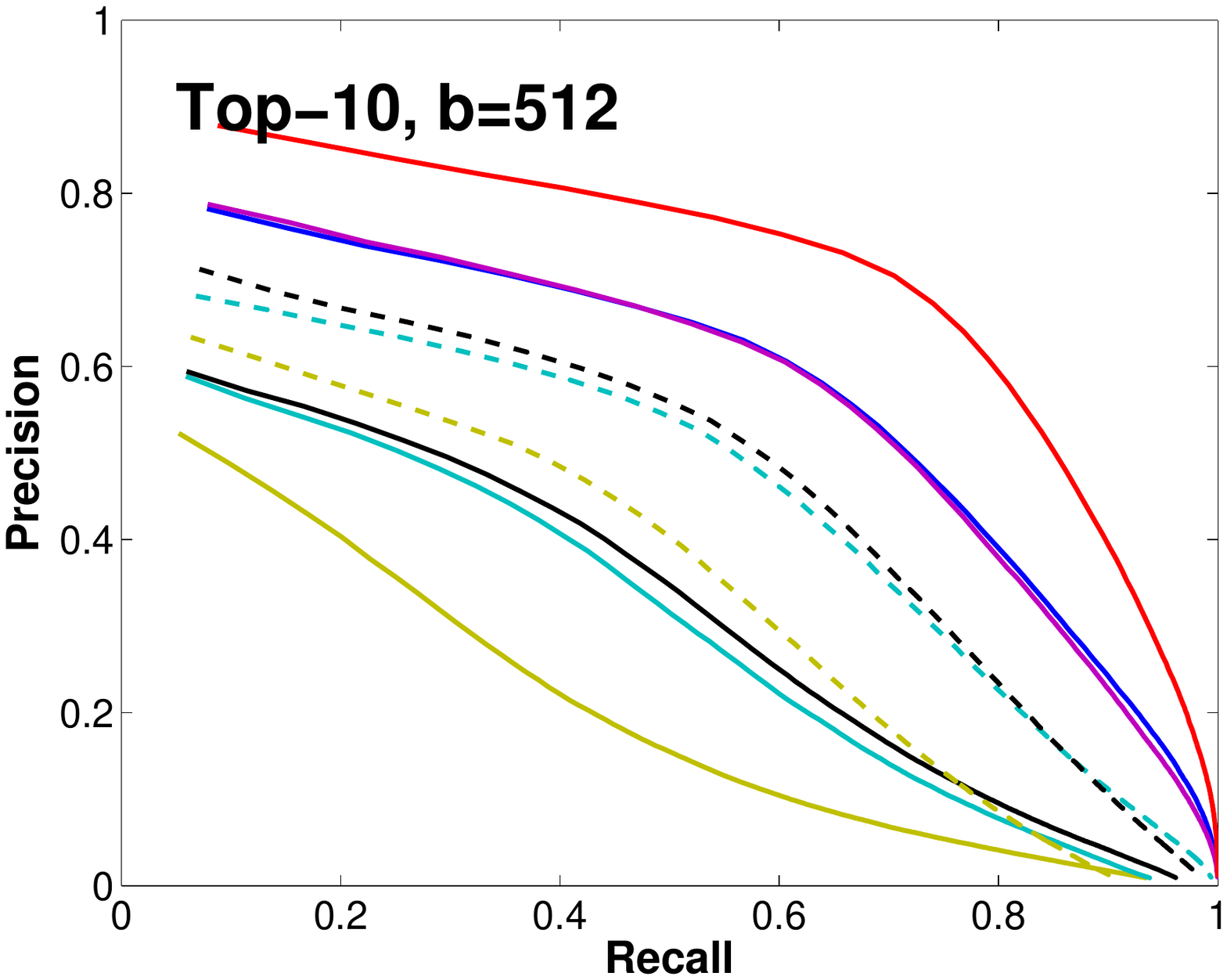}
\vskip -3pt
\subcaption{Netflix dataset}
\end{subfigure}
\begin{subfigure}[c]{1 \textwidth}
\end{subfigure}
\caption{\label{fig:fixed} 
Precision Recall curves (higher is better) for different methods on Movielens and Netflix datasets, retrieving Top-1, 5 and 10 items. \textbf{Baselines:} \emph{Signed ALSH}~\cite{shrivastava014e}, \emph{L2 ALSH}~\cite{shrivastava2014asymmetric} and \emph{Simple LSH}~\cite{neyshabur2014simpler}. \textbf{Proposed Methods:} \emph{QUIP-cov(x)}, \emph{QUIP-cov(q)}, \emph{QUIP-opt}. Curves for fixed bit experiments are plotted in solid line for both the baselines and proposed methods, where the number of bits used are $\mathbf{b = 64, 128, 256, 512}$ respectively, from left to right. Curves for fixed time experiment are plotted in dashed lines. The fixed time plots are the same as the fixed bit plots for the proposed methods. For the baseline methods, the number of bits used in fixed time experiments are $\mathbf{b = 192, 384, 768, 1536}$ respectively, so that their running time is comparable with that of the proposed methods.}
\end{figure}

We conducted experiments with 4 datasets which are summarized below:
\begin{description}[leftmargin=10pt,itemsep=-0.3ex, topsep=0pt]
\item[Movielens] This dataset consists of user ratings collected by the MovieLens site from web users. We use the same SVD setup as described in the ALSH paper~\cite{shrivastava2014asymmetric} and extract 150 latent dimensions from SVD results. This dataset contains 10,681 database vectors and 71,567 query vectors.

\item[Netflix] The Netflix dataset comes from the Netflix Prize challenge~\cite{Bennett07thenetflix}. It contains 100,480,507 ratings that users gave to Netflix movies. We process it in the same way as suggested by~\cite{shrivastava2014asymmetric}. That leads to 300  dimensional data. There are 17,770 database vectors and 480,189 query vectors.

\item[ImageNet] This dataset comes from the state-of-the-art GoogLeNet~\cite{szegedy2014cvpr} image classifier trained on ImageNet\footnote{The original paper ensembled 7 models and used 144 different crops. In our experiment, we focus on one global crop using one model.}. The goal is to speed up the maximum dot-product search in the last i.e., classification layer. Thus, the weight vectors for different categories form the database while the query vectors are the last hidden layer embeddings from the ImageNet validation set. The data has 1025 dimensions (1024 weights and 1 bias term). There are 1,000 database and 49,999 query vectors.

\item[VideoRec] This dataset consists of embeddings of user interests~\cite{davidson2010recsys}, trained via a deep neural network to predict a set of relevant videos for a user. The number of videos in the repository is 500,000. The network is trained with a multi-label logistic loss. As for the \emph{ImageNet} dataset, the last hidden layer embedding of the network is used as query vector, and the classification layer weights are used as database vectors. The goal is to speed up the maximum dot product search between a query and 500,000 database vectors. Each database vector has 501 dimensions (500 weights and 1 bias term). The query set contains 1,000 vectors.
\end{description}
Following~\cite{shrivastava2014asymmetric}, we focus on retrieving Top-1, 5 and 10 highest inner product neighbors for Movielens and Netflix experiments. For ImageNet dataset, we retrieve top-5 categories as common in the literature. For the VideoRec dataset, we retrieve Top-50 videos for recommendation to a user. We experiment with three variants our technique: (1) \emph{\textbf{QUIP-cov(x)}}: uses only database vectors at training, and replaces $\Sigma_{\mathbf{Q}}$ by $\Sigma_{X}$ in the \emph{k-Means} like codebook learning in Sec.~\ref{sec:kmeans}, (2) \emph{\textbf{QUIP-cov(q)}}: uses $\Sigma_{\mathbf{Q}}$ estimated from a held-out exemplar query set for \emph{k-Means} like codebook learning, and (3) \emph{\textbf{QUIP-opt}}: uses full optimization based quantization (Sec.~\ref{sec:opt}). We compare the performance (precision-recall curves) with 3 state-of-the-art methods: (1) \emph{\textbf{Signed ALSH}}~\cite{shrivastava2014asymmetric}, (2) \emph{\textbf{L2 ALSH}}~\cite{shrivastava2014asymmetric}\footnote{The recommended parameters $m=3, U_0=0.85, r=2.5$ were used in the implementation.}; and (3) \emph{\textbf{Simple LSH}}~\cite{neyshabur2014simpler}. We also compare against the PCA-tree version adapted to inner product search as proposed in~\cite{bachrach2014speeding}, which has shown better results than IP-tree~\cite{ram2012kdd}. The proposed quantization based methods perform much better than PCA-tree as shown in the appendix. 

We conduct two sets of experiments: (i) \emph{fixed bit} - the number of bits used by all the techniques is kept the same, (ii) \emph{fixed time} - the time taken by all the techniques is fixed to be the same. In the fixed bit experiments, we fix the number of bits to be $b = 64, 128, 256, 512$. 
For all the \emph{QUIP} variants, the codebook size for each subspace, C, was fixed to be $256$, leading to a 8-bit representation of a database vector in each subspace. The number of subspaces (i.e., blocks) was varied to be $k=8,16,32,64$ leading to $64, 128, 256, 512$ bit representation, respectively. For the fixed time experiments, we first note that the proposed \emph{QUIP} variants use table lookup based distance computation while the LSH based techniques use POPCNT-based Hamming distance computation. Depending on the number of bits used, we found POPCNT to be 2 to 3 times faster than table lookup. Thus, in the fixed-time experiments, we increase the number of bits for LSH-based techniques by 3 times to ensure that the time taken by all the methods is the same.

\begin{figure}
\vskip -12pt
\centering
\begin{subfigure}[c]{1 \textwidth}
\centering
\includegraphics[trim=0.6in 2.5in 0.9in 2.5in,clip,width=.26 \textwidth]{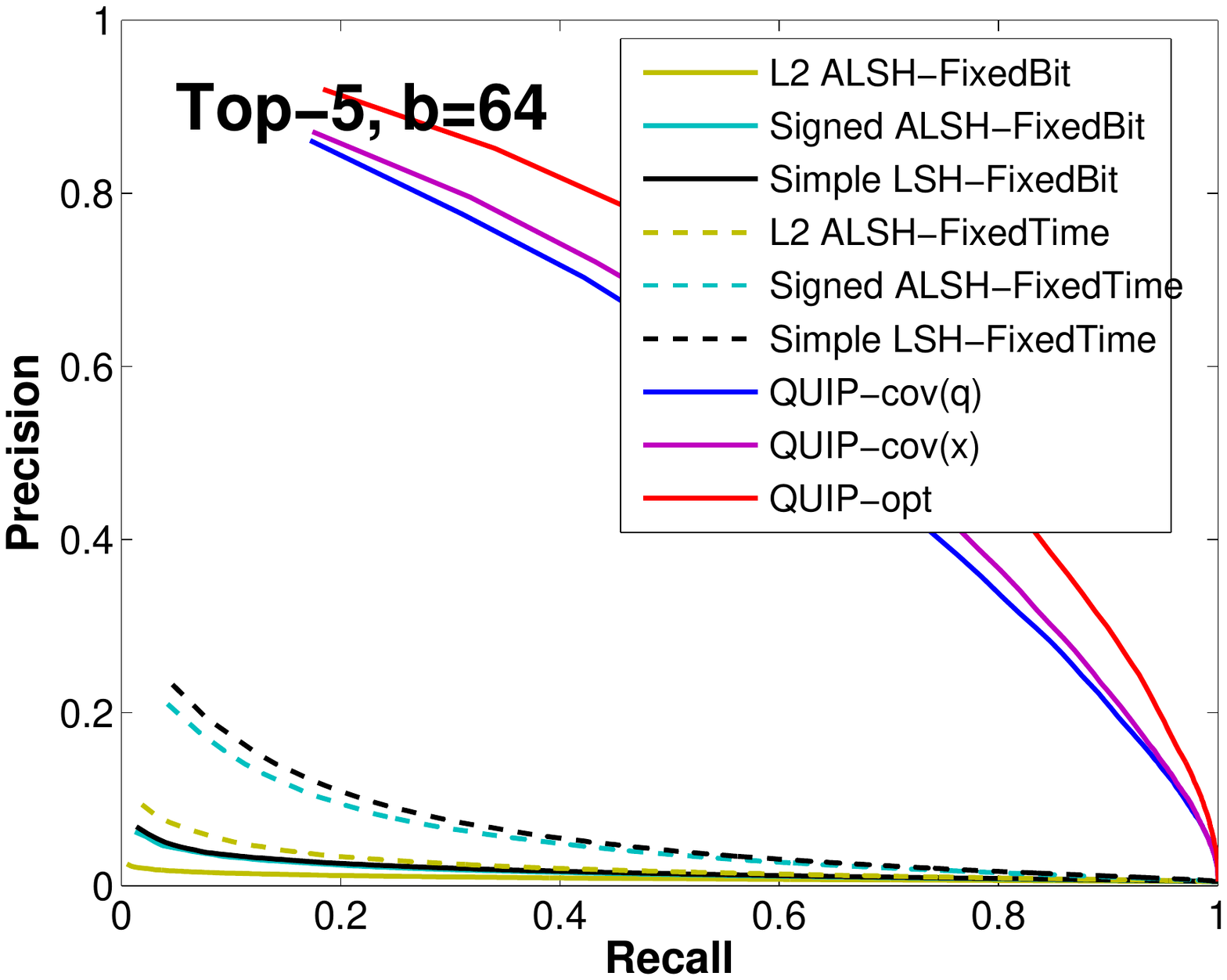} \hskip -3pt
\includegraphics[trim=0.9in 2.5in 0.9in 2.5in,clip,width=.248 \textwidth]{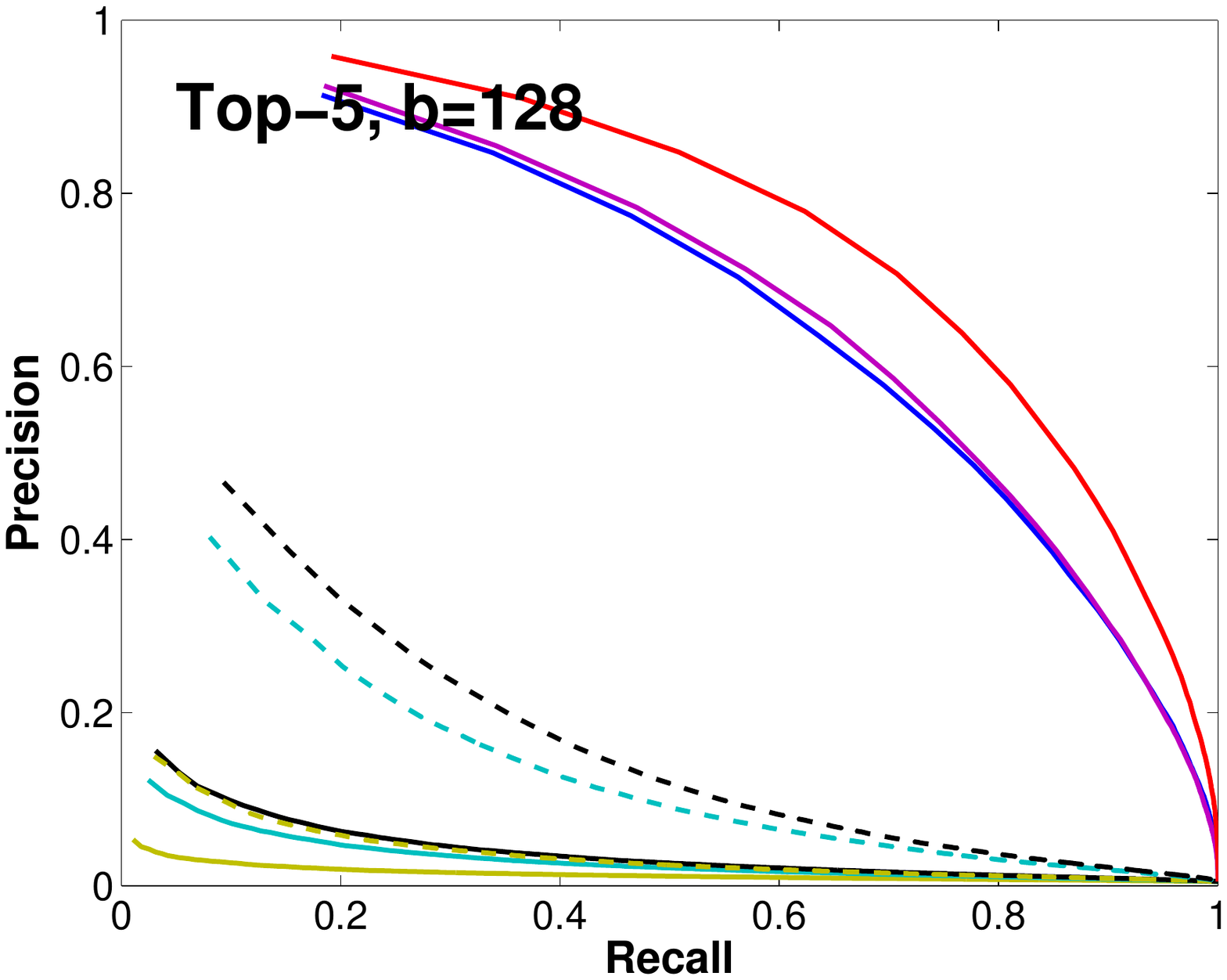} \hskip -3pt
\includegraphics[trim=0.9in 2.5in 0.9in 2.5in,clip,width=.248 \textwidth]{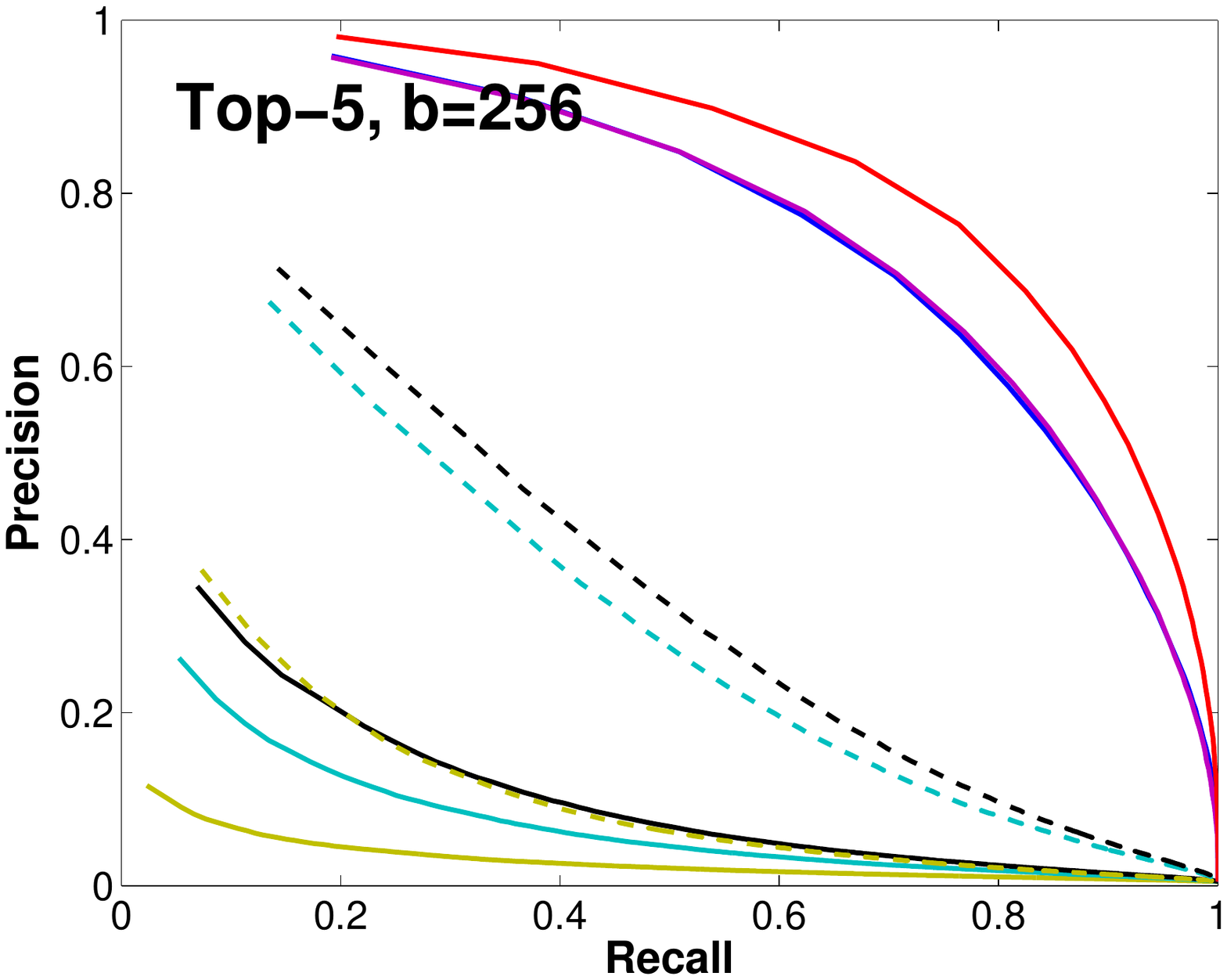} \hskip -3pt
\includegraphics[trim=0.9in 2.5in 0.9in 2.5in,clip,width=.248 \textwidth]{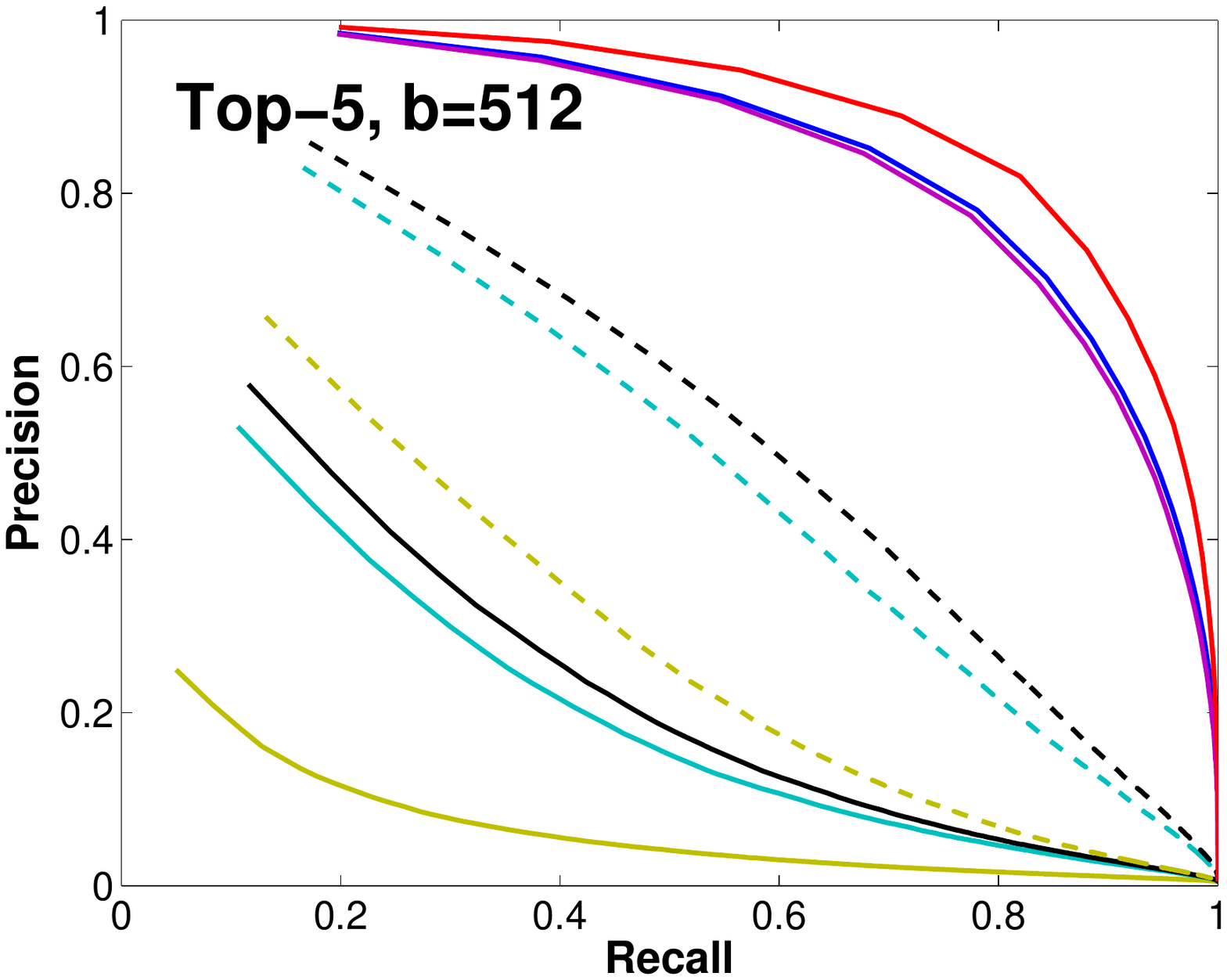}
\vskip -3pt
\subcaption{ImageNet dataset, retrieval of Top 5 items.}
\end{subfigure}
\begin{subfigure}[c]{1 \textwidth}
\centering
\includegraphics[trim=0.6in 2.5in 0.9in 2.5in,clip,width=.26 \textwidth]{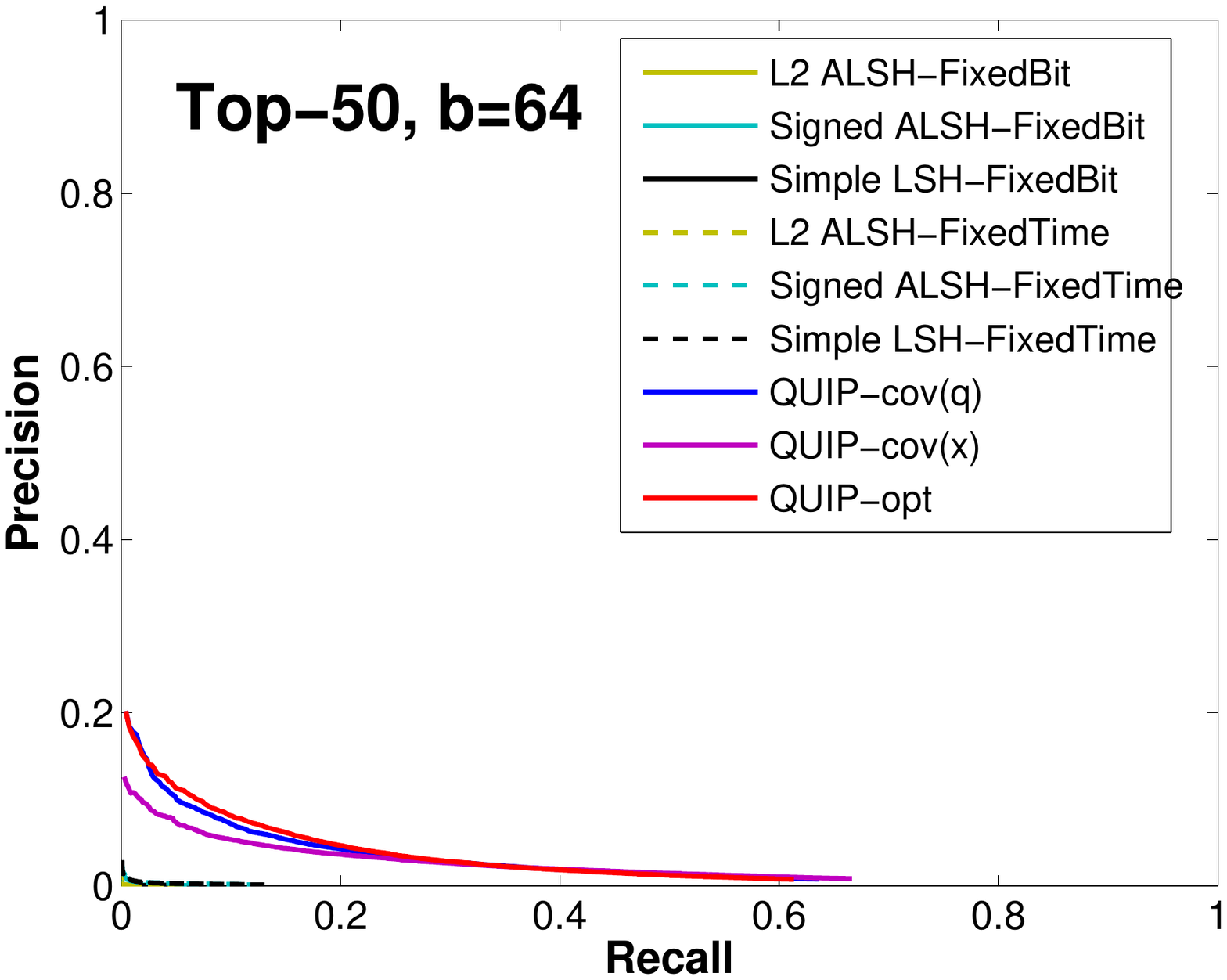} \hskip -3pt
\includegraphics[trim=0.9in 2.5in 0.9in 2.5in,clip,width=.248 \textwidth]{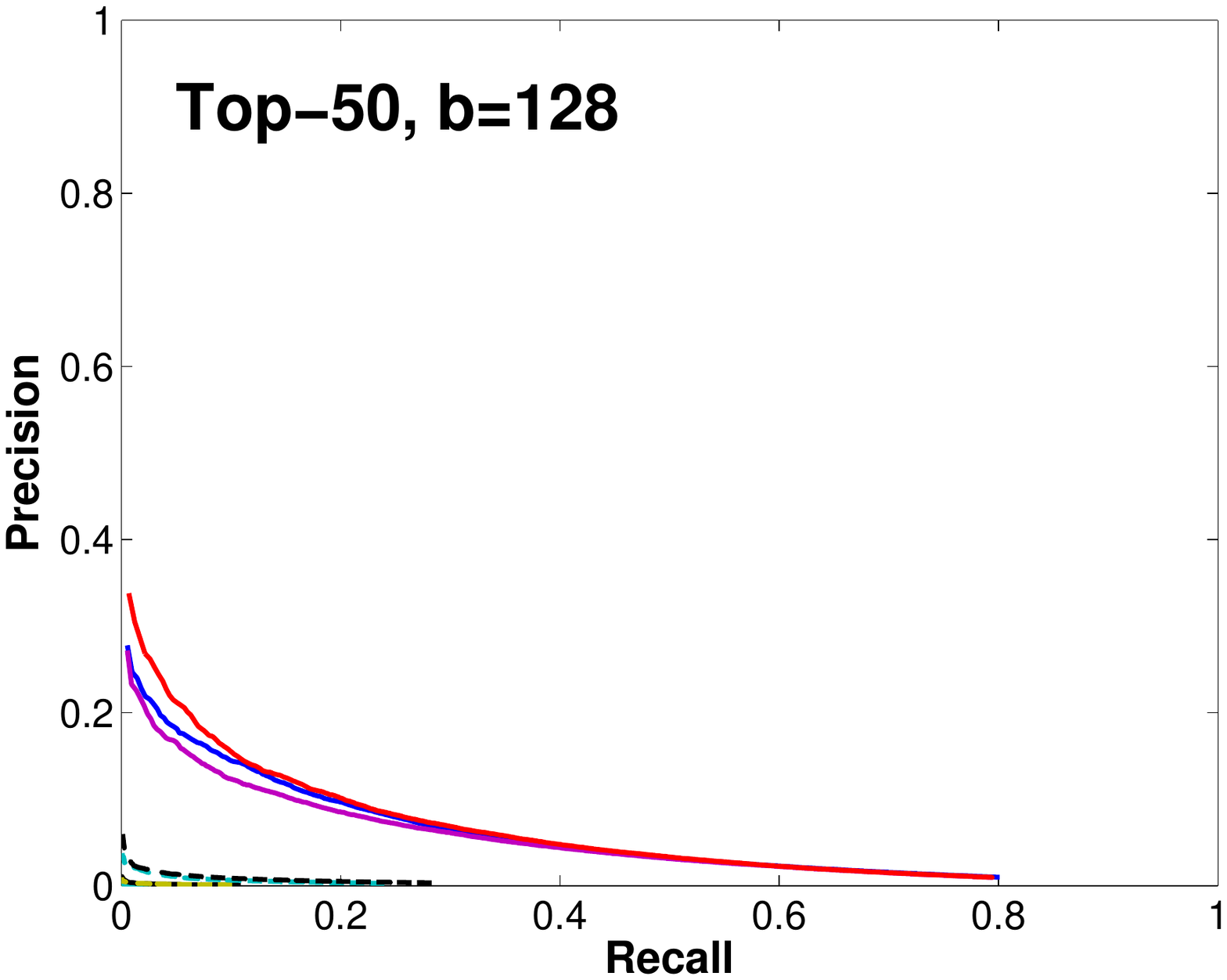} \hskip -3pt
\includegraphics[trim=0.9in 2.5in 0.9in 2.5in,clip,width=.248 \textwidth]{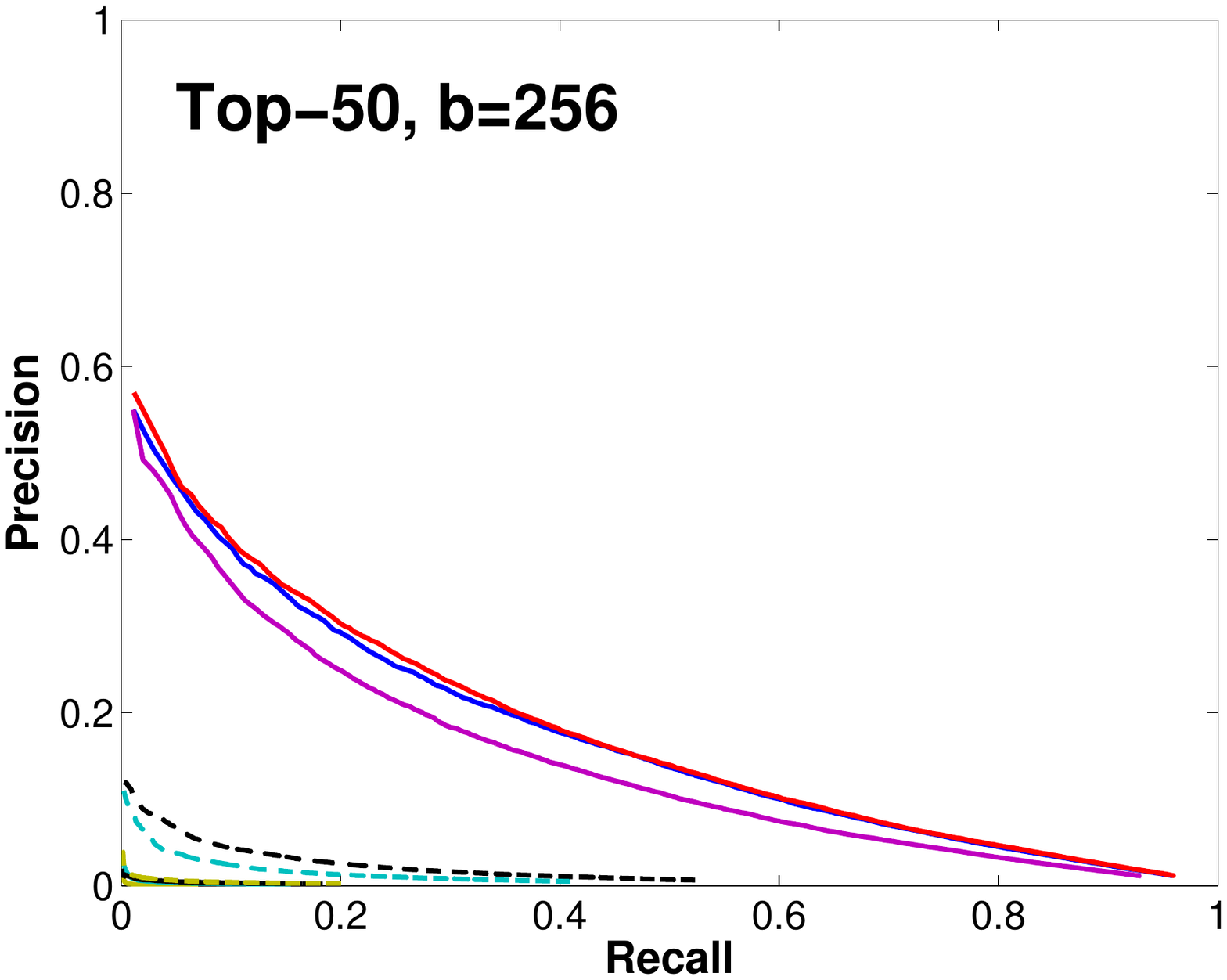} \hskip -3pt
\includegraphics[trim=0.9in 2.5in 0.9in 2.5in,clip,width=.248 \textwidth]{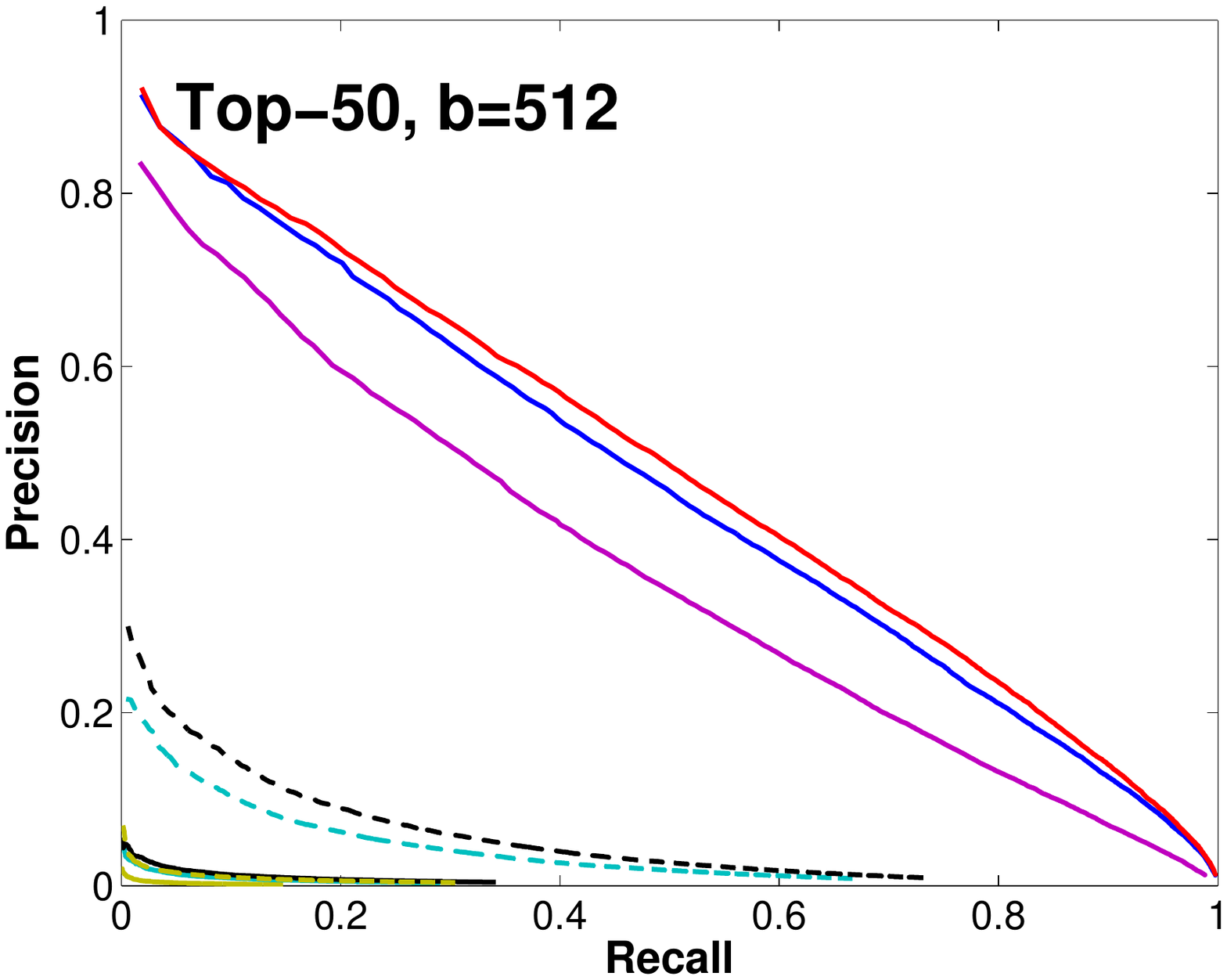}
\vskip -3pt
\subcaption{VideoRec dataset, retrieval of Top 50 items.}
\end{subfigure}
\vskip -6pt
\begin{subfigure}[c]{1 \textwidth}
\end{subfigure}
\caption{\label{fig:fixed_add}
Precision Recall curves for \emph{ImageNet} and \emph{VideoRec}. See appendix for more results.}
\vskip -12pt
\end{figure}


Figure~\ref{fig:fixed} shows the precision recall curves for \emph{Movielens} and \emph{Netflix}, and Figure~\ref{fig:fixed_add} shows the same for the \emph{ImageNet} and \emph{VideoRec} datasets. All the quantization based approaches outperform LSH based methods significantly when all the techniques use the same number of bits. Even in the fixed time experiments, the quantization based approaches remain superior to the LSH-based approaches (shown with dashed curves), even though the former uses 3 times less bits than latter, leading to significant reduction in memory footprint.
Among the quantization methods, \emph{QUIP-cov(q)} typically performs better than \emph{QUIP-cov(x)}, but the gap in performance is not that large. In theory, the non-centered covariance matrix of the queries ($\Sigma_{Q}$) can be quite different than that of the database ($\Sigma_{X}$), leading to drastically different results. However, the comparable performance implies that it is often safe to use $\Sigma_{X}$ when learning a codebook. On the other hand, when a small set of example queries is available, \emph{QUIP-opt} outperforms both \emph{QUIP-cov(x)} and \emph{QUIP-cov(q)} on all four datasets. This is because it learns the codebook with constraints that steer learning towards  retrieving the maximum dot product neighbors in addition to minimizing the quantization error. The overall training for \emph{QUIP-opt} was quite fast, requiring 3 to 30 minutes using a single-thread implementation, depending on the dataset size.  


\vspace{-2mm}

\section{Tree-Quantization Hybrids for Large Scale Search}
The quantization based inner product search techniques described above provide a significant speedup over the brute force search while retaining high accuracy. However, the search complexity is still linear in the number of database points similar to that for the binary embedding methods that do exhaustive scan using Hamming distance ~\cite{shrivastava014e}. When the database size is very large, such a linear scan even with fast computation may not be able to provide the required search efficiency. In this section, we describe a simple procedure to further enhance the speed of QUIPS based on data partitioning. The basic idea of tree-quantization hybrids is to combine tree-based recursive data partitioning with QUIPS applied to each partition. At the training time, one first learns a locality-preserving tree such as hierarchical k-means tree, followed by applying QUIPS to each partition. In practice only a shallow tree is learned such that each leaf contains a few thousand points. Of course, a special case of tree-based partitioners is a flat partitioner such as k-means. At the query time, a query is assigned to more than one partition to deal with the errors caused by hard partitioning of the data. This soft assignment of query to multiple partitions is crucial for achieving good accuracy for high-dimensional data.
 

In the \emph{VideoRec} dataset, where $n=500,000$, the quantization approaches (including \emph{QUIP-cov(x), QUIP-cov(q), QUIP-opt}) reduce the search time by a factor of $7.17$, compared to that of brute force search. The tree-quantization hybrid approaches (\emph{Tree-QUIP-cov(x), Tree-QUIP-cov(q), Tree-QUIP-opt}) use 2000 partitions, and each query is assigned to the nearest 100 partitions based on its dot-product with the partition centers. These Tree-QUIP hybrids lead to a further speed up of $5.97$x over  QUIPS, leading to an overall end-to-end speed up of $42.81$x over brute force search. To illustrate the effectiveness of the hybrid approach, we plot the precision recall curve in \emph{Fixed-bit} and \emph{Fixed-time} experiment on \emph{VideoRec} in Figure~\ref{fig:fixed_tree}. From the Fixed-bit experiments, Tree-Quantization methods have almost the same accuracy as their non-hybrid counterparts (note that the curves almost overlap in Fig. \ref{fig:fixed_tree}(a) for these two versions), while resulting in about 6x speed up. From the fixed-time experiments, it is clear that with the same time budget the hybrid approaches return much better results because they do not scan all the datapoints when searching. 

\begin{figure}
\vskip -12pt
\centering
\begin{subfigure}[c]{0.48 \textwidth}
\includegraphics[trim=1.3in 3.3in 1.3in 3.3in,clip,width=1 \textwidth]{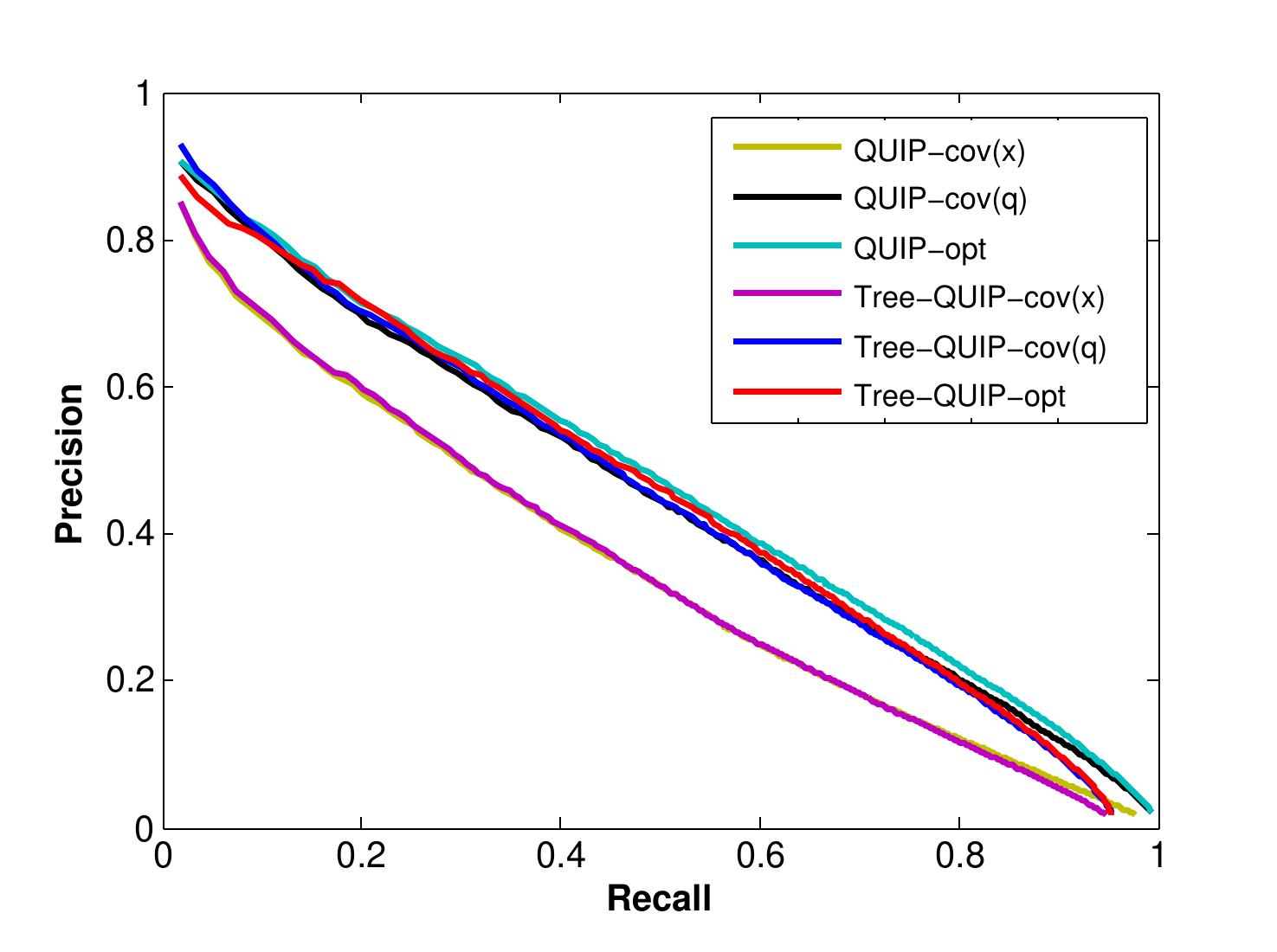} 
\subcaption{Fixed-bit experiment.}
\end{subfigure}
\begin{subfigure}[c]{0.48 \textwidth}
\includegraphics[trim=1.3in 3.3in 1.3in 3.3in,clip,width=1 \textwidth]{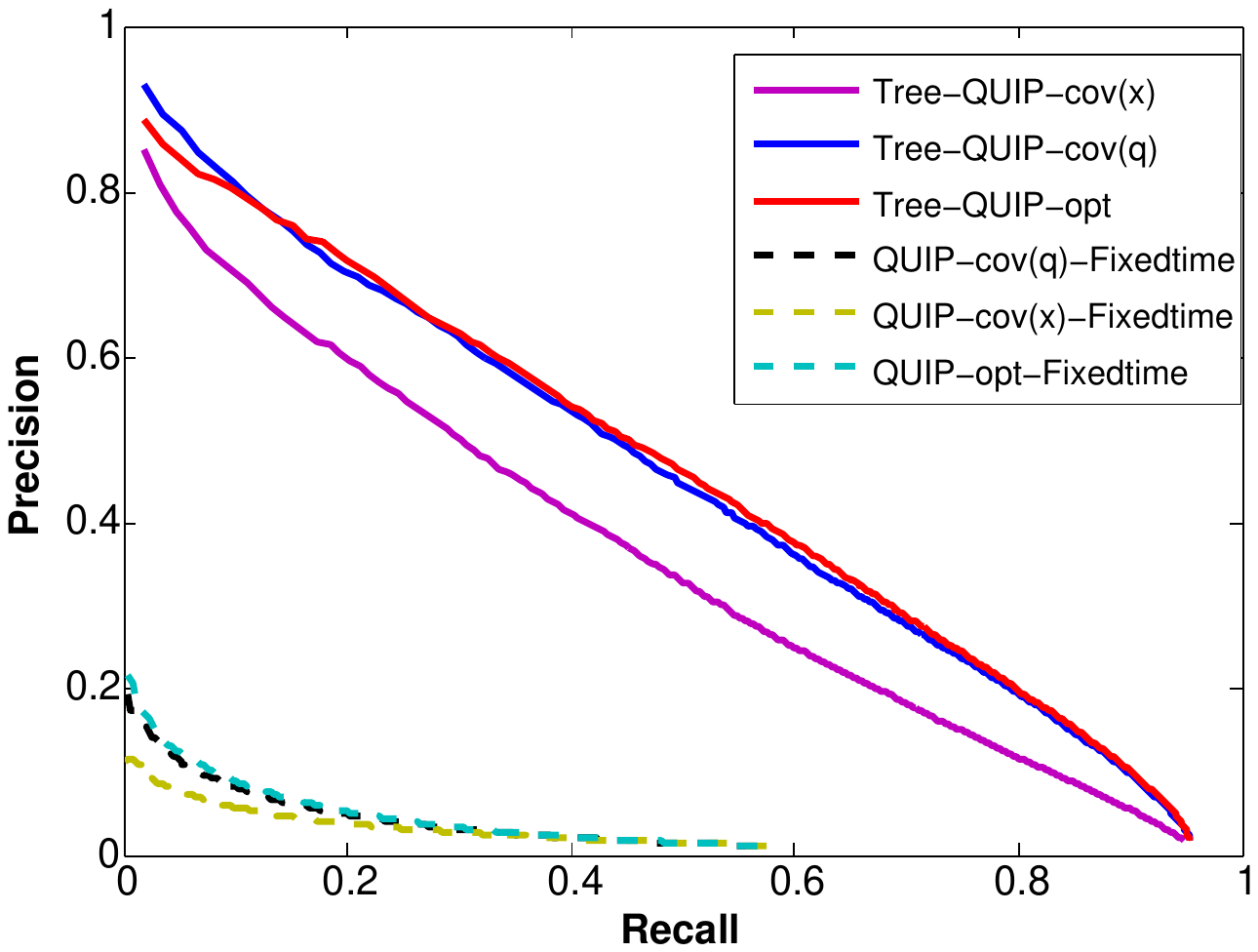} 
\subcaption{Fixed-time experiment.}
\end{subfigure}
\hskip -3pt
\vskip -3pt
\caption{\label{fig:fixed_tree}
Precision recall curves on \emph{VideoRec} dataset, retrieving Top-50 items, comparing quantization based methods and tree-quantization hybrid methods. In (a), we conduct fixed bit comparison where both the non-hybrid methods and hybrid methods use the same 512 bits. The non-hybrid methods are considerable slower in this case (5.97x). In (b), we conduct fixed time experiment, where the time of retrieval is fixed to be the same as taken by the hybrid methods (2.256ms). The non-hybrid approaches give much lower accuracy in this case. }
\vskip -12pt
\end{figure}

\section{Conclusion}
\label{sec:con}
\vspace{-3mm}
We have described a quantization based approach for fast approximate inner product search, which relies on robust learning of codebooks in multiple subspaces. One of the proposed variants leads to a very simple kmeans-like learning procedure and yet outperforms the existing state-of-the-art by a significant margin. We have also introduced novel constraints in the quantization error minimization framework that lead to even better codebooks, tuned to the problem of highest dot-product search. Extensive experiments on retrieval and classification tasks show the advantage of the proposed method over the existing techniques. In the future, we would like to analyze the theoretical guarantees associated with the constrained optimization procedure. In addition, in the tree-quantization hybrid approach, the tree partitioning and the quantization codebooks are trained separately. As a future work, we will consider training them jointly.

\section{Appendix}
\subsection{Additional Experimental Results}
The results on \emph{ImageNet} and \emph{VideoRec} datasets for different number of top neighbors and different number of bits are shown in Figure~\ref{fig:top}. In addition, we compare the performance of our approach against \emph{PCA-Tree}. The recall curves with respect to different number of  returned neighbors are shown in Figure~\ref{fig:recall}. 

\begin{figure}
\vskip -6pt
\centering
\begin{subfigure}[c]{1 \textwidth}
\centering
\vskip -9pt
\includegraphics[trim=0.6in 2.5in 0.9in 2.5in,clip,width=.26 \textwidth]{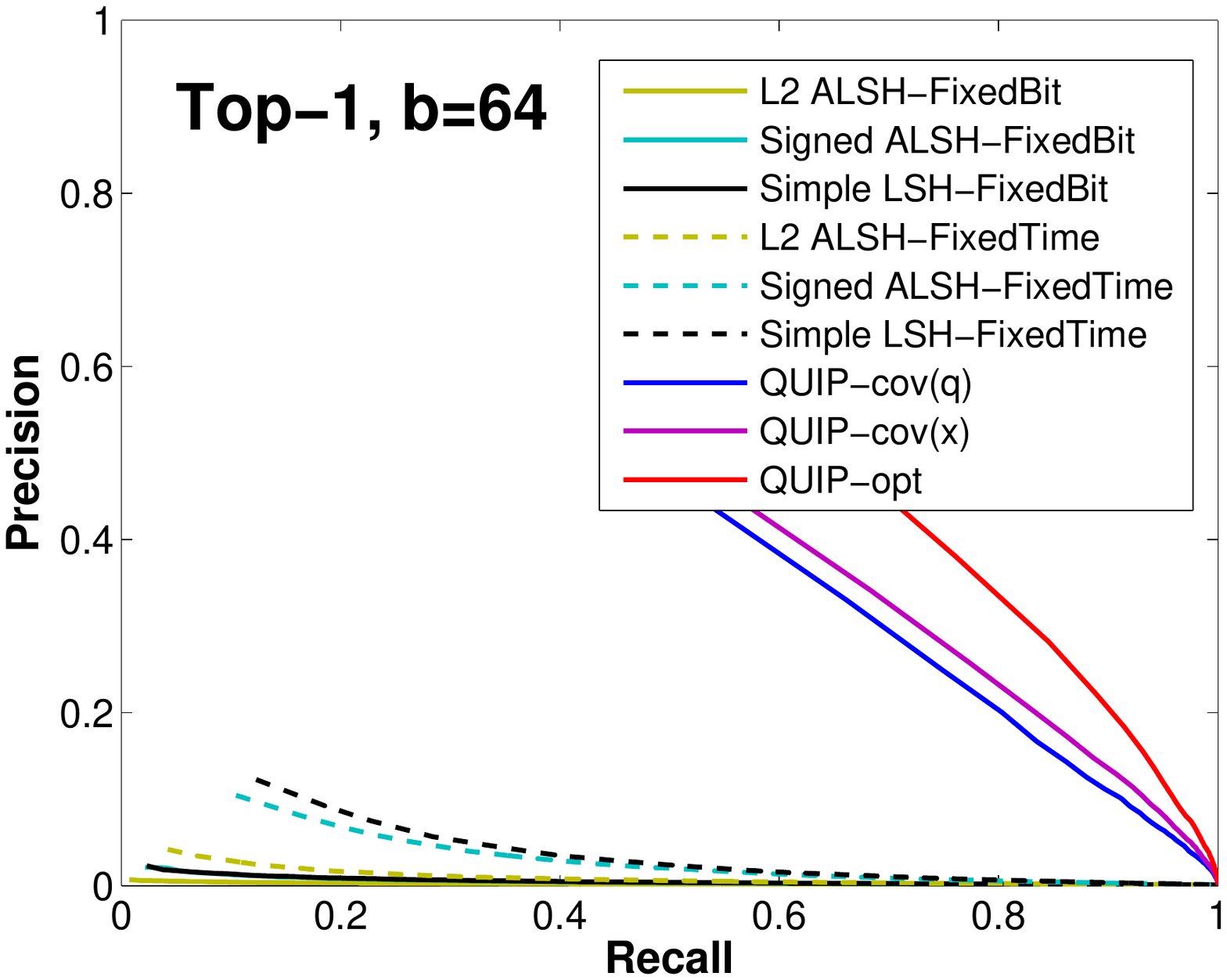} \hskip -3pt
\includegraphics[trim=0.9in 2.5in 0.9in 2.5in,clip,width=.248 \textwidth]{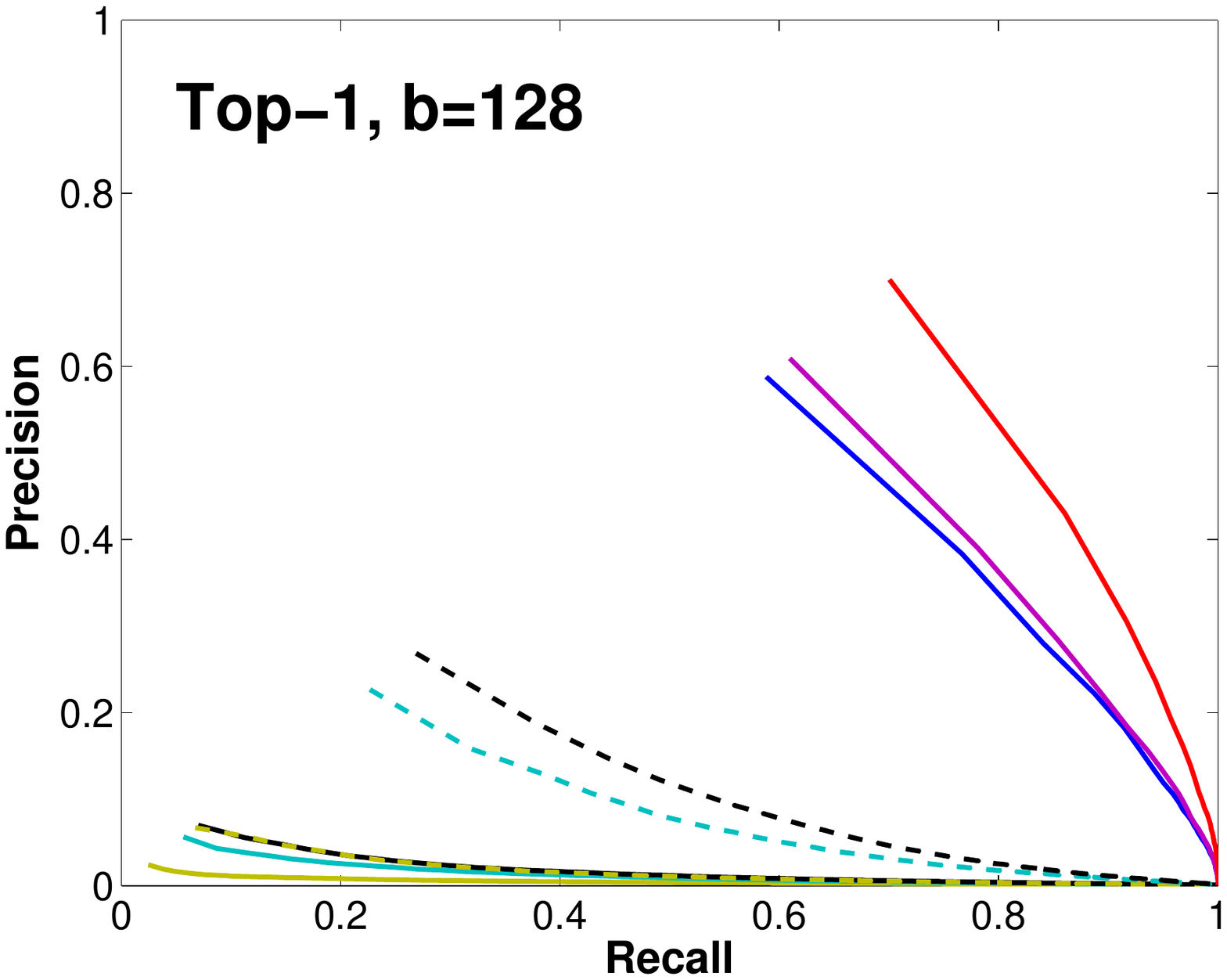} \hskip -3pt
\includegraphics[trim=0.9in 2.5in 0.9in 2.5in,clip,width=.248 \textwidth]{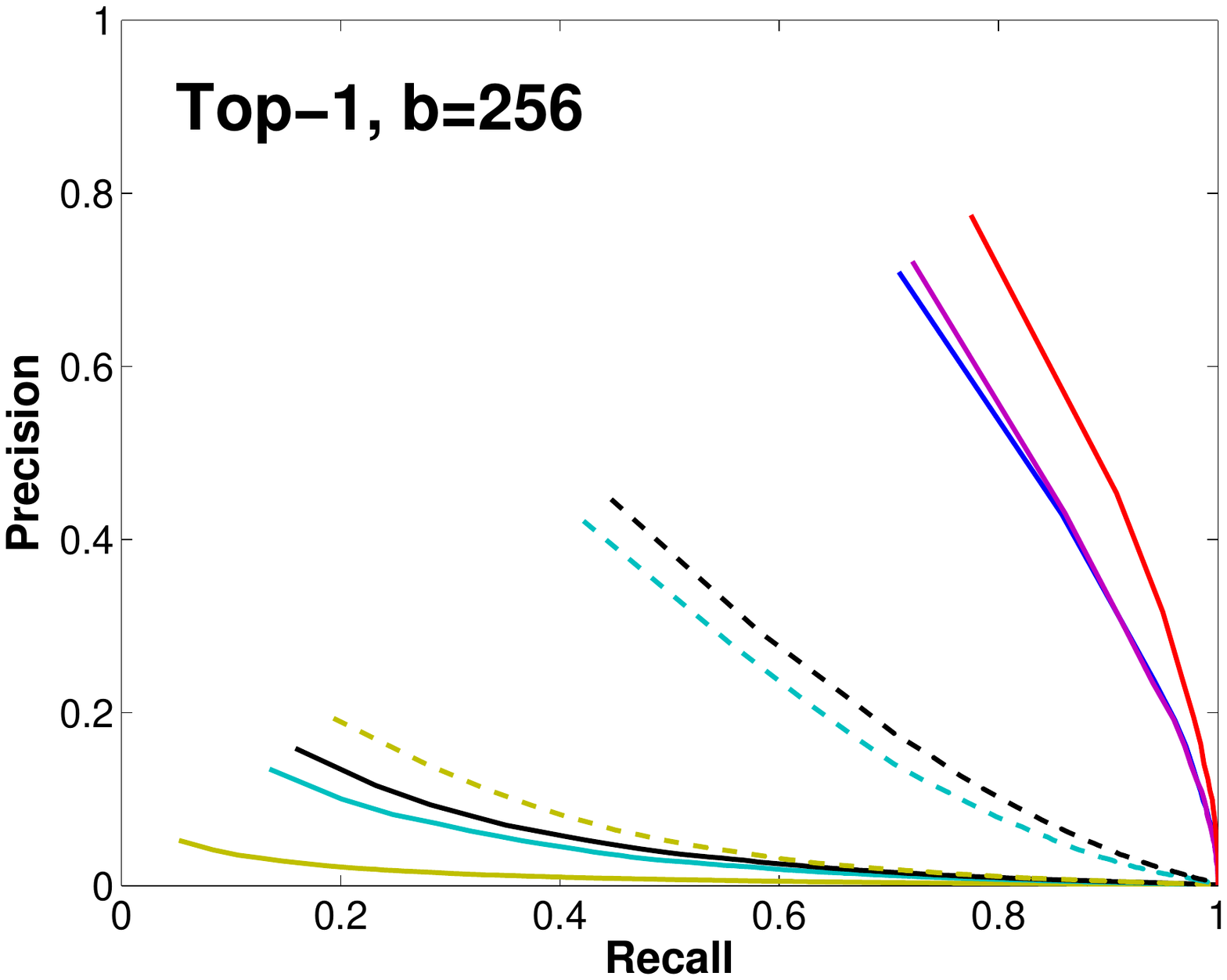} \hskip -3pt
\includegraphics[trim=0.9in 2.5in 0.9in 2.5in,clip,width=.248 \textwidth]{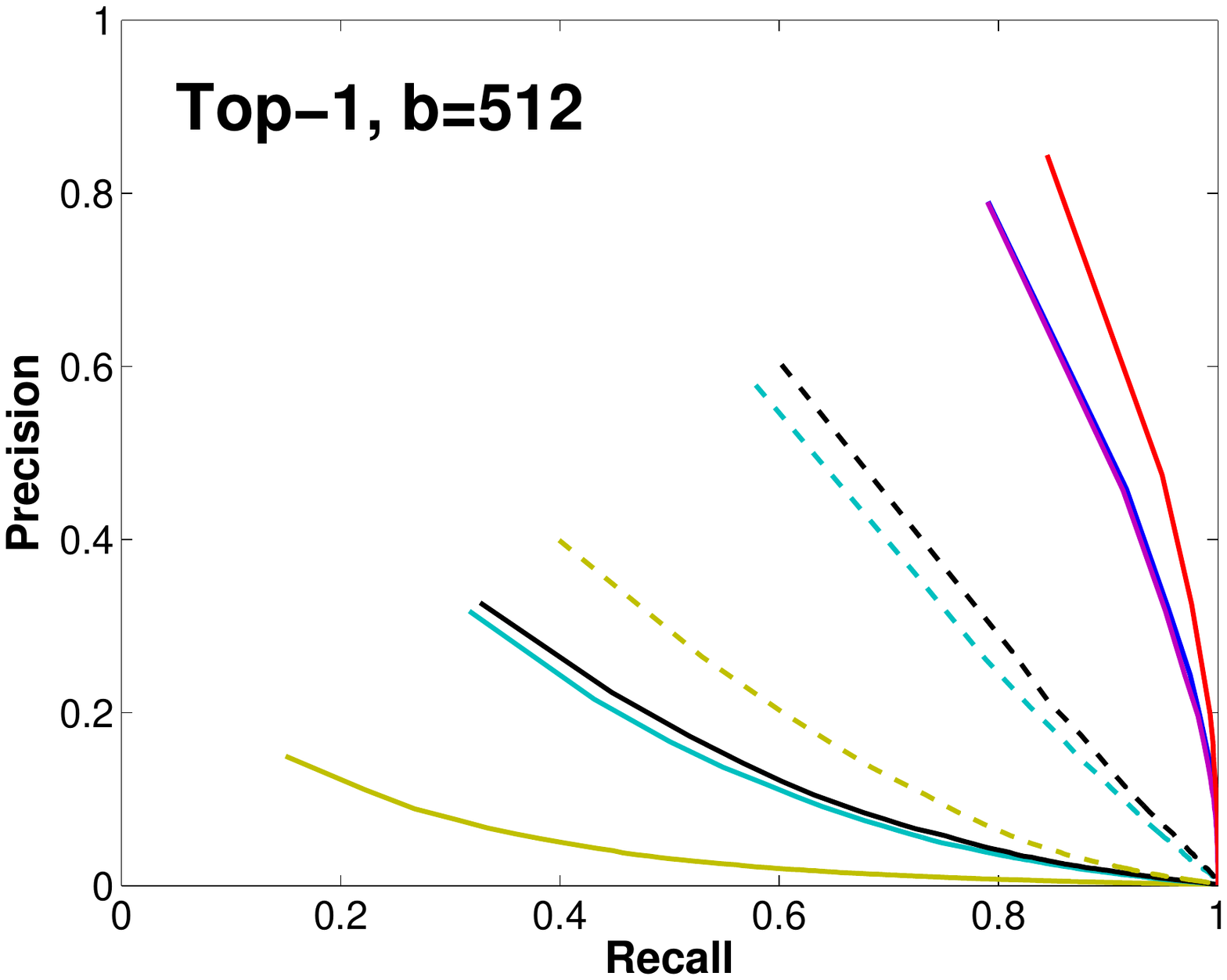}
\vskip -9pt
\includegraphics[trim=0.6in 2.5in 0.9in 2.5in,clip,width=.26 \textwidth]{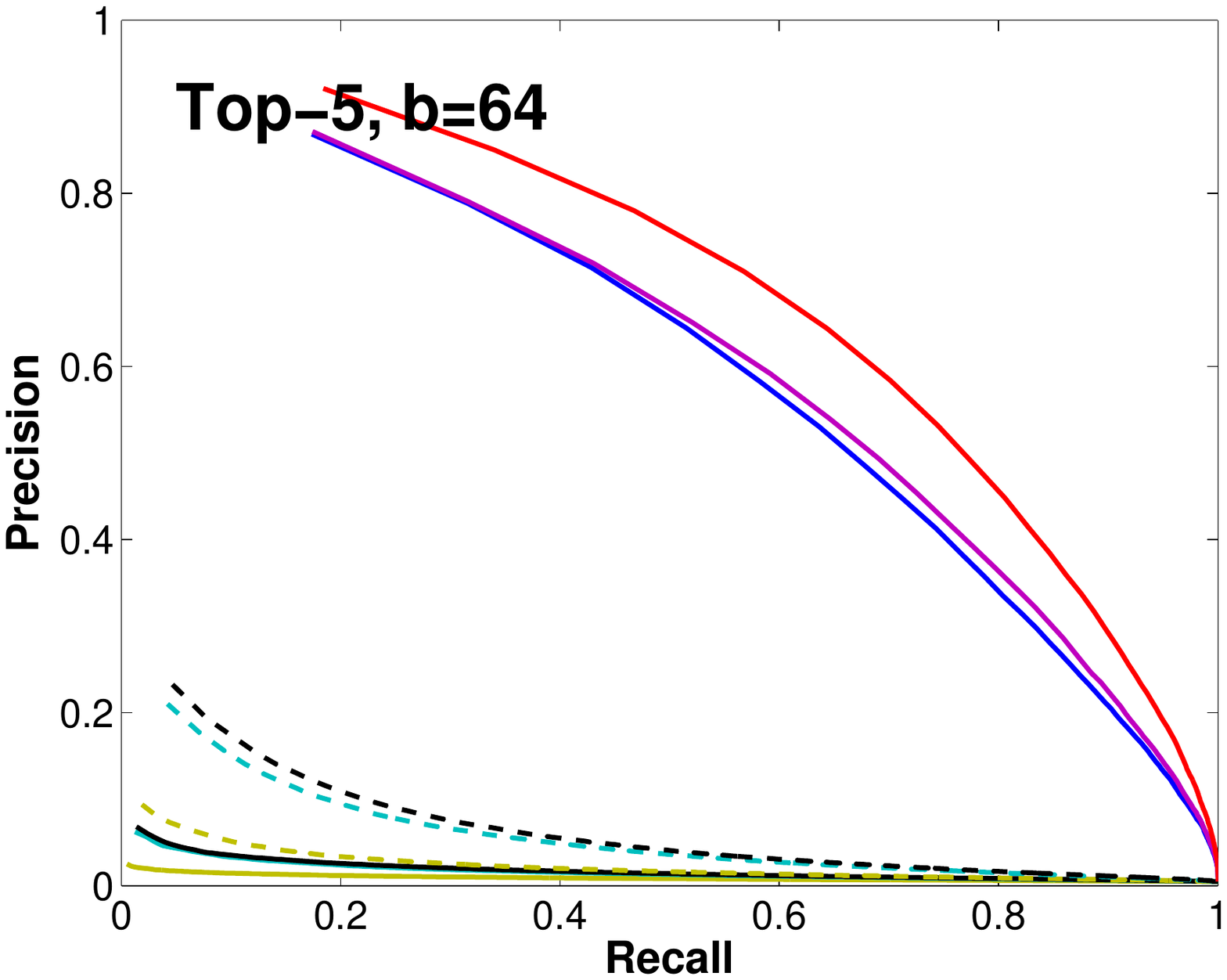} \hskip -3pt
\includegraphics[trim=0.9in 2.5in 0.9in 2.5in,clip,width=.248 \textwidth]{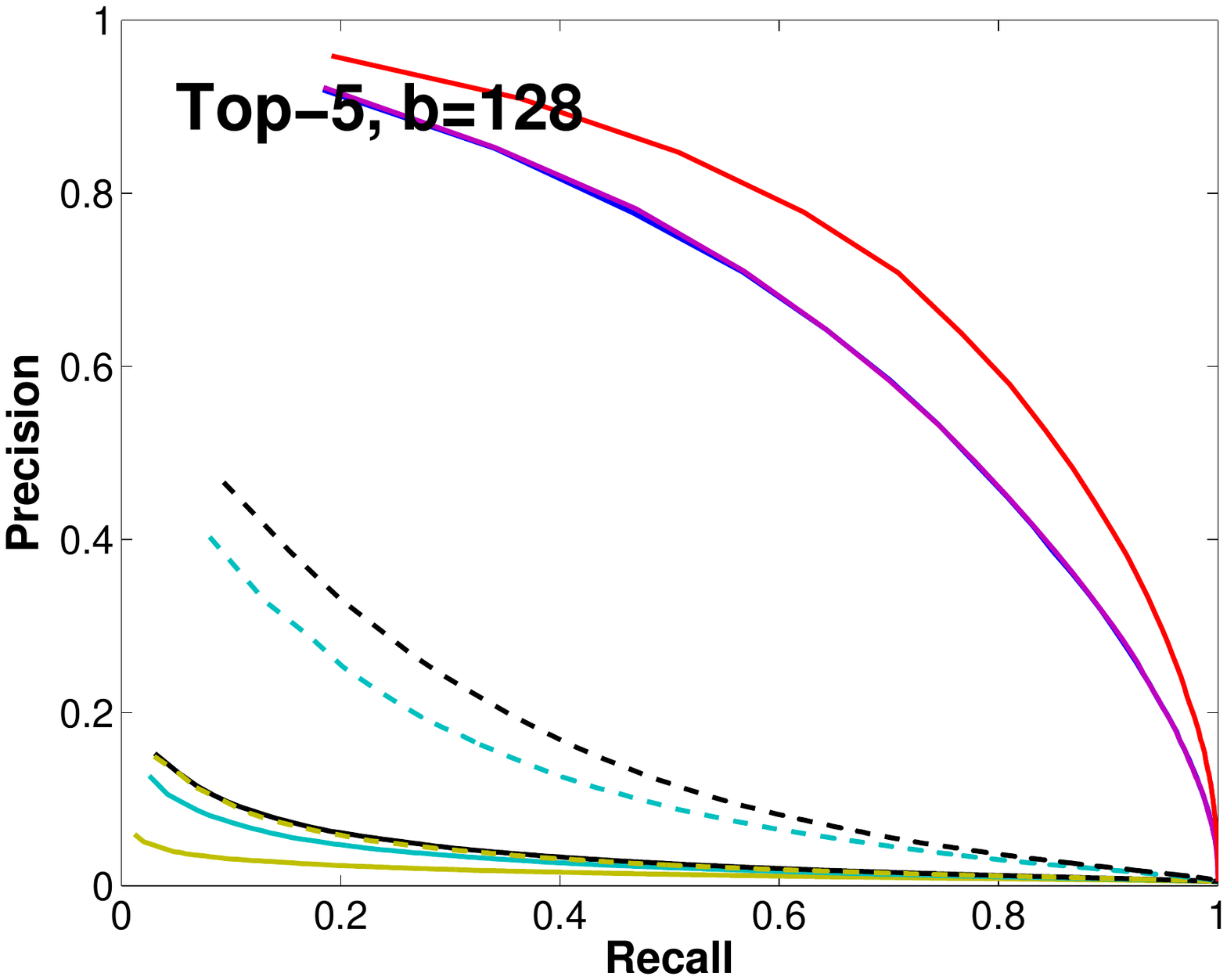} \hskip -3pt
\includegraphics[trim=0.9in 2.5in 0.9in 2.5in,clip,width=.248 \textwidth]{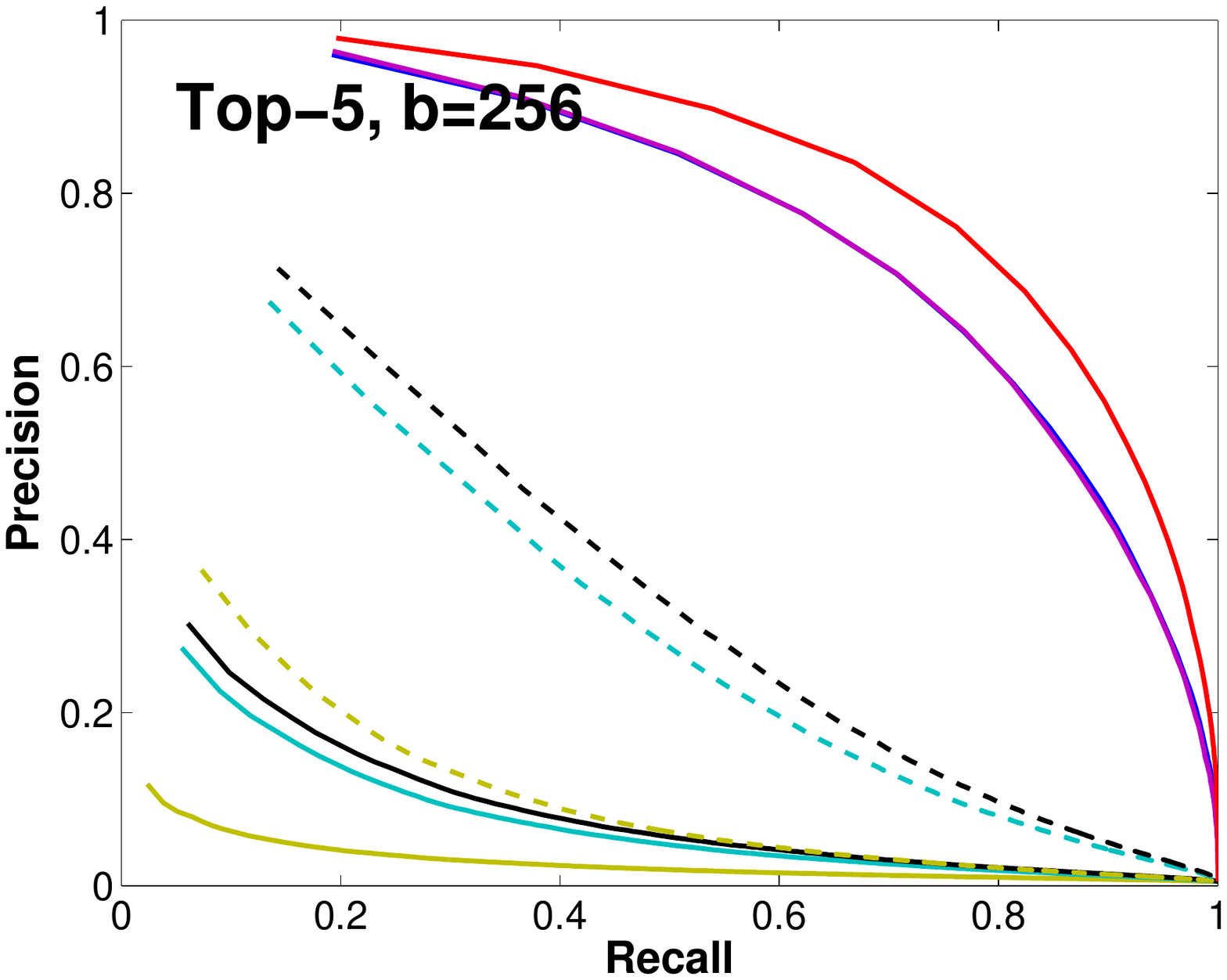} \hskip -3pt
\includegraphics[trim=0.9in 2.5in 0.9in 2.5in,clip,width=.248 \textwidth]{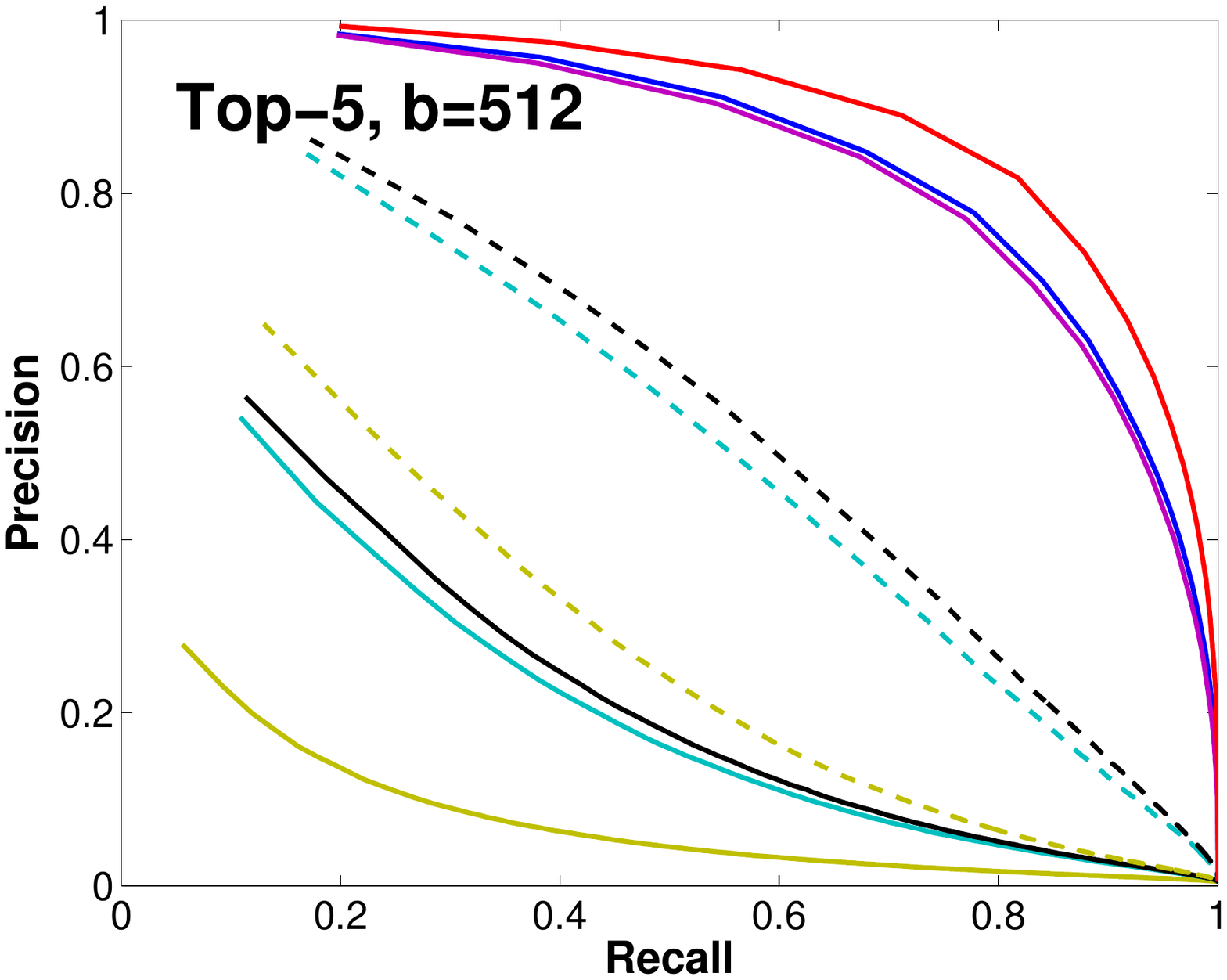}
\vskip -9pt
\includegraphics[trim=0.6in 2.5in 0.9in 2.5in,clip,width=.26 \textwidth]{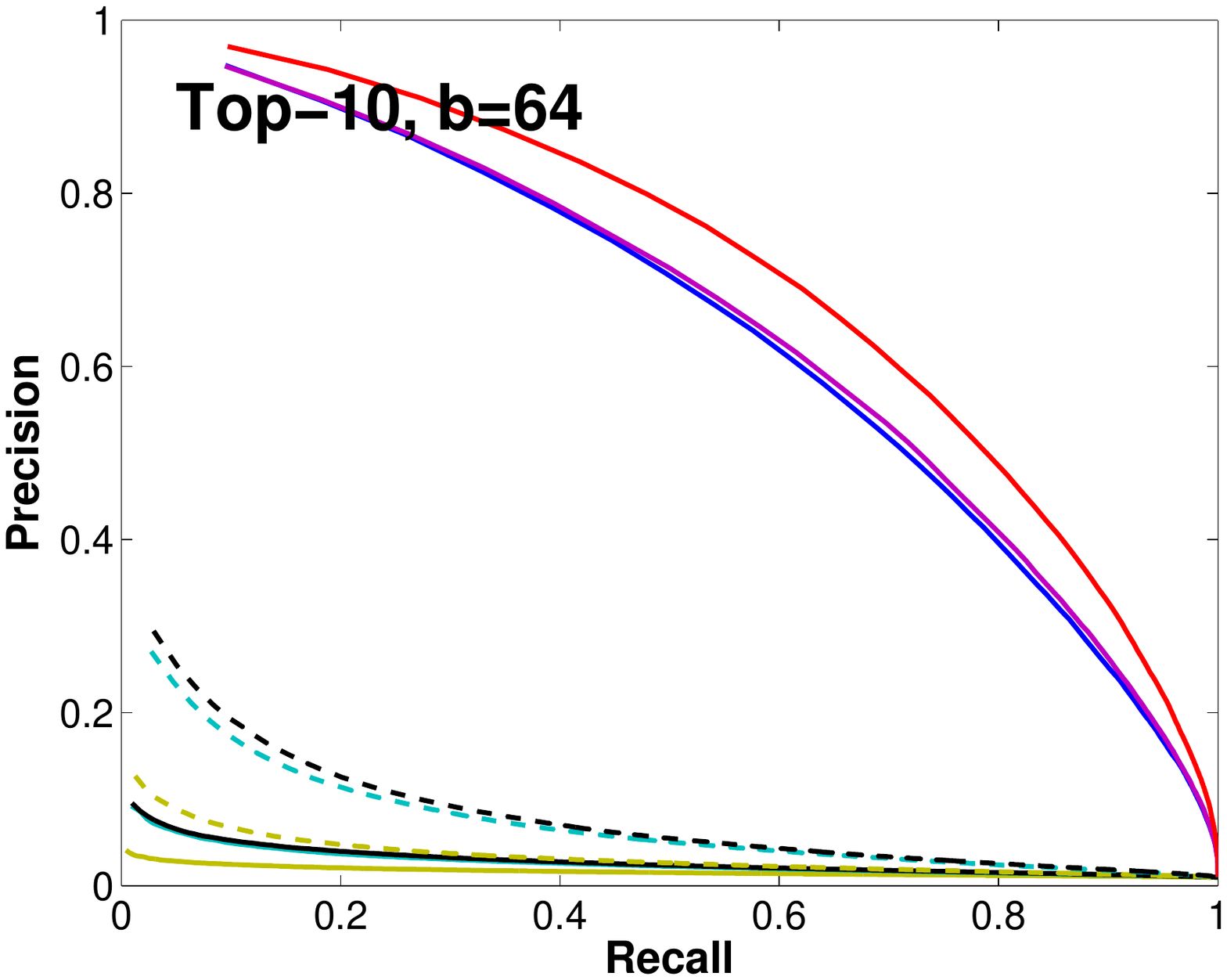} \hskip -3pt
\includegraphics[trim=0.9in 2.5in 0.9in 2.5in,clip,width=.248 \textwidth]{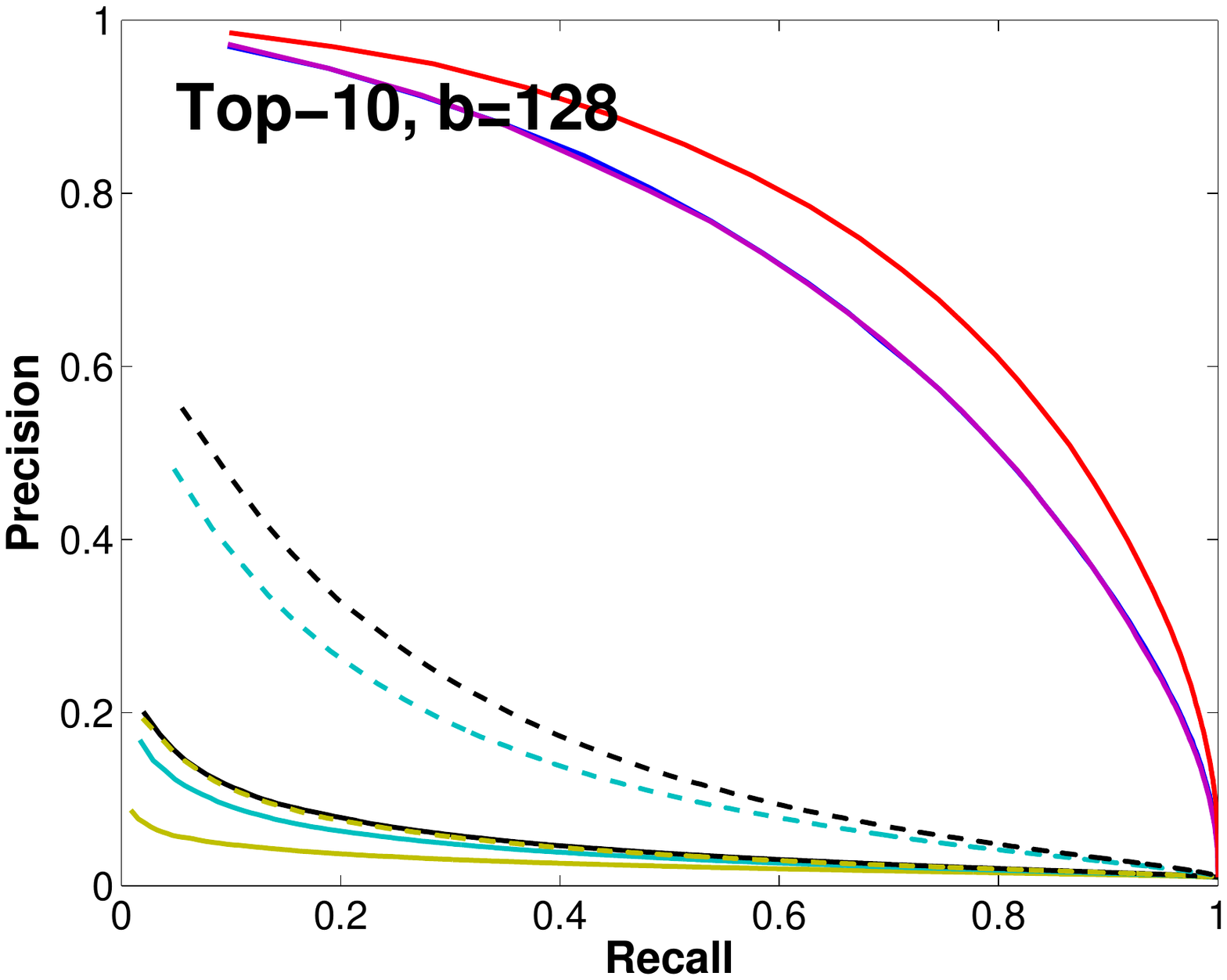} \hskip -3pt
\includegraphics[trim=0.9in 2.5in 0.9in 2.5in,clip,width=.248 \textwidth]{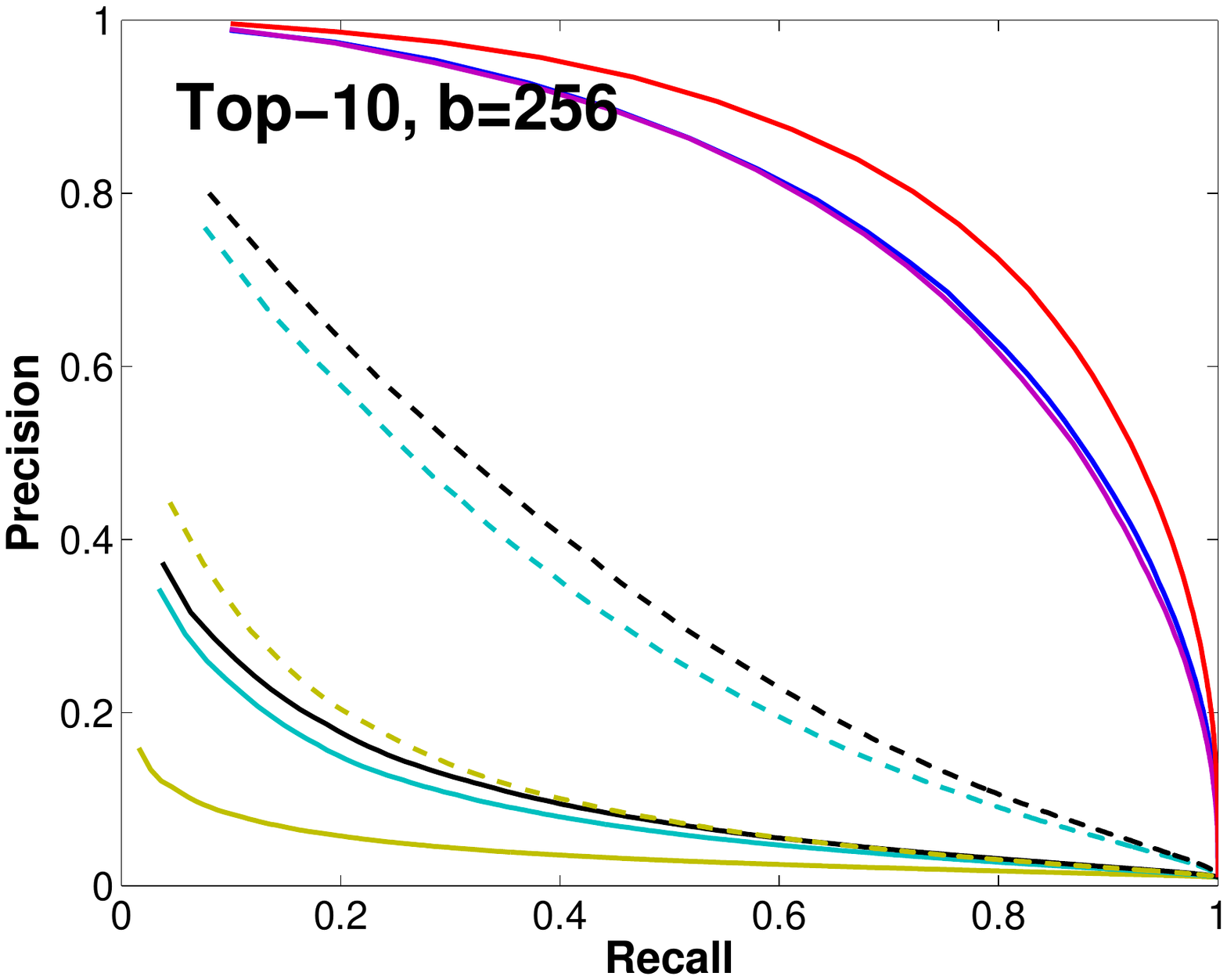} \hskip -3pt
\includegraphics[trim=0.9in 2.5in 0.9in 2.5in,clip,width=.248 \textwidth]{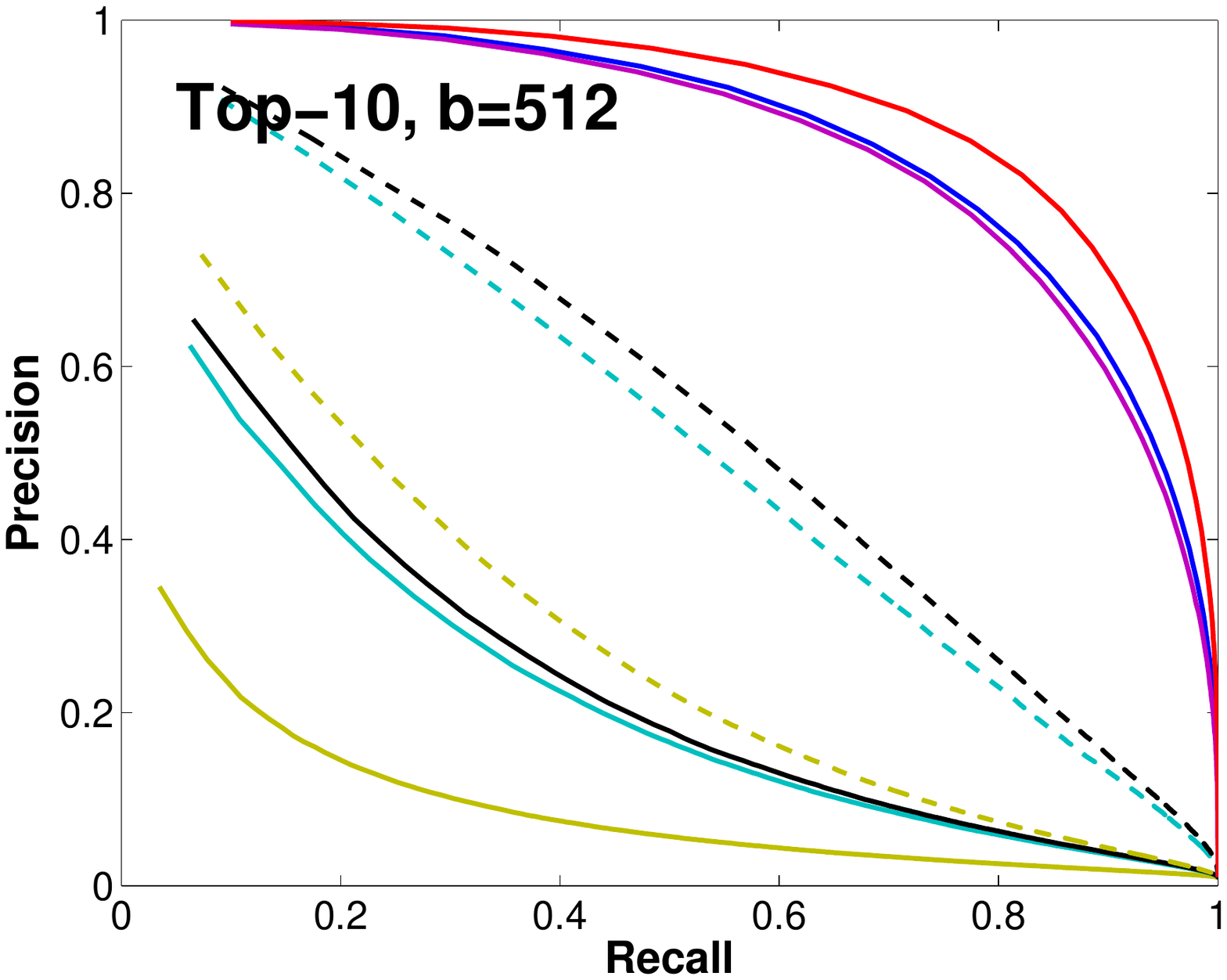}
\vskip -3pt
\subcaption{ImageNet dataset, retrieval of Top-1, 5 and 10 items.}
\end{subfigure}
\begin{subfigure}[c]{1 \textwidth}
\centering
\vskip -3pt
\includegraphics[trim=0.6in 2.5in 0.9in 2.5in,clip,width=.26 \textwidth]{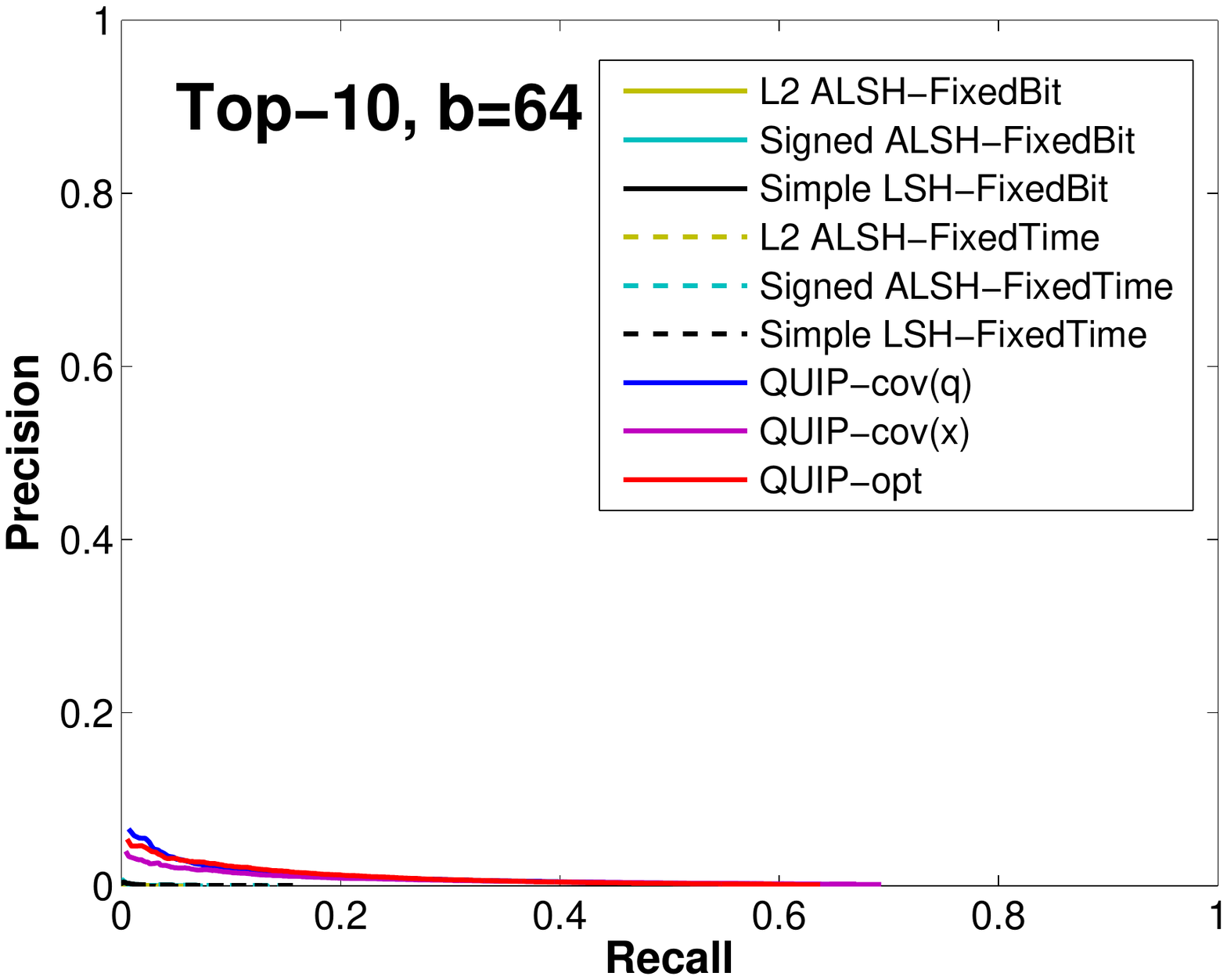} \hskip -3pt
\includegraphics[trim=0.9in 2.5in 0.9in 2.5in,clip,width=.248 \textwidth]{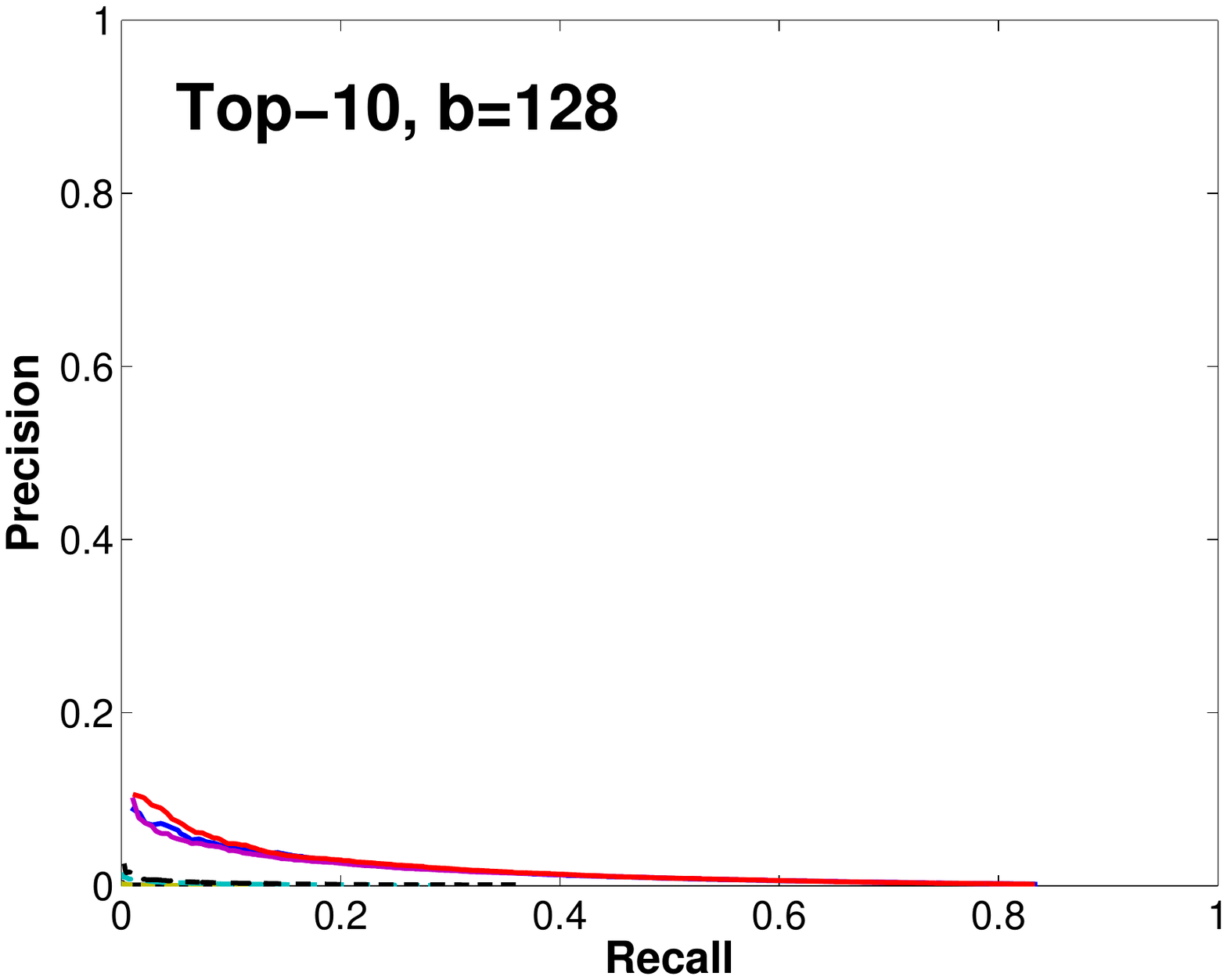} \hskip -3pt
\includegraphics[trim=0.9in 2.5in 0.9in 2.5in,clip,width=.248 \textwidth]{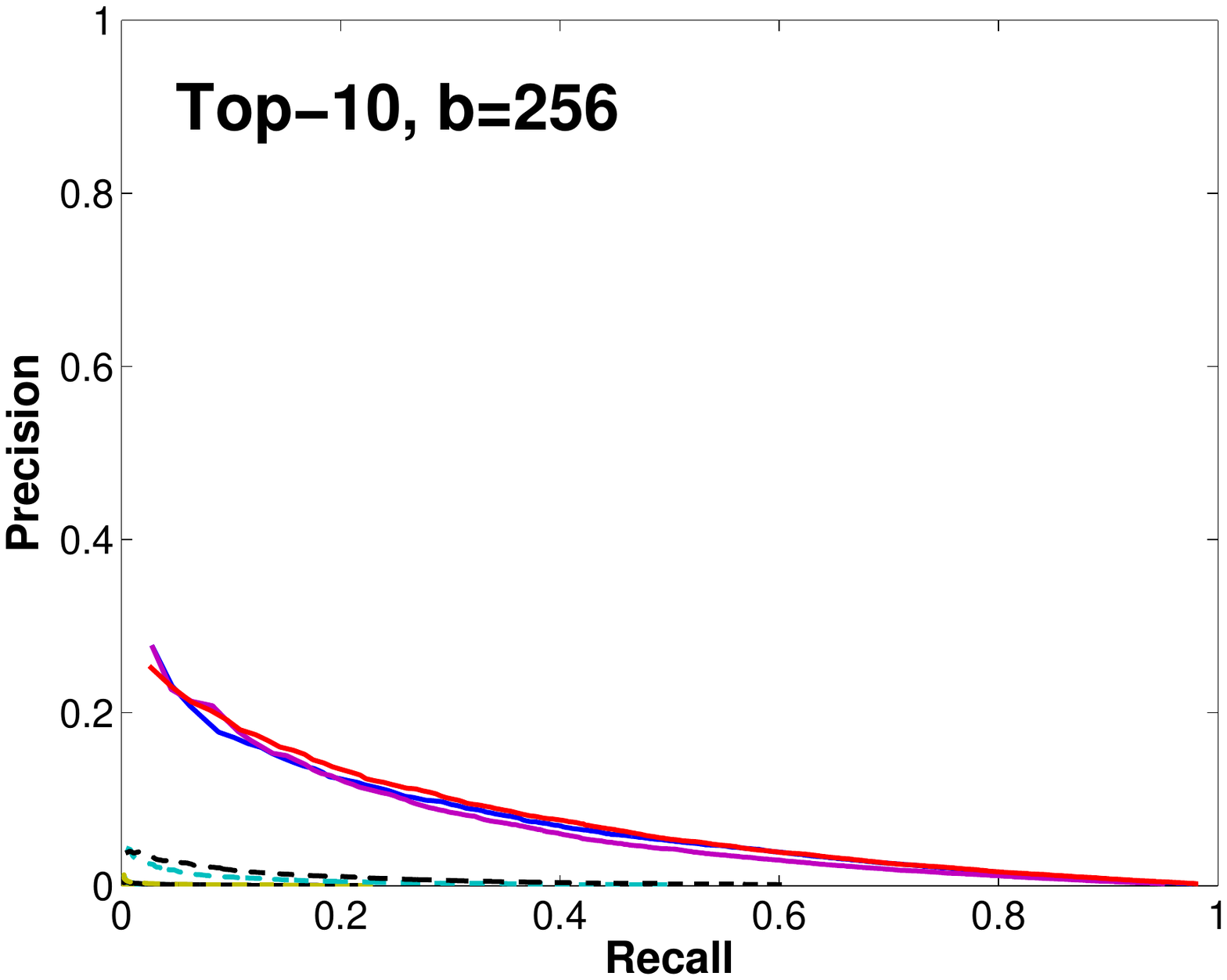} \hskip -3pt
\includegraphics[trim=0.9in 2.5in 0.9in 2.5in,clip,width=.248 \textwidth]{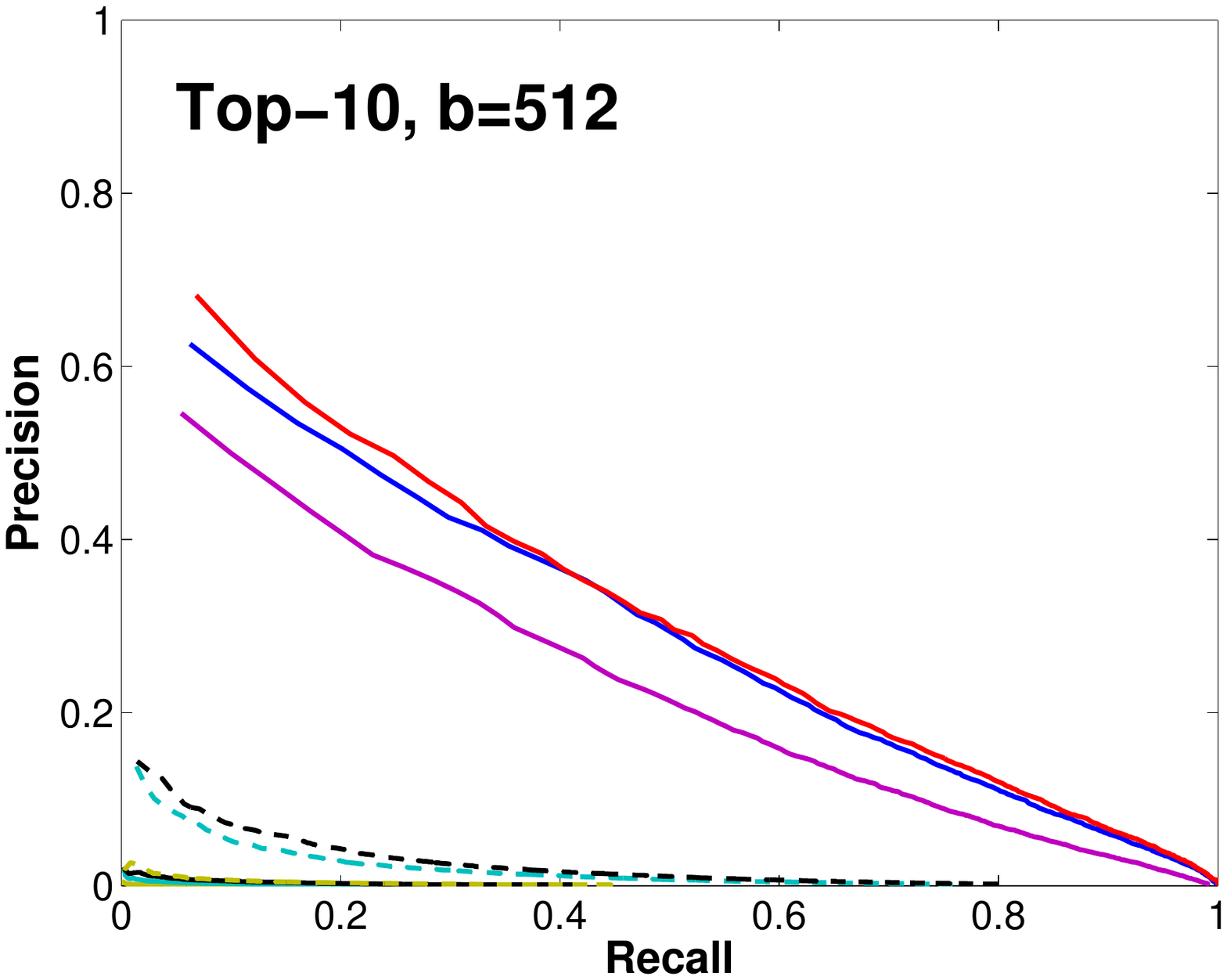}
\vskip -9pt
\includegraphics[trim=0.6in 2.5in 0.9in 2.5in,clip,width=.26 \textwidth]{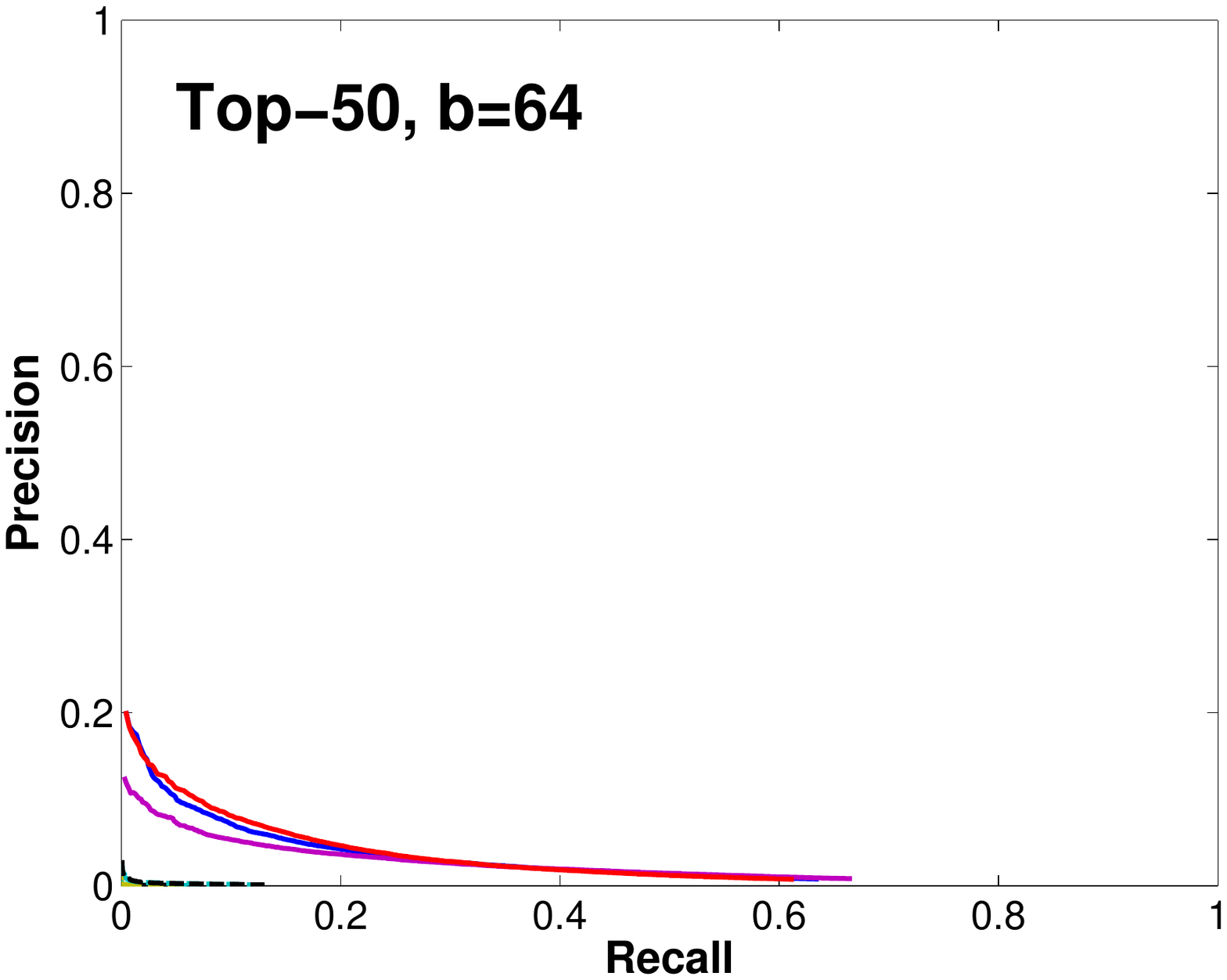} \hskip -3pt
\includegraphics[trim=0.9in 2.5in 0.9in 2.5in,clip,width=.248 \textwidth]{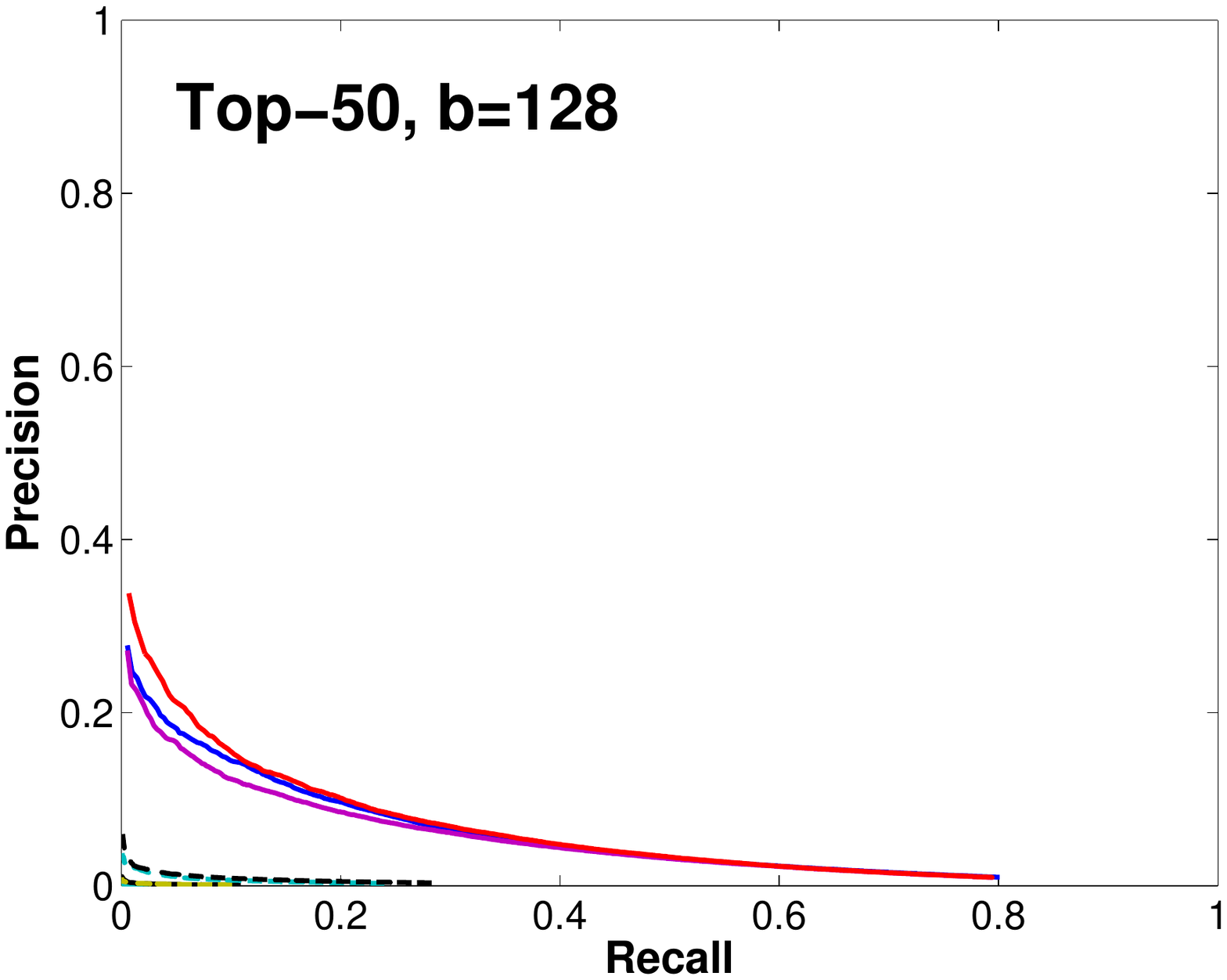} \hskip -3pt
\includegraphics[trim=0.9in 2.5in 0.9in 2.5in,clip,width=.248 \textwidth]{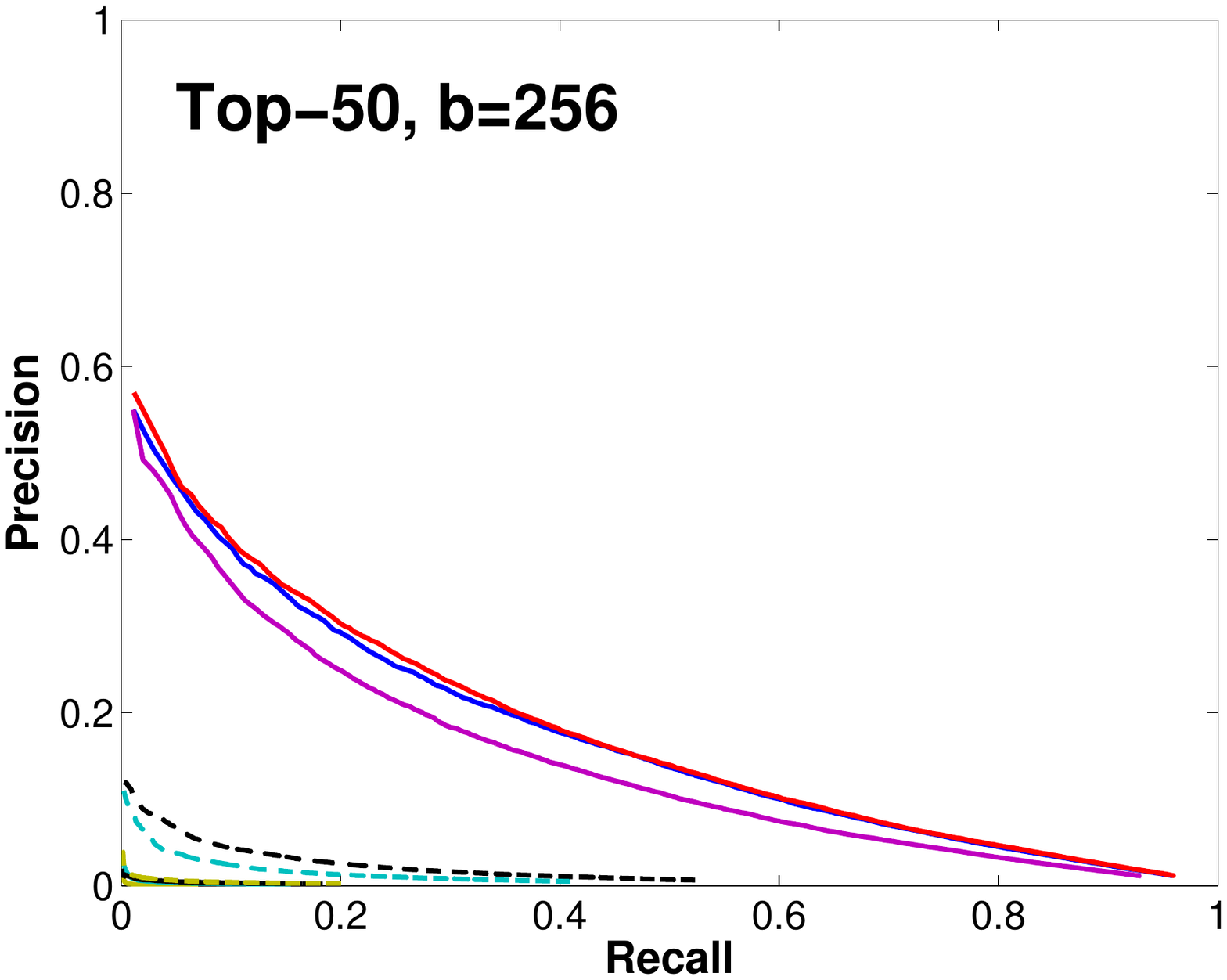} \hskip -3pt
\includegraphics[trim=0.9in 2.5in 0.9in 2.5in,clip,width=.248 \textwidth]{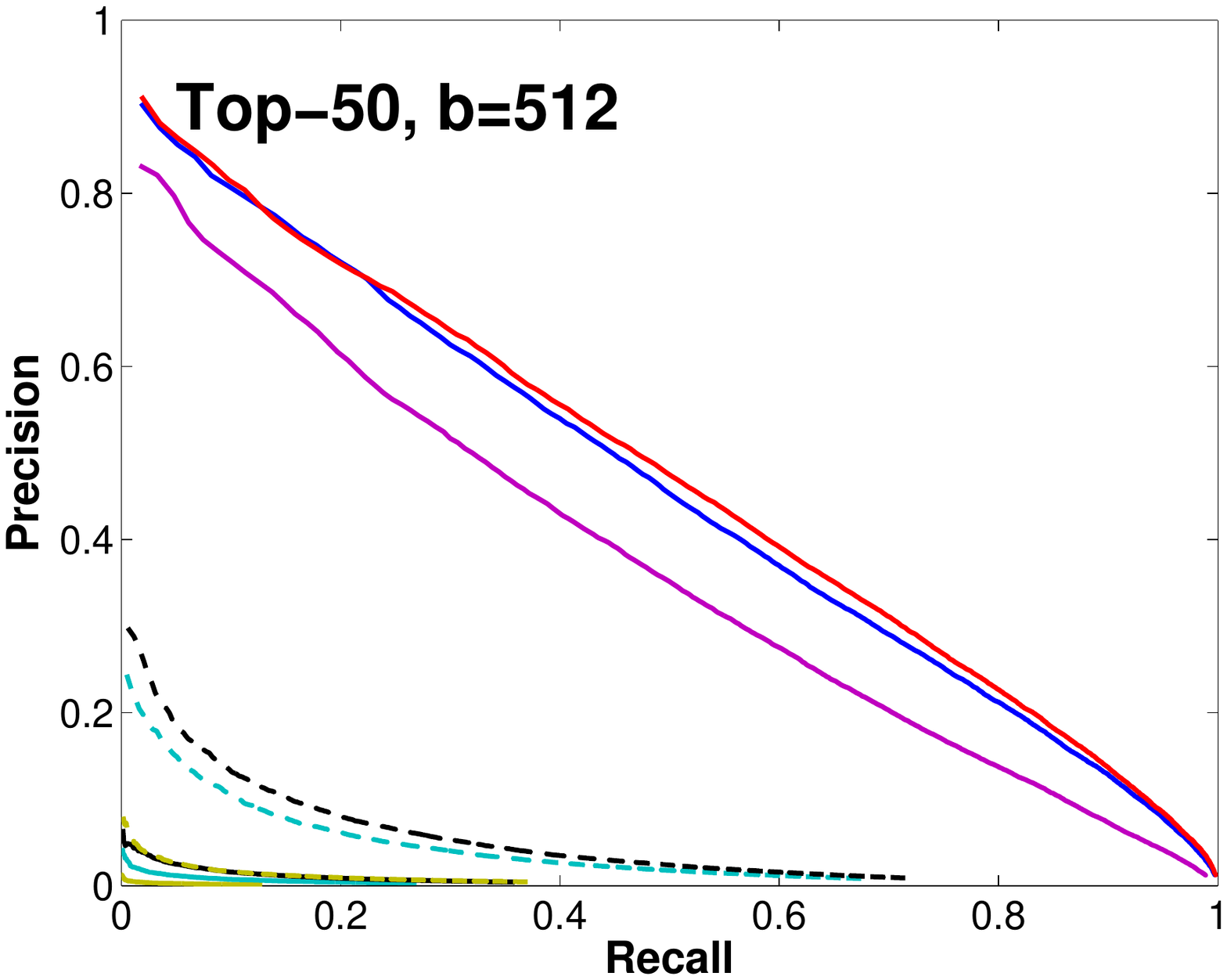}
\vskip -9pt
\includegraphics[trim=0.6in 2.5in 0.9in 2.5in,clip,width=.26 \textwidth]{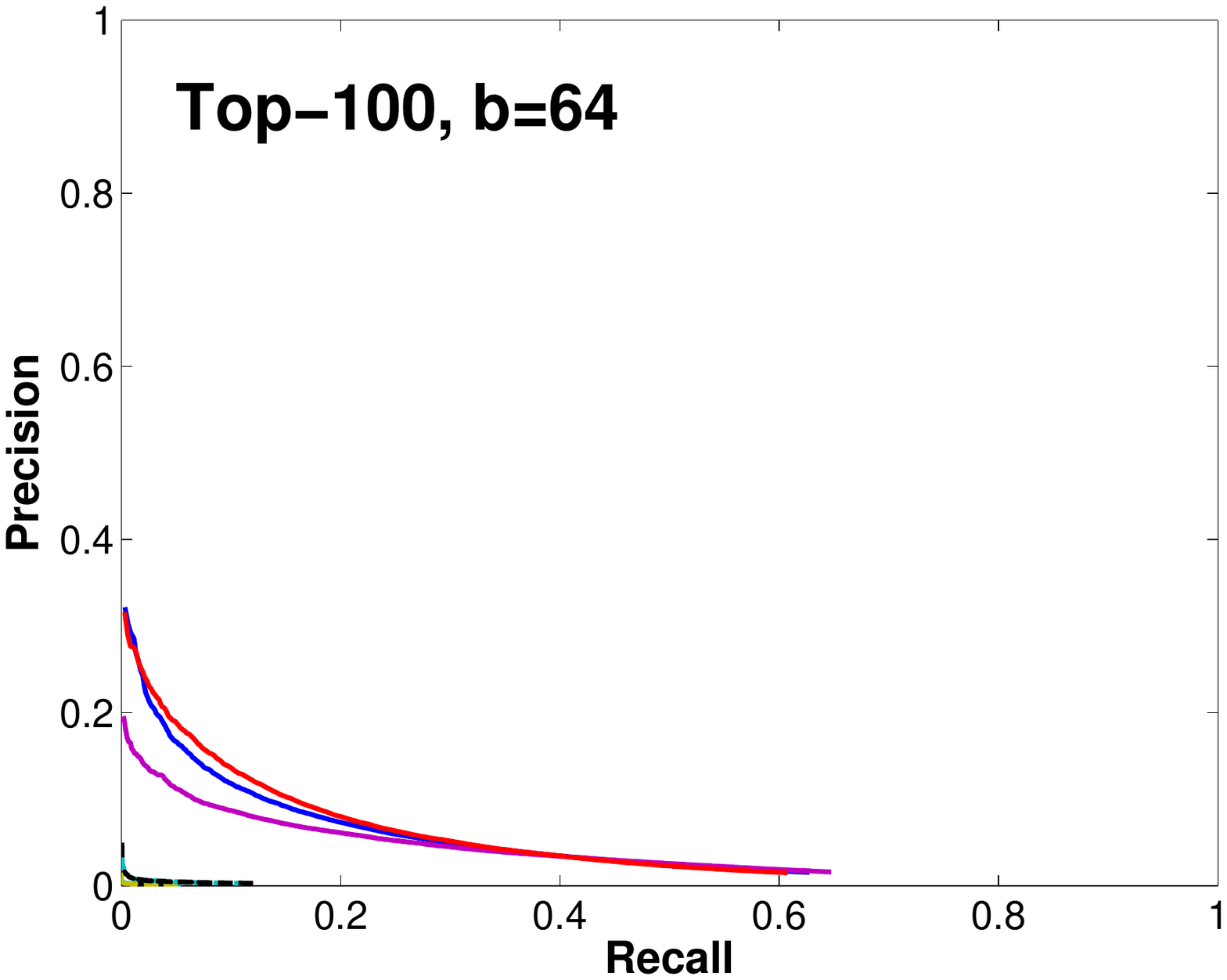} \hskip -3pt
\includegraphics[trim=0.9in 2.5in 0.9in 2.5in,clip,width=.248 \textwidth]{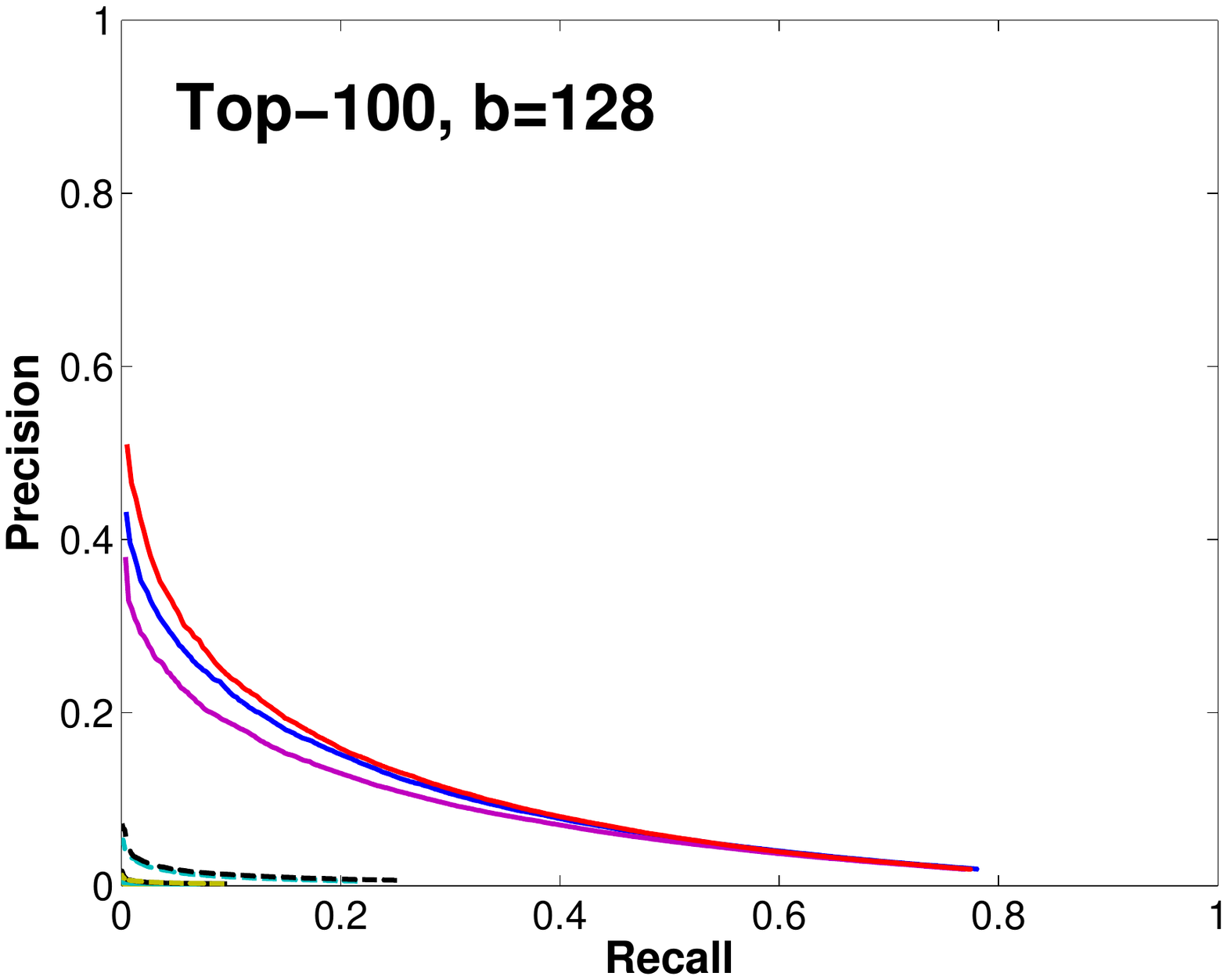} \hskip -3pt
\includegraphics[trim=0.9in 2.5in 0.9in 2.5in,clip,width=.248 \textwidth]{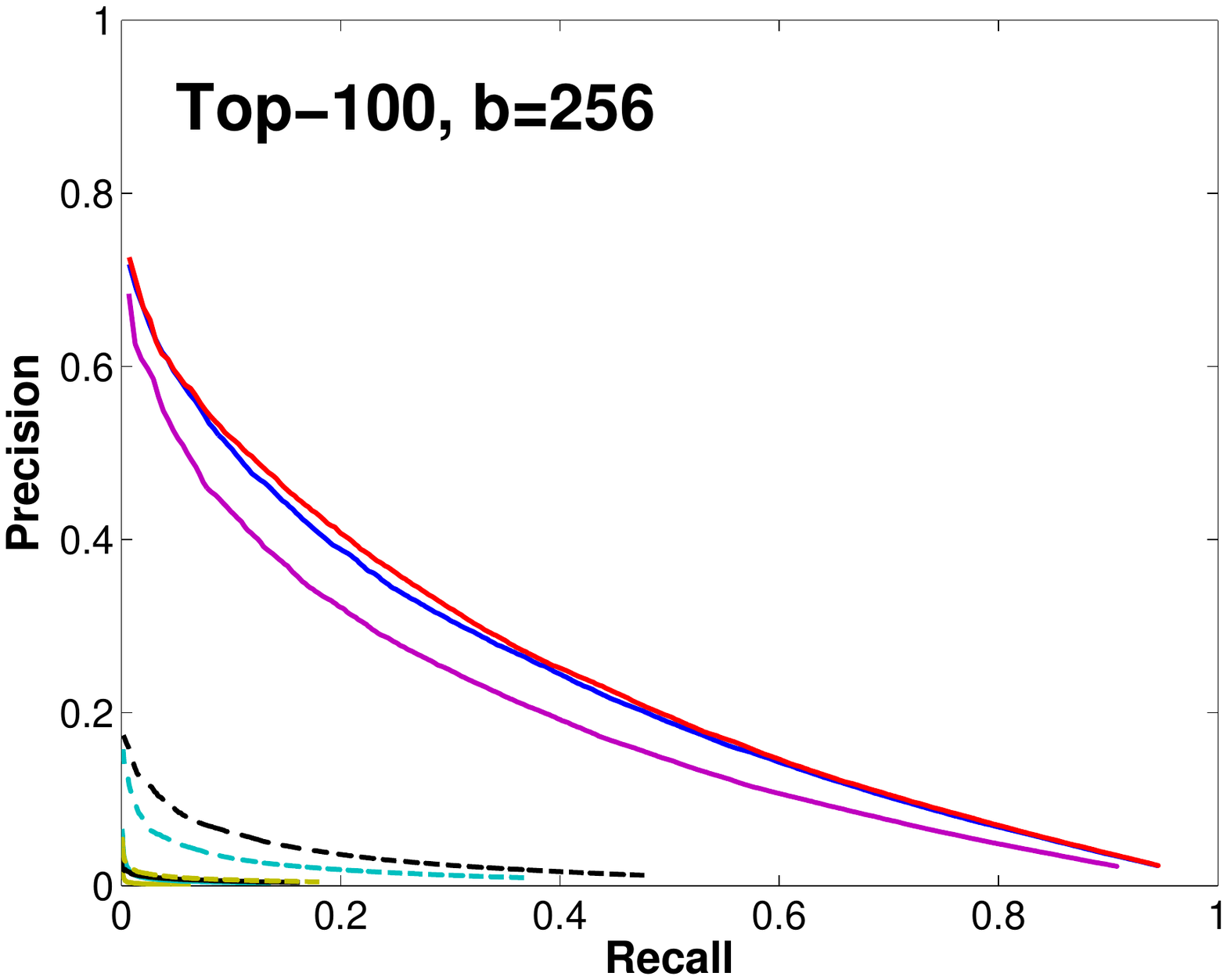} \hskip -3pt
\includegraphics[trim=0.9in 2.5in 0.9in 2.5in,clip,width=.248 \textwidth]{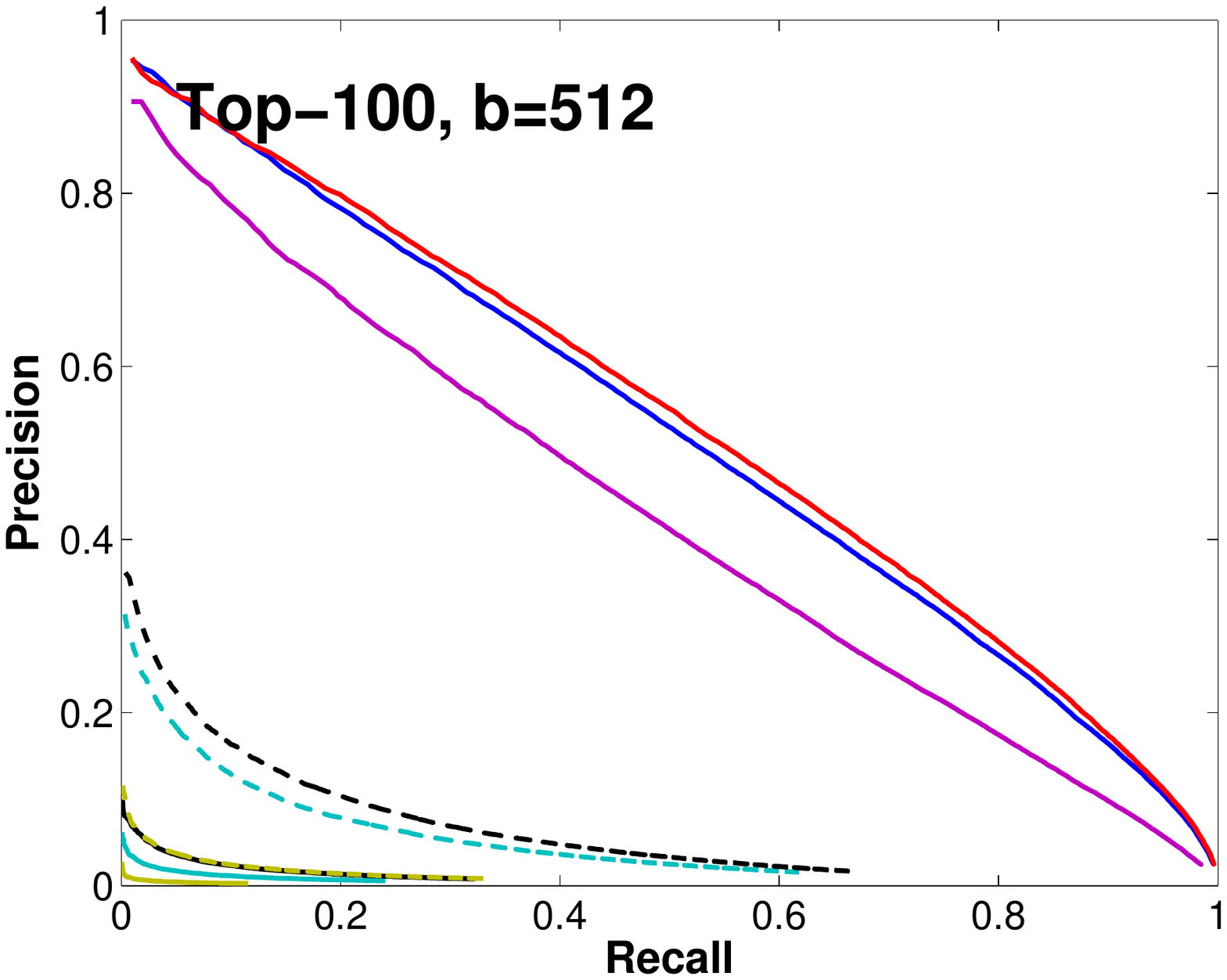}
\vskip -3pt
\subcaption{VideoRec dataset, retrieval of Top-10, 50 and 100 items.}
\vskip -3pt
\end{subfigure}
\begin{subfigure}[c]{1 \textwidth}
\end{subfigure}
\caption{\label{fig:top}
Precision Recall curves using different methods on \emph{ImageNet} and \emph{VideoRec}.}
\vskip -9pt
\end{figure}

\begin{figure}
\begin{subfigure}[c]{.49 \textwidth}
\includegraphics[trim=0.6in 2.5in 0.75in 2.5in,clip,width=1 \textwidth]{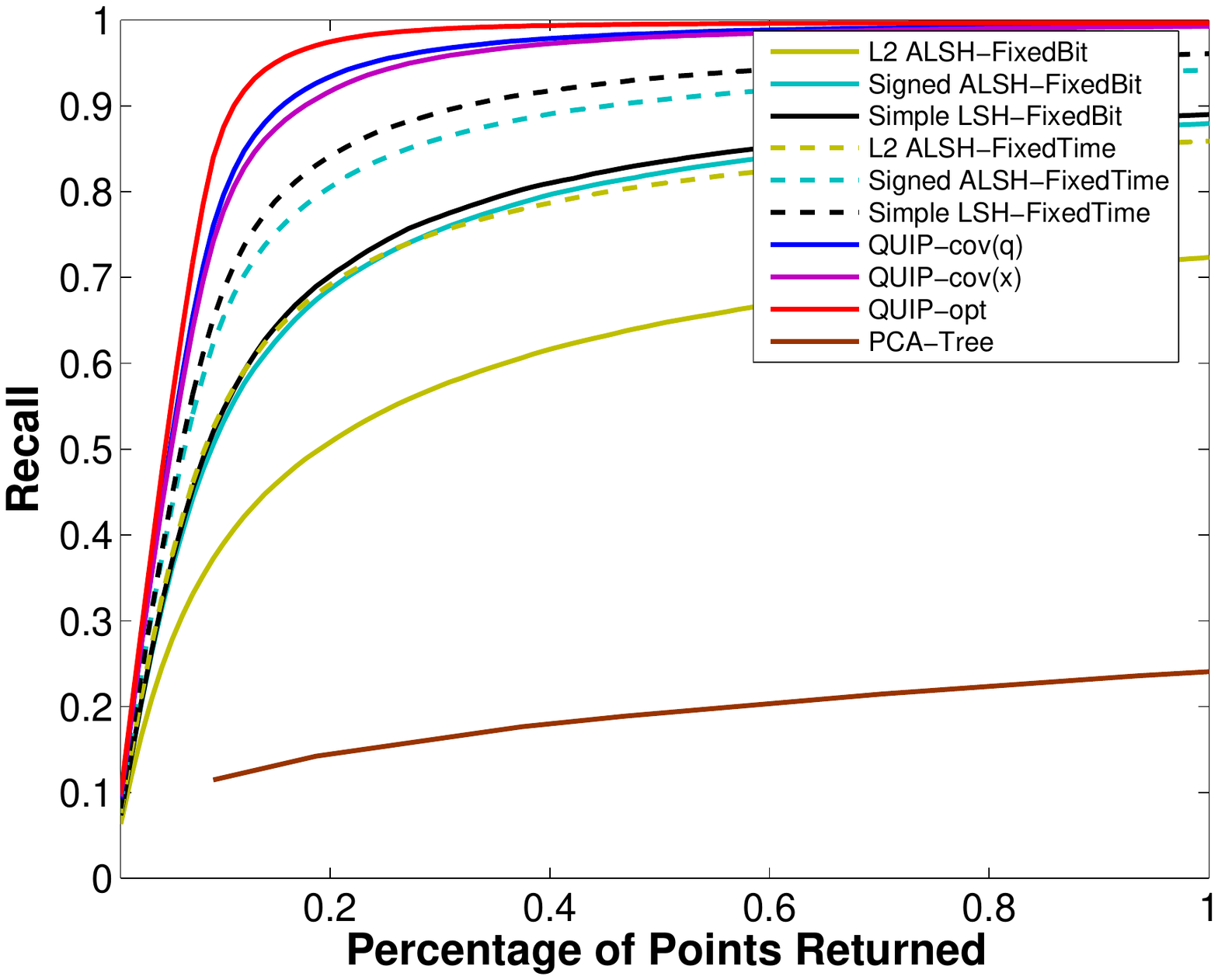}
\subcaption{Movielens, top-10}
\end{subfigure}
\begin{subfigure}[c]{.49 \textwidth}
\includegraphics[trim=0.6in 2.5in 0.75in 2.5in,clip,width=1 \textwidth]{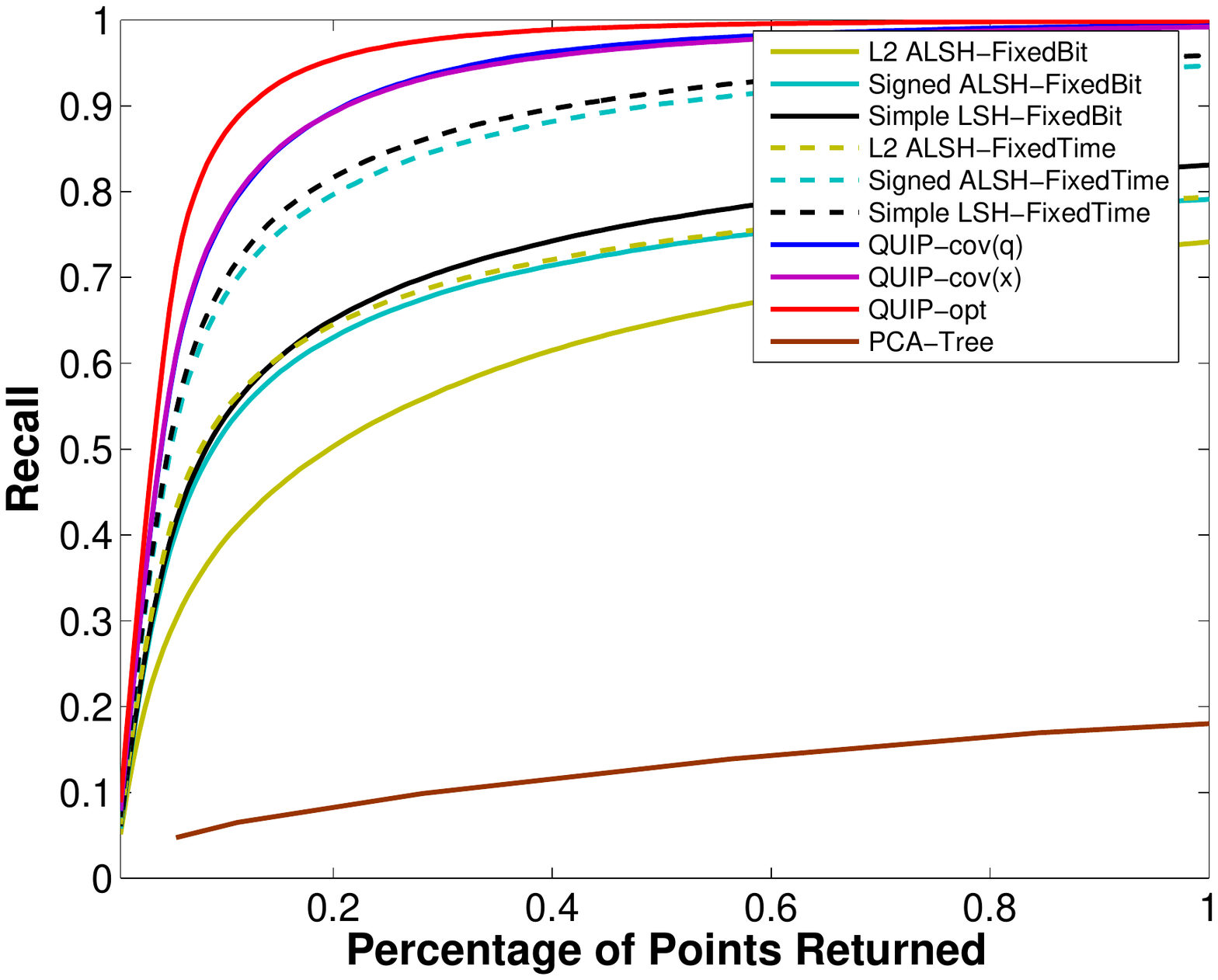}
\subcaption{Netflix, top-10}
\end{subfigure}
\begin{subfigure}[c]{.49 \textwidth}
\includegraphics[trim=0.6in 2.5in 0.75in 2.5in,clip,width=1 \textwidth]{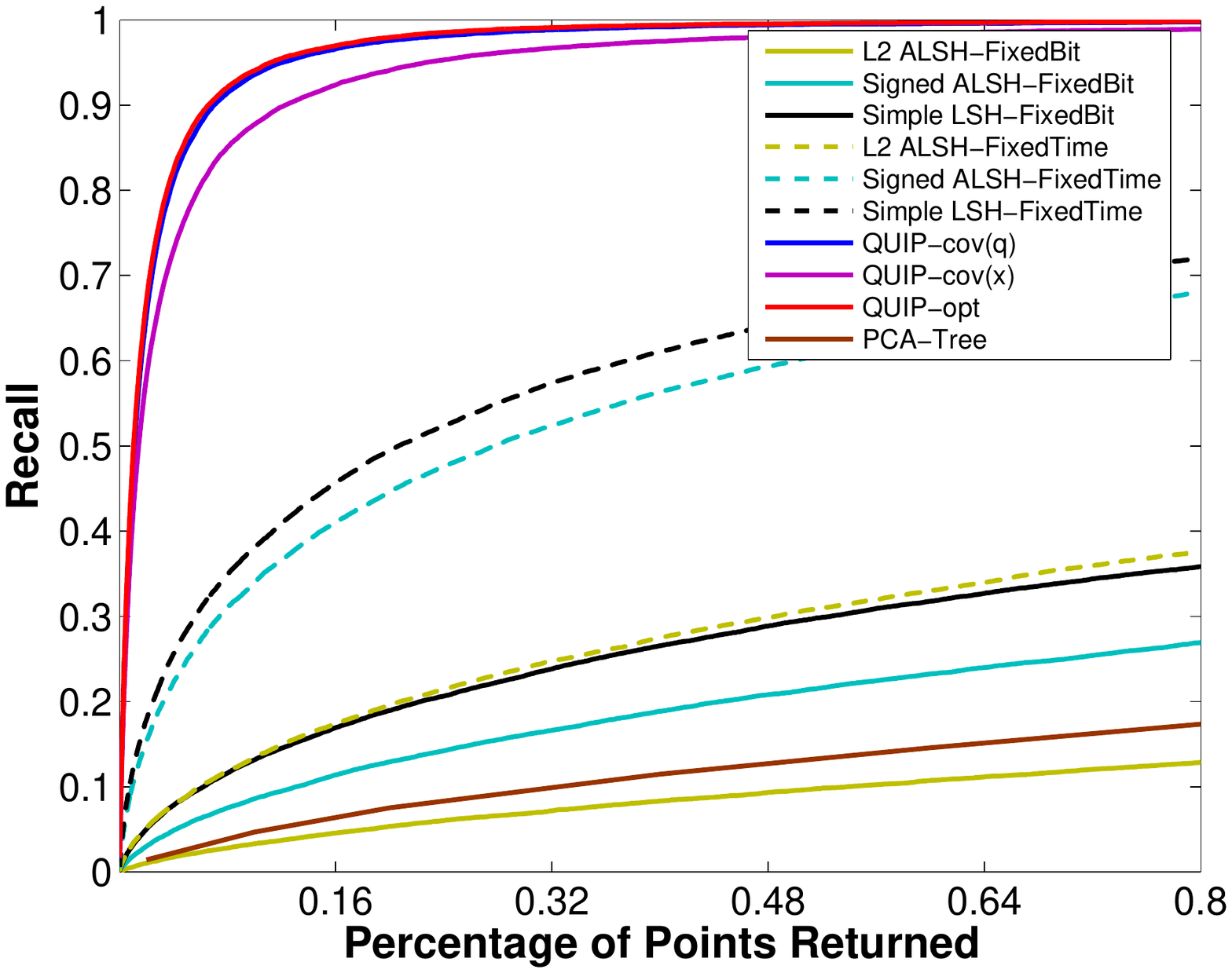}
\subcaption{VideoRec, top-50}
\end{subfigure}
\begin{subfigure}[c]{.49 \textwidth}
\includegraphics[trim=0.6in 2.5in 0.75in 2.5in,clip,width=1 \textwidth]{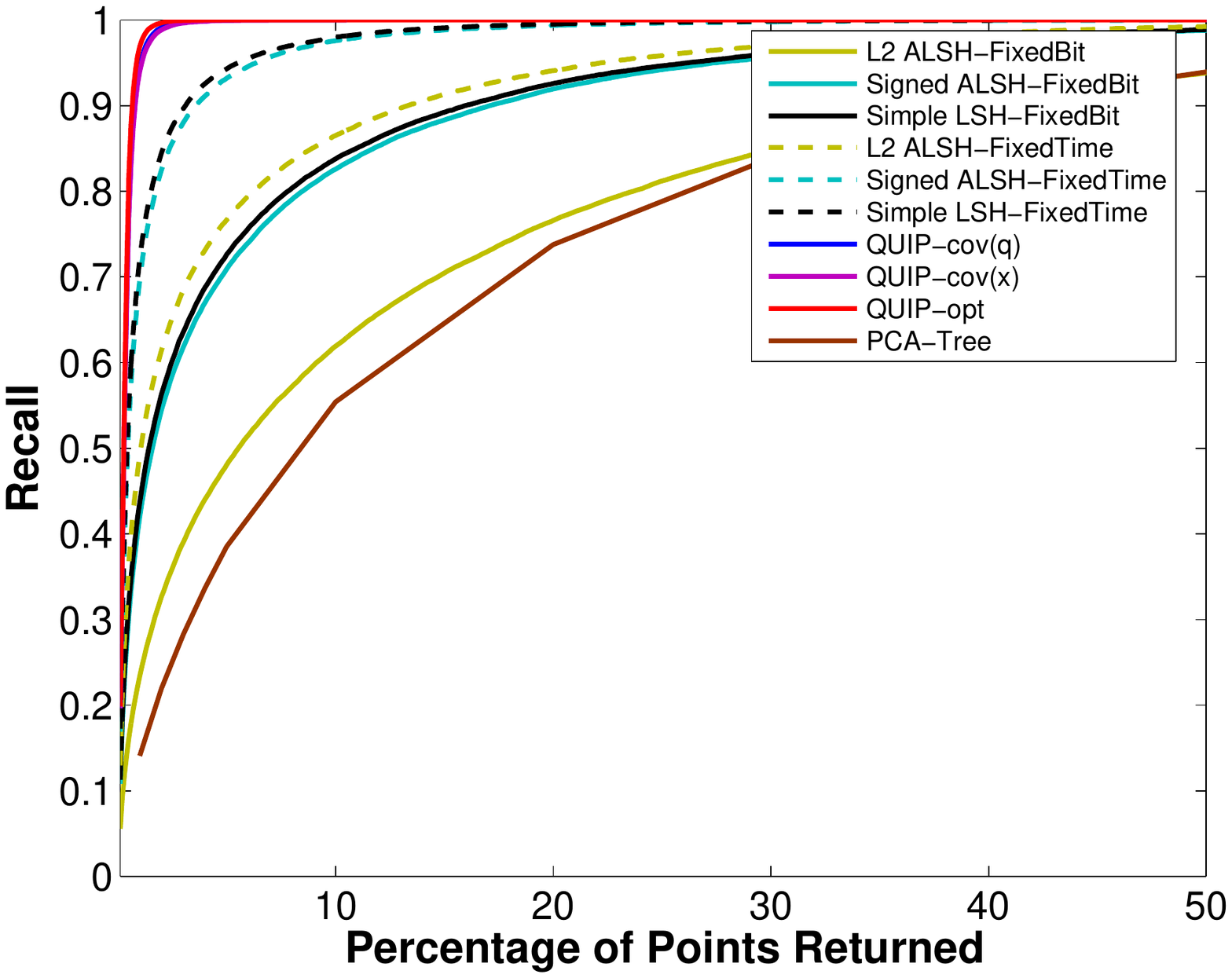}
\subcaption{ImageNet, top-5}
\end{subfigure}
\caption{\label{fig:recall}
Recall curves for different techniques under different numbers of returned neighbors (shown as the percentage of total number of points in the database). We plot the recall curve instead of the precision recall curve because \emph{PCA-Tree} uses original vectors to compute distances therefore the precision will be the same as recall in Top-K search. The number of bits used for all the plots is $512$, except for \emph{Signed ALSH-FixedTime},  \emph{L2 ALSH-FixedTime} and  \emph{Simple LSH-FixedTime}, which use $1536$ bits. \emph{PCA-Tree} does not perform well on these datasets, mostly due to the fact that the dimensionality of our datasets is relatively high ($150$ to $1025$ dimensions), and trees are known to be more susceptible to dimensionality. Note the the original paper from Bachrach et al. [2] used datasets with dimensionality $50$.}
\end{figure}

\subsection{Theoretical analysis - proofs}

In this section we present proofs of all the theorems presented in the main body of the paper. We also show some additional theoretical results on our quantization based method.

\subsubsection{Vectors' balancedness - proof of Theorem \ref{perm_theorem_main}}

In this section we prove Theorem \ref{perm_theorem_main} 
and show that one can also obtain balancedness property with the use of the random rotation.


\begin{proof}
Let us denote $v=(v_{1},...,v_{d})$ and $perm(v)=[B_{1},...,B_{K}]$,
where $B_{i}$ is the $i$th block ($i=1,...,K$).
Let us fix some block $B_{j}$.
For a given $i$ denote by $X^{j}_{i}$ a random variable such that $X^{j}_{i}=v_{i}^{2}$ if $v_{i}$ is the block $B_{j}$ after applying random permutation and $X^{j}_{i} = 0$ otherwise.
Notice that a random variable $N^{j} = \sum_{i=1}^{d} X^{j}_{i}$
captures this part of the squared norm of the vector $v$ that resides in block $j$. We have: 
\begin{equation}
E[N^{j}] = \sum_{i=1}^{d} E[X^{j}_{i}] = \sum_{i=1}^{d} \frac{1}{K}v_{i}^{2} = \frac{1}{K}\|v\|^{2}_{2}.
\end{equation}

Since the analysis presented above can be conducted for every block $B_{j}$, we complete the proof. 

\end{proof}

Another possibility is to use random rotation, that can be performed for instance by applying 
random normalized Hadamard matrix $\mathcal{H}_{n}$. The Hadamard matrix is a matrix with entries taken from the set $\{-1,1\}$, 
where the rows form an orthogonal system. Random normalized Hadamard matrix can be obtained from the above one by first multiplying 
by the random diagonal matrix $\mathcal{D}$, (where the entries on the diagonal are taken uniformly and independently from the 
set $\{-1,1\}$) and then by rescaling by the factor $\frac{1}{\sqrt{d}}$, where $d$ is the dimensionality of the data. Since dot 
product is invariant in regards to permutations or rotations, we end up with the equivalent problem.

If we take the random rotation approach then we have the following:

\begin{theorem}
\label{hadamard_theorem}
Let $v$ be a vector of dimensionality $d$ and let $0 < \eta <1$.
Then after applying to $v$ linear transformation $\mathcal{H}_{n}$, the transformed vector is 
$\eta$-balanced with probability at least $1-2de^{-\frac{(1-\eta)^{2}K^{2}}{2}}$, where $K$ is the number of blocks.
\end{theorem}

\begin{proof}

\begin{figure}
\centering
\begin{subfigure}[c]{1 \textwidth}
\includegraphics[trim=0.8in 2.5in 0.9in 2.5in,clip,width=.49 \textwidth]{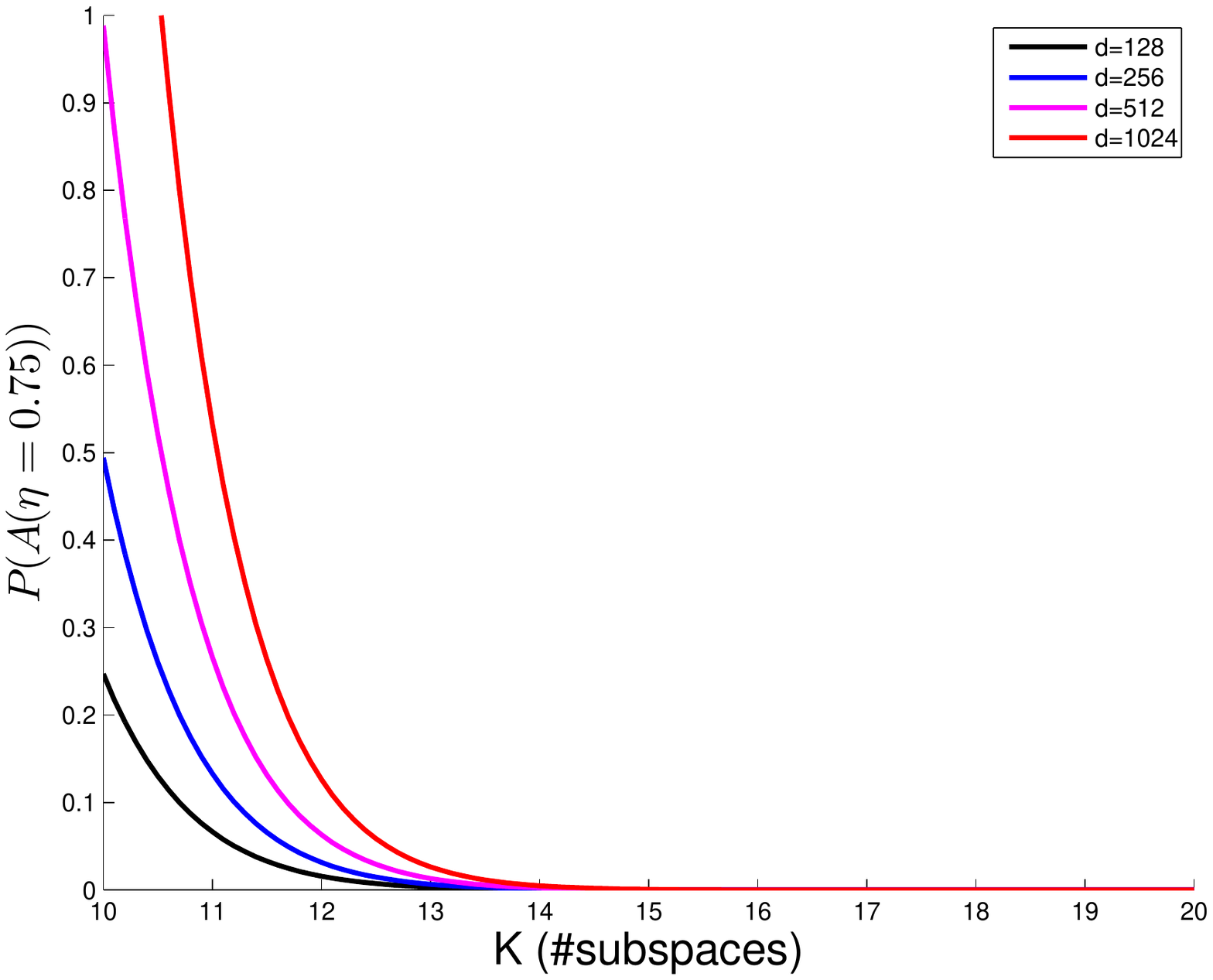}
\includegraphics[trim=0.8in 2.5in 0.9in 2.5in,clip,width=.49 \textwidth]{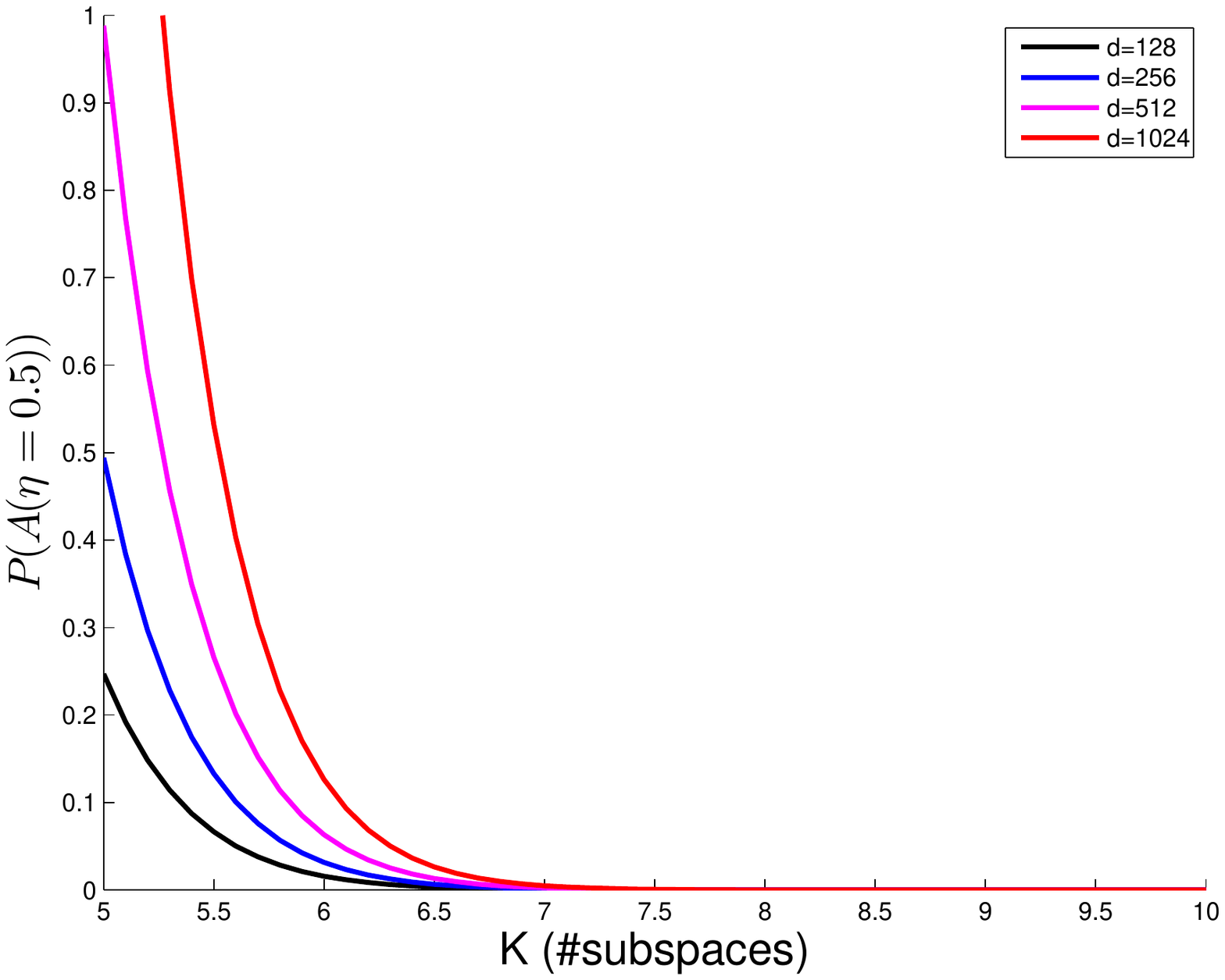}
\end{subfigure}
\caption{Upper bound on the probability 
of an event $A(\eta)$ that a vector $v$ obtained by the random rotation is not $\eta$-balanced as a function of the number of
subspaces $K$. The left figure corresponds to $\eta=0.75$ and the right one to $\eta=0.5$.
Different curves correspond to different data dimensionality ($d=128,256,512,1024$).}
\label{pfail_upper_bound}
\end{figure}

We start with the following Azuma's concentration inequality that we will also use later:

\begin{lemma}
\label{conc_lemma}
Let $X_{1},X_{2},...$ be random variables such that $E[X_{1}]=0$, $E[X_{i}|X_{1},...,X_{i-1}] = 0$
and $-\alpha_{i} \leq X_{i} \leq \beta_{i}$ for $i=1,2,...$
and some $\alpha_{1},\alpha_{1},...,\beta_{1},\beta_{2},...>0$. Then $\{X_{1},X_{2},...\}$ is a martingale and the following holds for any $a>0$:
$$
\mathbb{P}(\sum_{i=1}^{n}X_{i} \geq a) \leq exp(-\frac{2a^{2}}{\sum_{i=1}^{n}(\alpha_{i}+\beta_{i})^{2}}).
$$
\end{lemma}

Let us denote: $v=(v_{1},...,v_{d})$.
The $j$th entry of the transformed $x$ is of the form:
$h_{j,1}v_{1} + ... + h_{j,d}v_{d}$, where $(h_{j,1},...,h_{j,d})$ is the $j$th row of $\mathcal{H}_{n}$ and thus each $h_{j,i}$ (for the fixed $j$) takes uniformly at random and independently a value from the set $\{-\frac{1}{\sqrt{d}},\frac{1}{\sqrt{d}}\}$.

Let us consider random variable $Y_{1} = \sum_{j=1}^{\frac{d}{K}} (h_{j,1}v_{1}+...+h_{j,d}v_{d})^{2}$ that captures the squared $L_{2}$-norm of the first block of the transformed vector $v$. We have:
\begin{equation}
E[Y_{1}] = \sum_{j=1}^{\frac{d}{K}} (\frac{1}{d}v_{1}^{2}+...+\frac{1}{d}v_{d}^{2})+
2\sum_{j=1}^{\frac{d}{K}} \sum_{1\leq i_{1} < i_{2} \leq d}v_{i_{1}}v_{i_{2}}E[h_{j,i_{1}}h_{j,i_{2}}]=\frac{\|v\|_{2}^{2}}{K},
\end{equation}

where the last inequality comes from the fact that $E[h_{j,i_{1}}h_{j,i_{2}}]=0$
for $i_{1} \neq i_{2}$.
Of course the same argument is valid for other blocks, thus we can conclude that in expectation the transformed vector is $1$-balanced. Let us prove now some concentration inequalities regarding this result.
Let us fix some $j \in \{1,...,d\}$.
Denote $\xi_{i_{1},i_{2}} = v_{i_{1}}v_{i_{2}}h_{j,i_{1}}h_{j,i_{2}}$.
Let us find an upper bound on the probability $\mathbb{P}(|\sum_{1 \leq i_{1} < i_{2} \leq d} \xi_{i_{1},i_{2}}|>a)$ for some fixed $a>0$.
We have already noted that $E[\sum_{1 \leq i_{1} < i_{2} \leq d} \xi_{i_{1},i_{2}}]=0$. 

Thus, by applying Lemma \ref{conc_lemma}, we get the following:
\begin{equation}
\mathbb{P}(|\sum_{1 \leq i_{1} < i_{2} \leq d} \xi_{i_{1},i_{2}}|>a) \leq
2e^{-\frac{a^{2}d^{2}}{2(\sum_{i=1}^{d} v_{i}^{2})^{2}}}.
\end{equation}
Therefore, by the union bound, $\mathbb{P}(|Y_{1}-\frac{\|v\|^{2}_{2}}{K}| > \frac{da}{K}) \leq \frac{2d}{K}e^{-\frac{a^{2}d^{2}}{2(\sum_{i=1}^{d} v_{i}^{2})^{2}}}$.
Let us fix $\eta > 0$.
Thus by taking $a=\frac{\sigma K \|v\|^{2}_{2}}{d}$, and again applying the union bound (over all the blocks) we conclude that the transformed vector $v$ is not $\eta$-balanced with probability at most $2de^{-\frac{(1-\eta)^{2}K^{2}}{2}}$.
That completes the proof.
\end{proof}

Calculated upper bound on the probability of failure from Theorem \ref{hadamard_theorem} as a function of 
the number of blocks $K$ is presented on Fig. \ref{pfail_upper_bound}. We clearly see that failure probability 
exponentially decreases with number of blocks $K$.

\subsubsection{Proof of Theorem \ref{lipsch_theory_main}}

If some boundedness and balancedness conditions regarding datapoints can be assumed, we can 
obtain exponentially-strong concentration results regarding unbiased estimator considered in the paper. 
Next we show some results that can be obtained even if the boundedness and balancedness conditions do not hold.
Below we present the proof of Theorem \ref{lipsch_theory_main}.

\begin{proof}

Let us define: $\mathcal{Z} = \sum_{k=1}^{K} \mathcal{Z}^{(k)}$, where: $\mathcal{Z}^{(k)} = q^{(k)T}x^{(k)} - q^{(k)T}u_{x}^{(k)}$.
We have:
\begin{align}
\begin{split}
\label{main_ineq}
\mathbb{P}(\mathcal{F}(a,\epsilon)) = \mathbb{P}((q^{T}x > a) \land (q^{T}u_{x}>q^{T}x(1+\epsilon)) \lor (q^{T}u_{x}<q^{T}x(1-\epsilon)))\\
 \leq \mathbb{P}(|q^{T}x-q^{T}u_{x}| > a \epsilon)\\
 =\mathbb{P}(|\sum_{k=1}^{K}(q^{(k)T}x^{(k)}-q^{(k)T}u_{x}^{(k)})| > a \epsilon)\\
 = \mathbb{P}(|\sum_{k=1}^{K} \mathcal{Z}^{(k)}| > a\epsilon).
\end{split}
\end{align}

Note that from Eq. (\ref{main_ineq}), we get:
\begin{equation}
\label{imp_eq}
\mathbb{P}(\mathcal{F}(a,\epsilon)) \leq \mathbb{P}(|\sum_{k=1}^{K} \mathcal{Z}^{(k)}| > a \epsilon).
\end{equation}

Let us fix now the $k$th block ($k=1,...,K$). 
From the $\eta$-balancedness we get that every datapoint truncated to its $k$th block is within distance 
$\gamma = \sqrt{(\frac{1}{K}+(1-\eta))}r$ to $p^{(k)}$ (i.e. $z$ truncated to its $k$th block). Now consider in the linear space related to the $k$th block the ball $\mathcal{B}^{'}(p^{(k)},\gamma)$. Note that 
since the dimensionality of each datapoint truncated to the $k$th block is $\frac{d}{K}$, we can conclude that all datapoints truncated to their $k$th blocks that reside in $\mathcal{B}^{'}(p^{(k)},\gamma)$ can be covered by $c$ balls of radius $r^{'}$ each, where: $(\frac{\gamma}{r^{'}})^{\frac{d}{K}} = c$. We take as the set of quantizers $u^{(k)}_{1},...,u^{(k)}_{C}$ for the $k$th block the centers of mass of sets consisting of points from these balls. 
We will show now that sets: $\{u^{(k)}_{1},...,u^{(k)}_{C}\}$ ($k=1,...,K$) defined in such a way are the codebooks we are looking for.


From the triangle inequality and Cauchy-Schwarz inequality, we get: 
\begin{equation}
\label{z_bound}
|\mathcal{Z}^{(k)}| \leq (\max_{q \in Q} \|q^{(k)}\|_{2})(\max_{x \in X} \|x^{(k)}-u_{x}^{(k)}\|_{2}) \leq 2q_{max}r^{'} = 2q_{max}\gamma c^{-\frac{K}{d}}.
\end{equation}

This comes straightforwardly from the way we defined sets: 
$\{u^{(k)}_{1},...,u^{(k)}_{C}\}$ for $k=1,...,K$.

Let us take: $X_{i}=\mathcal{Z}^{(i)}$.
Thus, from (\ref{z_bound}), we see that $\{X_{1},...,X_{K}\}$ defined in such a way satisfies assumptions of Lemma \ref{conc_lemma} for $c_{k}=2q_{max}\gamma c^{-\frac{K}
{d}}$.

Therefore, from Lemma \ref{conc_lemma}, we get:

\begin{equation}
\mathbb{P}(|\sum_{k=1}^{K} \mathcal{Z}^{(k)}| > a\epsilon) \leq 
2e^{-(\frac{a \epsilon}{r})^{2}\frac{C^{\frac{2K}{d}}}{8q^{2}_{max}(1+(1-\eta)K)}},
\end{equation}

and that, by (\ref{imp_eq}), completes the proof.

The dependence of the probability of failure $\mathcal{F}(a,\epsilon)$
from Theorem \ref{lipsch_theory_main} on the number of subspaces $K$ is presented on Fig. \ref{f_bound}.

\end{proof}

\begin{figure}
\centering
\begin{subfigure}[c]{1 \textwidth}
\includegraphics[trim=0.7in 2.5in 0.9in 2.5in,clip,width=.49 \textwidth]{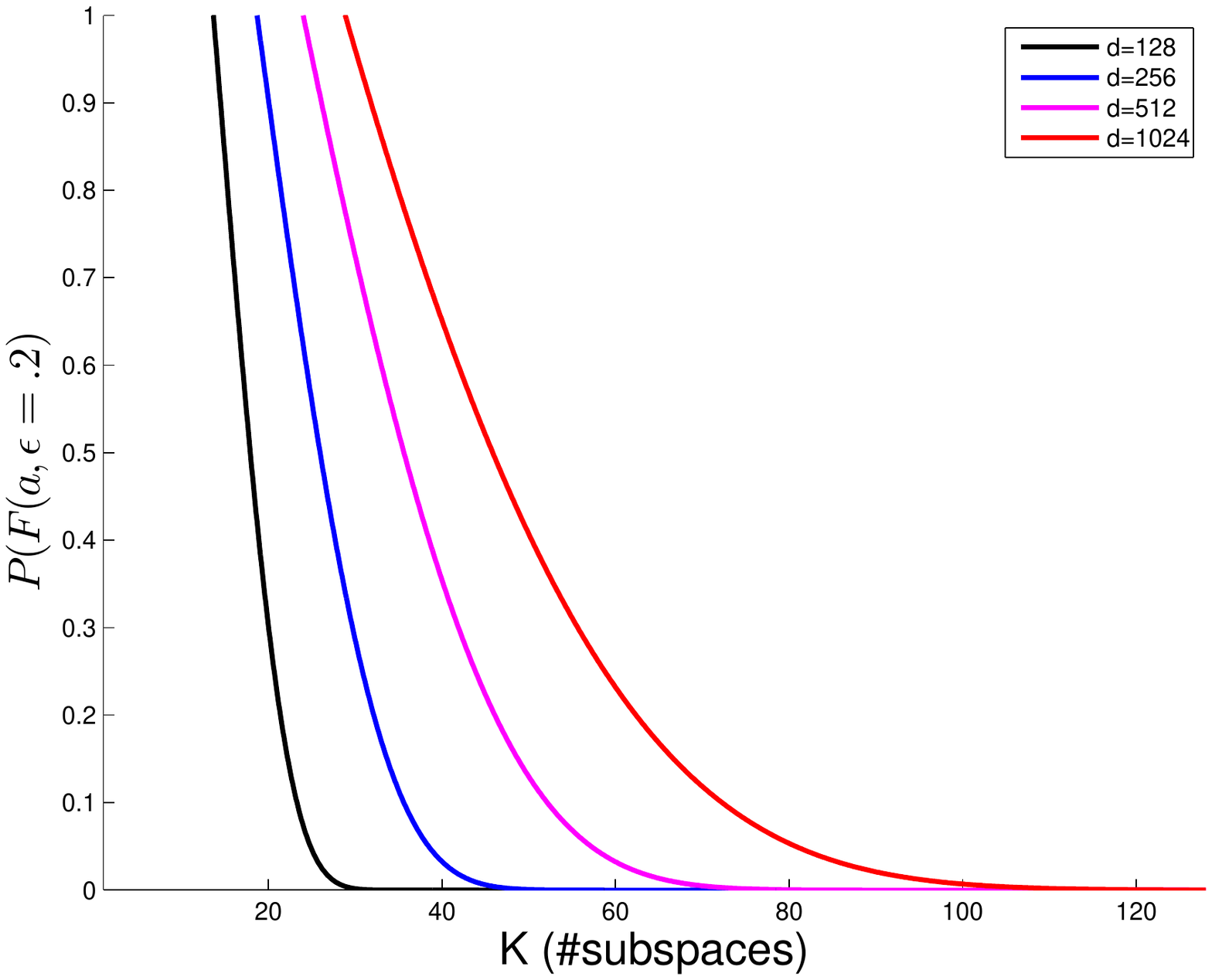}
\includegraphics[trim=0.7in 2.5in 0.9in 2.5in,clip,width=.49 \textwidth]{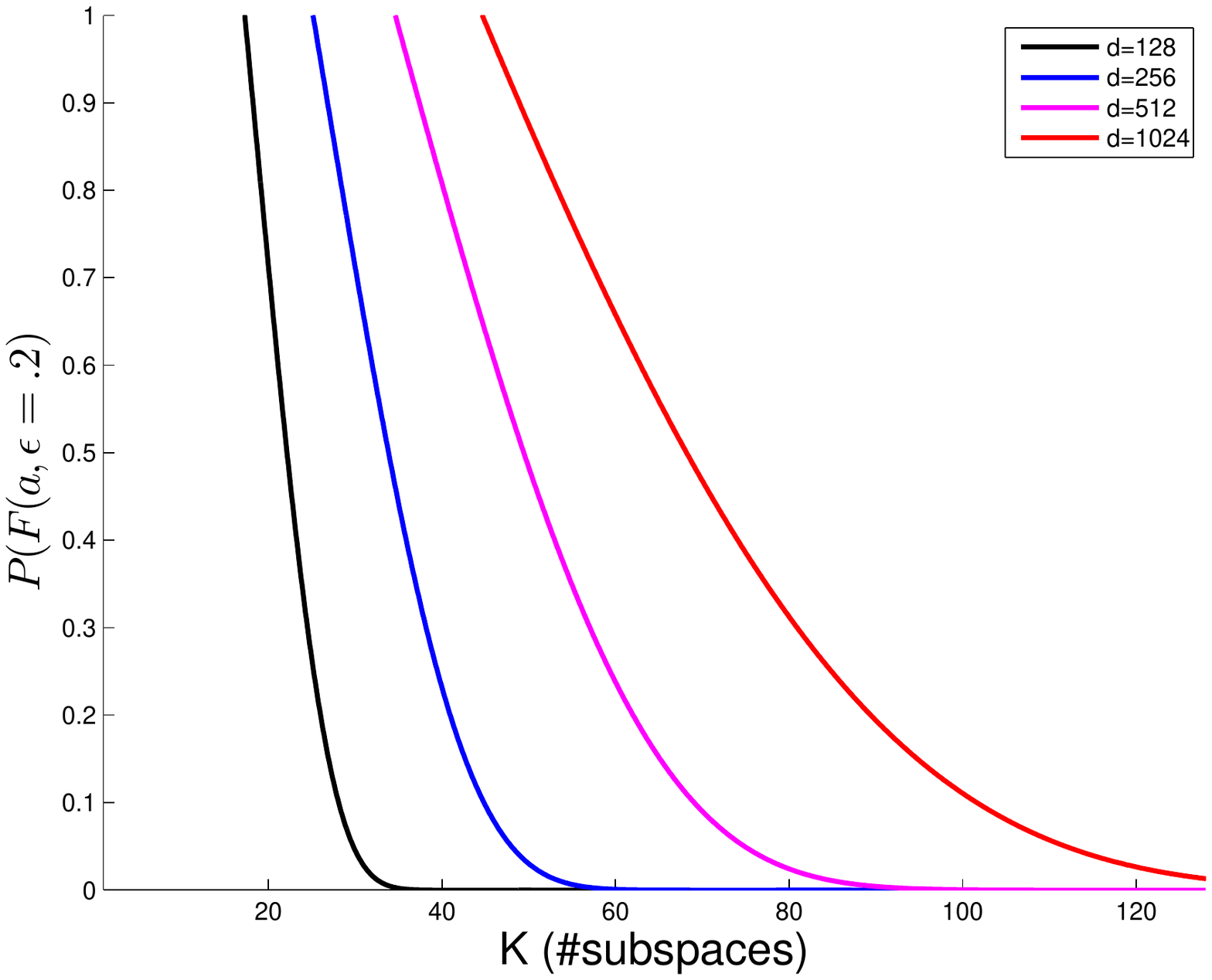}
\end{subfigure}
\caption{Upper bound on the probability 
of an event $\mathcal{F}(a,\epsilon)$ as a function of the number of
subspaces $K$ for $\epsilon=0.2$. The left figure corresponds to $\eta=0.75$ and the right one to $\eta=0.5$.
Different curves correspond to different data dimensionality ($d=128,256,512,1024$). We assume that the entire data is in the unit-ball and 
the norm of $q$ is uniformly split across all $K$ chunks.}
\label{f_bound}
\end{figure}

The following result is of its own interest since it does not assume anything about balancedness or boundedness.
It shows that minimizing the objective function $L = \sum_{k=1}^{K}L^{(k)}$, where:
$L^{(k)} = E_{q \sim \mathbf{Q}} [\sum_{S^{(k)}_{c}} \sum_{x\in S^{(k)}_{c}}(q^{(k)T} x^{k} - q^{(k)T} u^{(k)}_{x})^2]$,
leads to concentration results regarding error made by the algorithm.

\begin{theorem}
\label{variance_theorem}
The following is true:
$$\mathbb{P}(\mathcal{F}(a,\epsilon)) \leq \frac{K^{3}\max_{k=1,...,K}L^{(k)}}{|X|a^{2}\epsilon^{2}}.$$
\end{theorem}

\begin{proof}
Fix some $k \in \{1,...,K\}$.
Let us consider first the expression 
$L^{(k)} = E_{q \sim \mathbf{Q}} [\sum_{S^{(k)}_{c}} \sum_{x\in S^{(k)}_{c}}(q^{(k)T} x^{k} - q^{(k)T} u_{x}^{(k)})^2]$ 
that our algorithm aims to minimize. We will show that it is a rescaled version of the variance of the random variable 
$\mathcal{Z}$.

We have: 

\begin{align}
\begin{split}
Var(\mathcal{Z}^{(k)}) = E_{q \sim \mathbf{Q}, x \sim \mathbf{X}}[(q^{(k)T}x^{(k)}-q^{(k)T}u_{x}^{(k)})^{2}]
-(E_{q \sim \mathbf{Q}, x \sim \mathbf{X}}[q^{(k)T}x^{(k)}-q^{(k)T}u_{x}^{(k)}])^{2} \\
=E_{q \sim \mathbf{Q}, x \sim \mathbf{X}}[(q^{(k)T}x^{(k)}-q^{(k)T}u_{x}^{k})^{2}],
\label{var_eq1}
\end{split}
\end{align}

where the last inequality comes from the unbiasedness of the estimator (Lemma \ref{thm:unbiased}). 

Thus we obtain:

\begin{align}
\begin{split}
Var(\mathcal{Z}^{(k)}) = E_{q \sim \mathbf{Q}} [\sum_{x \in X} \frac{1}{|X|} (q^{(k)T}x^{(k)}-q^{(k)T}u_{x}^{(k)})^{2}] =\frac{1}{|X|}L^{(k)}.
\end{split}
\end{align}

Therefore, by minimizing $L^{(k)}$ we minimize the variance of the random variable that measures the discrepancy between exact answer and quantized answer to the dot product query for the space truncated to the fixed $k$th block.
Denote $u_{x}=(u_{x}^{(1)},...,u_{x}^{(K)})$.
We are ready to give an upper bound on $\mathbb{P}(\mathcal{F}(a,\epsilon))$.

We have:

\begin{align}
\begin{split}
\mathbb{P}(\mathcal{F}(a,\epsilon))
 \leq \mathbb{P}(|q^{T}x-q^{T}u_{x}| > a \epsilon)
 =\mathbb{P}(|\sum_{k=1}^{K}(q^{(k)T}x^{(k)}-q^{(k)T}u_{x}^{(k)})| > a \epsilon)\\
 \leq \mathbb{P}(\sum_{k=1}^{K}|(q^{(k)T}x^{(k)}-q^{(k)T}u_{x}^{(k)})| > a \epsilon) \\
 \leq \mathbb{P}(\exists_{k \in \{1,...,K\}}|q^{(k)T}x^{(k)}-q^{(k)T}u_{x}^{(k)})| > \frac{a \epsilon}{K})\\
\leq \frac{K^{3}\max_{k \in \{1,...,K\}}(Var(q^{(k)T}x^{(k)}-q^{(k)T}u_{x}^{(k)}))}{a^{2}\epsilon^{2}}\\
= \frac{K^{3}\max_{k=1,...,K} Var(\mathcal{Z}^{(k)})}{a^{2}\epsilon^{2}}.
\end{split}
\end{align}

The last inequality comes from Markov's inequality applied to the random variable $(\mathcal{Z}^{(k)})^{2}$ and the union bound.
Thus, by applying obtained bound on $Var(\mathcal{Z}^{(k)})$, we complete the proof.

\end{proof}

\subsubsection{Independent blocks - the proof of Theorem \ref{ind_theory1_main}}

Let us assume that different blocks correspond to independent sets of dimensions. 
Such an assumption is often reasonable in practice. 
If this is the case, we can strengthen our methods for obtaining tight concentration inequalities. 
The proof of Theorem \ref{ind_theory1_main} that covers this scenario is given below.

\begin{proof}
Let us assume first the most general case, when no balancedness is assumed. We begin the proof in the same way as we did 
in the previous section, i.e. fix some $k \in \{1,...,K\}$ and consider random variable $\mathcal{Z}^{(k)}$.
The goal is again to first find an upper bound on $Var(\mathcal{Z}^{(k)})$.
From the proof of Theorem \ref{variance_theorem} we get: 
$Var(\mathcal{Z}^{(k)}) = \frac{1}{|X|}L^{(k)}$. Then again, following the proof of 
Theorem \ref{variance_theorem}, we have:

\begin{equation}
\label{berry_intro}
\mathbb{P}(\mathcal{F}(a,\epsilon)) \leq \mathbb{P}(|\sum_{i=k}^{K} \mathcal{Z}^{(k)}| > a \epsilon)
\end{equation}

We will again bound the expression $\mathbb{P}(|\sum_{i=k}^{K} \mathcal{Z}^{(k)}| > a \epsilon)$.
We will use now the following version of the Berry-Esseen inequality ([11]):

\begin{theorem}
\label{berry}
Let $\{S_{1},...,S_{n}\}$ be a sequence of independent random variables with mean $0$, not necessarily identically distributed, with finite third moment each.
Assume that $\sum_{i=1}^{n}E[S_{i}^{2}] = 1$. Define: $W = \sum_{i=1}^{n} S_{i}$.
Then the following holds:
$$
|\mathbb{P}(W_{n} \leq x) - \phi(x)| \leq \frac{C}{1+|x|^{3}}\sum_{i=1}^{n}E[|S_{i}|^{3}],
$$
for every $x$ and some universal constant $C>0$, where
$\phi(x) = \mathbb{P}(g \leq x)$ and $g \sim \mathcal{N}(0,1)$.
\end{theorem}

Note that if dimensions corresponding to different blocks are independent, then
$\{\mathcal{Z}^{(1)},...,\mathcal{Z}^{(K)}\}$ is the family of independent random variables. This is the case, since every $\mathcal{Z}^{(k)}$ is defined as:
$\mathcal{Z}^{(k)} = q^{(k)T}x^{(k)}-q^{(k)T}u_{x}^{(k)})$.
Note that we have already noticed that the following holds: 
$E[\mathcal{Z}^{(k)}]=0$.
Let us take: $S^{(k)} = \frac{\mathcal{Z}^{(k)}}{\sqrt{\sum_{k=1}^{K} Var(\mathcal{Z}^{(k)})}}$. Clearly, we have: $\sum_{k=1}^{K}E[(S^{(k)})^{2}]=1$.
Besides, random variables $S^{(k)}$ defined in this way are independent and
$E[S^{(k)}]=0$ for $k=1,...,K$.
Denote: 

\begin{equation}
F = \sum_{k=1}^{K} E[|S^{(k)}|^{3}] = \frac{1}{(\sum_{k=1}^{K}Var(\mathcal{Z}^{(k)}))^{\frac{3}{2}}}\sum_{k=1}^{K} E[|\mathcal{Z}^{(k)}|^{3}]
\end{equation}

Thus, from Theorem \ref{berry} we get:

\begin{equation}
|\mathbb{P}\left(\frac{\sum_{k=1}^{K}\mathcal{Z}^{(k)}}
{\sqrt{\sum_{k=1}^{K}Var(\mathcal{Z}^{(k)})}} \leq x\right)-\phi(x)| \leq 
\frac{C}{1+x^{3}}F.
\end{equation}

Therefore, for every $c>0$ we have:

\begin{align}
\begin{split}
\mathbb{P}\left(\frac{|\sum_{k=1}^{K}\mathcal{Z}^{(k)}|}
{\sqrt{\sum_{k=1}^{K}Var(\mathcal{Z}^{(k)})}} > c\right) 
= 1 - \mathbb{P}\left(\frac{\sum_{k=1}^{K}\mathcal{Z}^{(k)}}
{\sqrt{\sum_{k=1}^{K}Var(\mathcal{Z}^{(k)})}} \leq c\right) \\
+ \mathbb{P}\left(\frac{\sum_{k=1}^{b}\mathcal{Z}^{(k)}}
{\sqrt{\sum_{k=1}^{K}Var(\mathcal{Z}^{(k)})}} < -c\right) \\
\leq 1 - \phi(c) + \phi(-c) + \frac{2C}{1+c^{3}}F
\end{split}
\end{align}

Denote $\hat{\phi}(x) = 1 - \phi(x)$.
Thus, we have:

\begin{align}
\label{esseen_final}
\begin{split}
\mathbb{P}\left(|\sum_{k=1}^{K} \mathcal{Z}^{(k)}| > 
c\sqrt{\sum_{k=1}^{K} Var(\mathcal{Z}^{(k)})}\right) \leq 1-\phi(c)+\phi(-c)+\frac{2C}{1+c^{3}}F \\
= 2\hat{\phi}(c) + \frac{2C}{1+c^{3}}F \\
\leq \frac{2}{\sqrt{2 \pi}c}e^{-\frac{c^{2}}{2}}+\frac{2C}{1+c^{3}}F, 
\end{split}
\end{align}

where in the last inequality we used a well-known fact that: 
$\hat{\phi}(x) \leq \frac{1}{\sqrt{2\pi}x}e^{-\frac{x^{2}}{2}}$.

If we now take: $c = \frac{a \epsilon}{\sqrt{\sum_{k=1}^{K} Var(\mathcal{Z}^{(k)}}}$,
then by applying (\ref{esseen_final}) to (\ref{berry_intro}), we get:
\begin{align}
\begin{split}
\mathbb{P}(|\sum_{k=1}^{K}\mathcal{Z}^{(k)}| > a\epsilon) \leq 
\frac{2\sum_{k=1}^{b}Var(\mathcal{Z}^{(k)})}{\sqrt{2 \pi} a\epsilon}
e^{-\frac{a^{2}\epsilon^{2}}{2(\sum_{k=1}^{K}Var(\mathcal{Z}^{(k)}))^{2}}}\\
+ \frac{2C}{1+(\sum_{k=1}^{K}Var(\mathcal{Z}^{(k)}))^{\frac{3}{2}}}\sum_{k=1}^{K}E[|\mathcal{Z}^{(k)}|^{3}],
\end{split}
\end{align}

Substituting the exact expression for $Var(\mathcal{Z}^{(k)})$, we get:

\begin{equation}
\label{final_equation}
\mathbb{P}(|\sum_{k=1}^{K}\mathcal{Z}^{(k)}| > a\epsilon) \leq
\frac{2\sum_{k=1}^{K}L^{(k)}}{\sqrt{2\pi}|X|a\epsilon}e^{-\frac{a^{2}\epsilon^{2}|X|^{2}}{2(\sum_{k=1}^{K}L^{(k)})^{2}}}+
\frac{2C(\sum_{k=1}^{K} L^{(k)})^{\frac{3}{2}}}{a^{3}\epsilon^{3}|X|^{\frac{3}{2}}}
\sum_{k=1}^{K}E[|\mathcal{Z}^{(k)}|^{3}].
\end{equation}

Note that $|\mathcal{Z}^{(k)}| = |q^{(k)T}x^{(k)}-q^{(k)T}u_{x}^{(k)}|=
|q^{(k)T}(x^{(k)} - u_{x}^{(k)})| \leq 
\|q^{(k)}\|_{2}\|x^{(k)}-u_{x}^{(k)}\|_{2}$.
The latter expression is at most $q_{max}\Delta$, by the definition of $\Delta$
and $q_{max}$. Thus we get: $|\mathcal{Z}^{(k)}|^{3} \leq q_{max}^{3}\Delta^{3} \leq a$,
where the last inequality follows from the assumptions on $\Delta$ from the statement of the theorem. Therefore, from \ref{final_equation} we get:

\begin{equation}
\mathbb{P}(|\sum_{k=1}^{K}\mathcal{Z}^{(k)}| > a\epsilon) \leq
\frac{2\sum_{k=1}^{K}L^{(k)}}{\sqrt{2\pi}|X|a\epsilon}e^{-\frac{a^{2}\epsilon^{2}|X|^{2}}{2(\sum_{k=1}^{K}L^{(k)})^{2}}}+
\frac{2CK(\sum_{k=1}^{K} L^{(k)})^{\frac{3}{2}}}{a^{2}\epsilon^{3}|X|^{\frac{3}{2}}}.
\end{equation}

Thus, taking into account (\ref{berry_intro}) and putting $\beta = 2C$, we complete the proof.
\end{proof}

{
\bibliographystyle{ieee}
\bibliography{thesisrefs}
}

\end{document}